\documentclass[10pt,twocolumn,letterpaper]{article}

\usepackage{iccv}
\usepackage{times}
\usepackage{epsfig}
\usepackage{graphicx}
\usepackage{amsmath}
\usepackage{amssymb}
\makeatletter
\@namedef{ver@everyshi.sty}{}
\makeatother
\usepackage{pgfplots}
\usepackage{xparse}
 \ExplSyntaxOn
 \NewDocumentCommand{\longset}{mm}
  {
   \begin{gathered}
   \egreg_longset:nn { #1 } { #2 }
   \end{gathered}
  }
  \box_new:N \l_egreg_longset_set_box
 \seq_new:N \l_egreg_longset_elements_seq

 \cs_new_protected:Nn \egreg_longset:nn
  {
   \hbox_set:Nn \l_egreg_longset_set_box { $#1=\lbrace\,$ }
   #1 = \lbrace\,
   \begin{minipage}[t]
   {
     \dim_eval:n { \displaywidth - \box_wd:N \l_egreg_longset_set_box - 4em }
   }
   \muskip_set:Nn \thinmuskip { 1\thinmuskip plus 5mu }
   \seq_set_from_clist:Nn \l_egreg_longset_elements_seq { #2 }
   $\seq_use:Nn \l_egreg_longset_elements_seq { , \penalty0~ }\,\rbrace$
   \end{minipage}
  }
 \ExplSyntaxOff

\newcommand{\roberto}[1]{\todo[inline,color=orange]{Roberto: #1}}
\newcommand{\timo}[1]{\todo[inline,color=yellow]{Timo: #1}}
\newcommand{\andrea}[1]{\todo[inline,color=blue!30]{Andrea: #1}}

 \usepackage{amsthm}
 \usepackage{url}
  \usepackage{bbm}
\usepackage{varwidth}

 \usepackage{epsfig}
\usepackage{mathtools}

\usepackage[bb=boondox]{mathalfa}

 \usepackage{multirow}
 \usepackage{booktabs}
 \usepackage{comment}

\usepackage{algorithm}
\usepackage{algpseudocode}
\algnewcommand{\IIf}[1]{\State\algorithmicif\ #1\ \algorithmicthen}
\algnewcommand{\EElse}{\State\algorithmicelse\ }
\algnewcommand{\EndIIf}{\unskip\ \algorithmicend\ \algorithmicif}

  \usepackage{times}
  \usepackage{tabularx}
\makeatletter
\@namedef{ver@everyshi.sty}{}
\makeatother
\usepackage[disable]{todonotes}
\newcommand{\R}{\mathbb{R}}

\newcommand{\la}{\langle}
\newcommand{\ra}{\rangle}

\DeclareMathOperator*{\argmin}{arg min}
\DeclareMathOperator*{\argmax}{arg max}

 \mathchardef\mhyphen="2D
 
\newcommand{\pluseq}{\mathrel{{+}{=}}}
\newcommand{\minuseq}{\mathrel{{-}{=}}}
\newcommand{\plusplus}{\mathrel{{+}{+}}}

\providecommand{\abs}[1]{\lvert#1\rvert}

\newtheorem{lemma}{Lemma}
\newtheorem{proposition}{Proposition}

\usepackage{enumitem}
\setlist[1]{noitemsep, topsep=0pt,leftmargin=*}
\providecommand{\myparagraph}[1]{\noindent\textbf{#1}\hspace{0.2em}}
\usepackage[pagebackref=true,breaklinks=true,letterpaper=true,colorlinks,bookmarks=false]{hyperref}

\iccvfinalcopy % *** Uncomment this line for the final submission

\def\iccvPaperID{8638} % *** Enter the ICCV Paper ID here

\begin{document}
	
	%%%%%%%%% TITLE
	\title{Making Higher Order MOT Scalable: An Efficient Approximate Solver for Lifted Disjoint Paths}

	%%Version with split in Hanover institute
	\author{{Andrea Hornakova$^{*1}$ \qquad Timo Kaiser$^{*2}$ \qquad Paul Swoboda$^1$ \qquad Michal Rolinek$^3$}\\{ Bodo Rosenhahn$^2$ \qquad Roberto Henschel$^2$} \\
		%\texttt{\small{ahmed.a.mirza1@gmail.com}} \\ % I checked other papers and often no email address is given. I would leave it out here as well.
		{\small $^*$Authors contributed equally, $^1$Max Planck Institute for Informatics, Saarland Informatics Campus, $^2$Institute for Information Processing,}\\ {\small Leibniz University Hannover,
			$^3$Max Planck Institute for Intelligent Systems, T\"ubingen}}
	
	\maketitle
	%\ificcvfinal\fi
	
	%%%%%%%%% ABSTRACT
	\begin{abstract}
 We present an efficient approximate message passing solver for the lifted disjoint paths problem (LDP), a natural but NP-hard model for multiple object tracking (MOT).
 Our tracker scales to very large instances that come from long and crowded MOT sequences.
 Our approximate solver enables us to process the MOT15/16/17 benchmarks  without sacrificing solution quality and allows for solving MOT20, which has been out of reach up to now for LDP solvers due to its size and complexity.
  On all these four standard MOT benchmarks we achieve performance comparable or better than  current state-of-the-art methods including a tracker based on an optimal LDP solver.
\end{abstract}

	%%%%%%%%% BODY TEXT
	\section{Introduction}
\andrea{Should we include MOT16 and MOT15 in the abstract? Can we replace solvers (plural) by solver?}
\timo{i think we should}

% Why is MOT important in general ? 
Deriving high-level understanding from a video is a desired task that has been studied in computer vision for a long time. Nevertheless, solving the problem is a long way off.
A computer vision system able to extract the motions of objects appearing in a video  in terms of trajectories is considered as a prerequisite for the goal. This task, called multiple object tracking (MOT), has numerous applications, e.g.\ in the area of video surveillance~\cite{fenzi2014asev}, sports analysis~\cite{alahi2009sport,lu2013learning}, urban planning~\cite{alahi2017unsupervised}, or autonomous driving~\cite{liang2020pnpnet,frossard2018end}. 

Yet, solving MOT is challenging, especially for long and crowded sequences.
% How is MOT approached ? 
The predominant approach for MOT is the tracking-by-detection paradigm, which splits the problem into two subtasks. First, objects are detected in all video frames by an object detector. Then, the detections are linked across frames to form trajectories. While the performance of object detectors has improved considerably by recent advances of CNNs \cite{ren2015faster,yang2016exploit,redmon2016you,duan2019centernet}, the latter task called the  \textit{data association} remains challenging. 
The data association reasons from \textit{pairwise costs},  which indicate for each pair of detections the likelihood of belonging to the same object.
%Why is the data association problem difficult  w.r.t inputs? 

Appearance and spatio-temporal information are often ambiguous, especially in crowded scenes and pairwise costs can be misleading. Moreover, object detectors produce more errors in crowded scenes due to partial occlusions.
To resolve these issues, it is crucial that the data association incorporates global context.

% How is the data association solved ? 
The disjoint paths problem (DP)~\cite{zhang2008global,kovacs2015minimum} is a natural model for MOT. Results are computed efficiently  using a min-cost flow algorithm that delivers the global optimal solution. 
Unfortunately, the integration of long range temporal interactions is limited, as DP obeys the first-order Markov-chain assumption: for each trajectory, consistency is ensured only between directly linked detections, which is a strong simplification that ignores higher order consistencies among multiple linked detections.

%\begin{figure}[t]
%\include{runtime_plot}
%\caption{Runtime comparison between Ours with 6, 11 and 31 iterations against globally optimal Gurobi\cite{gurobi}. For the comparison we clipped the complex tracking sequence \textit{MOT20-01} from MOT20 dataset\cite{MOTChallenge20} after several numbers of frames to measure the runtime with subject to increasing problem complexity.}
%\label{fig:runtime}
%\end{figure}

% How can this deficiency be fixed ?
% To fix this deficiency, the disjoint paths  problem has been generalized to the lifted disjoint paths (LDP) \cite{hornakova2020lifted}, where additional connectivity priors (in terms of \textit{lifted} edges) are added, making the formulation much more expressive while maintaining the feasibility set of the DP formulation, see Section \ref{sec:problem_formulation}. 
% The optimization problem then  takes all connections within a trajectory into account (in terms of the pairwise costs). It thus enables to incorporate long range temporal interactions effectively, which  has been shown to improve recall and precision considerably  \cite{hornakova2020lifted}. Similar extensions have been made for the multicut problem \cite{tang2016multi,tang2017multiple}.

To fix this deficiency, \cite{hornakova2020lifted}  generalizes DP to lifted disjoint paths (LDP) by using additional connectivity priors in terms of \textit{lifted} edges. This makes the formulation much more expressive while it maintains the feasibility set of the DP (Sec.\ \ref{sec:problem_formulation}). 
The optimization problem enables to take  into account pairwise costs between arbitrary detections belonging to one trajectory. It thus enables to incorporate long range temporal interactions effectively and leads to considerable improvement of recall and precision  \cite{hornakova2020lifted}. Similar extensions have been made for the multicut problem \cite{tang2016multi,tang2017multiple}.

While the integration of the global context by LDP is crucial to obtain high-quality tracking results,  it makes  the data association problem  NP-hard.
%, and the set of feasible trajectories  has exponential growth.
Still, \cite{hornakova2020lifted} presented a~global optimal LDP solver usable for  semi-crowded sequences with reasonable computational effort.
However, when applied to  longer and crowded sequences, such approaches are not tractable anymore, due to too high demands on runtime and memory.
\andrea{Shall we use LDP solver instead of solvers? There is only one solver, right?}
\roberto{We can do it, but then consistent with the abstract}

In order to close this gap, we present the first approximate solver for LDP.
The resulting tracker scales to big problem instances and incorporates global context with similar accuracy as the global optimal LDP solver.
Moreover, our solver outputs certificates in terms of primal/dual gaps.
%While being approximative, the solver still outputs certificates in terms of primal/dual gaps and relies on a convex relaxation, hence still being a principled algorithm.

In particular, our solver is based on a Lagrangean (dual) decomposition of the problem.
This dual is iteratively optimized by dual block coordinate ascent (a.k.a.\ message passing) using techniques from~\cite{Swoboda_2017_CVPR}, see Sec.~\ref{sec:lagrange-decomposition-sub}.
The decomposition relies on subproblems that are added in a cutting plane fashion.
We obtain high-quality primal solutions by solving minimum cost flow problems with edge costs synthesizing information from both base and lifted edges from the dual task and improve them via a local search procedure. 
%Moreover, the costs are estimated on lightweight features to reduce computational efforts for large and dense datasets.
%Additional to the efficient solver the costs are estimated on lightweight features to reduce computational efforts for large and dense datasets. 

% We validate the approximating quality of solver on several state-of-the-art MOT benchmarks in Section \ref{sec:experiments}. It achieves \textcolor{blue}{comparable or better} performance as the current state of the art on MOT15\cite{MOTChallenge2015}, MOT16 and MOT17\cite{MOT16}, which are datasets known to be already saturating \cite{hornakova2020lifted}.

%We show that our solver achieves state-of-the-art performance on the MOT17 dataset \cite{?}, which consists of semi-crowded sequences. In particular, the results are on pair with the results obtained by exchanging our solver with an optimal LDP solver. 
We validate the quality of the solver on four standard MOT benchmarks (Sec.~\ref{sec:experiments}). We achieve comparable or better performance w.r.t.\ the current state-of-the-art trackers including the tracker based on the  optimal LDP solver~\cite{hornakova2020lifted} on MOT15/16/17~\cite{MOTChallenge2015,MOT16}.
Furthermore, our proposed tracker performs on par with state-of-the-art on the more challenging MOT20 dataset~\cite{MOTChallenge20} which is composed of long and crowded sequences. Lightweight features and a fast solver are crucial to perform tracking on such massive sequences.
%
%Due to the high density of objects, the extension LDP-based methods requires lightweight features and a fast solver. 
Our work thus extends the applicability of the successful LDP formulation to a wider range of  instances.

\myparagraph{Contribution} of this work is in summary as follows:
%\noindent In summary, the contribution of this work is as follows:
\noindent\begin{itemize}
    \item We make the LDP problem more accessible and applicable by introducing an approximate solver with better scalability properties than the global optimal LDP solver, while resulting in similar tracking performance, and being independent of Gurobi \cite{gurobi}. 
    \item We present an MOT system that is scalable to challenging sequences by using considerably less computationally demanding features than what is used in the state-of-the-art tracker \cite{hornakova2020lifted}. Our system %is the first to
    incorporates higher order consistencies in a scalable way, i.e.\ it uses an approximate solver and provides  a~gap to the optimum. 
   % \item Fast Convergence..
    
\end{itemize}
 We make our LDP solver\footnote{\url{https://github.com/LPMP/LPMP}} and our MOT pipeline\footnote{\url{https://github.com/TimoK93/ApLift}} available.

	\section{Related Work}

\myparagraph{(Lifted) disjoint paths. }The disjoint paths problem is a natural model for multiple object tracking and is solvable with fast combinatorial solvers~\cite{kovacs2015minimum}. It has been used for the data association step of MOT in~\cite{berclaz2011multiple,zhang2008global}. Its extensions  have been used for fusing different object detectors~\cite{chari2015pairwise} or multi-camera MOT~\cite{hofmann2013hypergraphs,leal2012branch}. Its main disadvantage is that it does not allow to integrate long range information because it evaluates only direct connections between object detections within a trajectory.
The lifted disjoint paths problem introduced in~\cite{hornakova2020lifted} enhances DP by introducing lifted edges that enable to reward or penalize arbitrary connections between object detections. This incorporation of long range information leads to a significant improvement of the tracking performance yielding state-of-the-art results on main MOT benchmark but makes the problem NP-hard. The authors provide a globally optimal solver using Gurobi~\cite{gurobi}. Despite a lot of efficient subroutines, the general time complexity of the provided solver remains exponential.
Therefore, in order to extend LDP-based methods to highly dense MOT problems as in MOT20 it is crucial to reduce the complexity of the used LDP solver
because the number of feasible connections between detections increases dramatically.

\myparagraph{Multicut and lifted multicut.} LDP is similar to (lifted) multicut~\cite{chopra1993partition,hornakova2017analysis}.
Multicut has been used for MOT in ~\cite{Ho_2020_ACCV,keuper2018motion,kumar2014multiple,ristani2014tracking,tang2015subgraph,tang2016multi}, lifted multicut in~\cite{babaee2018multiple,tang2017multiple}.
%Where the work~\cite{Ho_2020_ACCV} deserves a special attention because it is a self-supervised technique that does not need labels for the training data.
These trackers solve the underlying combinatorial problem approximately via heuristics without providing an estimation of the  gap to optimality.
Our approach,  delivers an approximate solution together with a lower bound enabling to assess the quality of the solution.
Additionally, LDP provides a strictly better relaxation than lifted multicut~\cite{hornakova2020lifted}.

%Although there exists a message passing algorithm for the multicut problem within the same framework as we use~\cite{Swoboda_2017_CVPR_multicut}, it has not been applied to the MOT. 
%\roberto{Can we remove the highlight about self-supervised tracking ? Just list it with the others? }

% {\color{blue}
% \myparagraph{Other combinatorial approaches to MOT.}
% Other combinatorial approaches to MOT include maximum clique problem~\cite{zamir2012gmcp,dehghan2015gmmcp}, maximum independent set~\cite{brendel2011multiobject} and multigraph-matching problem~\cite{hu2019dual}.
% The minimum cost arborescence problem, has been used in~\cite{henschel2014efficient}.

% \andrea{@Roberto, can you please re-formulate the following text. It is now a copy from LifT paper.}
% The works~\cite{Henschel_2018_CVPR_Workshops,henschel2016tracking} reformulate tracking multiple objects with long temporal interactions as a binary quadratic program. If the problem size is small, the optimization problem can be solved optimally by reformulating it to an equivalent binary linear program \cite{henschel2020simultaneous,von2018recovering}. For large instances, an approximation is necessary.  To this end, a specialized non-convex Frank-Wolfe method can be used \cite{Henschel_2018_CVPR_Workshops}.

% }

\myparagraph{Other data association models for MOT.}
Several works employ greedy heuristics to obtain tracking results \cite{bochinski2017high,zhou2021tracking,bergmann2019tracking}.
%, \eg based on the intersection-over-union between two detections \cite{bochinski2017high}, using a neural network that detects for a given frame all objects and their predicted positions in the previous frame \cite{zhou2021tracking}, or by propagating each trajectory to the next frame using the regression head of an object detector \cite{bergmann2019tracking}.
Such strategies normally  suffer from occlusions or ambiguous situations, causing  trajectory errors. Others use bipartite matchings \cite{huang2008robust,sadeghian2017tracking,zhu2018online,xu2019spatial,Wojke2017simple,Wojke2018deep} to assign  new detections with already computed trajectories of the past optimally. Since no global context is incorporated,  they are  prone to errors if individual edge costs are misleading. %For a global view, MOT  has been formulated as minimum cost arborescence problem \cite{henschel2014efficient}, which is corresponds to the minimum spanning tree problem for directed graphs. 

Higher order MOT frameworks ensure consistencies within all detections of a trajectory.
This can be done greedily, by computing one trajectory at a time via  a generalized minimum clique problem \cite{zamir2012gmcp}, or globally using an  extension to the maximum multi clique problem \cite{dehghan2015gmmcp}. 

\roberto{Relating the problems to Multicut would have the advantage that the reader understand that they are NP-hard and can only be solved so far using KL heuristics..}

%Several works approximate the tracking solution in the continuous domain.
Several works employ continuous domain relaxation.
When MOT is formulated as a binary quadratic program \cite{Henschel_2018_CVPR_Workshops,henschel2019multiple,von2018recovering,henschel2020simultaneous}, a modification of the Frank-Wolfe algorithm adapted to the non-convex case has been used  \cite{Henschel_2018_CVPR_Workshops}. 
%This enables to solve big problem instances \cite{Henschel_2018_CVPR_Workshops,henschel2019multiple}. If the number of objects to be tracked is small, the problem can be solved to optimality using a reformulation as an equivalent binary linear program  \cite{von2018recovering,henschel2020simultaneous}. 
Some approximations for binary linear programs use an LP-relaxation, optimize in the continuous domain and derive a~binary solution from the continuous one \cite{jiang2007linear,chari2015pairwise,brendel2011multiobject}.
They however do not provide the optimality gap, in contrast to our work.
Higher order MOT can be considered as a classification problem using graph convolutions \cite{braso2020learning}.
It allows to train features directly on the tracking task.
%Trajectories are obtained from the continuous domain. 
%Methods operating in the continuous domain are prone to get stuck in a local optimum (in the continuous domain) and do not provide optimality certificates, in contrast to our work. 

The multigraph-matching problem, a generalization of the graph matching problem, has been used  for MOT \cite{hu2019dual}.
Here, cycle consistency constraints of the multi-graph matching ensures higher order consistencies of the trajectories.
Message passing for higher order matching in MOT has been used in~\cite{aroraICCV13}
employing a~variant of MPLP~\cite{globerson2007fixing}.
In contrast to our formulation,~\cite{aroraICCV13} does not model occlusions and does not allow for connectivity priors.

Probabilistic approaches to multiple-target tracking include multiple hypotheses tracking~\cite{kim2015multiple,chong2018forty}, joint probabiblistic data association~\cite{rezatofighi2015joint,9011349} and others  \cite{9011349,meyer2018message}.
% Finally, multiple hypothesis tracker assign each feasible trajectory a weight. Solving the underlaying maximum weight independent set problem \cite{brendel2011multiobject,kim2015multiple} relies on aggressive pruning strategies \cite{kim2015multiple} or LP-relaxation \cite{brendel2011multiobject}.  

% \textcolor{blue}{Finally, multiple hypothesis tracker assign each feasible trajectory (instead of a pair of detections) a weight. Solving the underlaying maximum weight independent set problem \cite{brendel2011multiobject,kim2015multiple} relies on aggressive pruning strategies \cite{kim2015multiple} to reduce the search space, or to approximate the problem using a LP-relaxation \cite{brendel2011multiobject}.}  
\roberto{Paul wrote its the same as a maximum clique on the complement graph. Should I mention that? -> no}
\timo{I would like to remove it}

	\section{Problem Formulation}
\label{sec:problem_formulation}
The lifted disjoint paths problem (LDP) introduced in~\cite{hornakova2020lifted} is an optimization problem for finding a set of vertex-disjoint paths in a~directed acyclic graph.
The cost of each path is determined by the cost of edges in that path as in the basic disjoint paths problem (DP), but additionally there are higher order costs defined by lifted edges.
A~lifted edge contributes to the cost if its endpoints are part of the same path.
%Through this, connectivity costs are incorporated in the basic disjoint paths problem.
This problem is a natural formulation for multiple object tracking (MOT), where lifted edges allow to re-identify the same objects over long distance.
%An example LDP is shown in Figure~\ref{fig:problemIllustration}.

While of greater expressivity, the LDP is NP-hard~\cite{hornakova2020lifted} in contrast to DP which is reducible to the minimum cost flow.
Below, we recapitulate the  formulation of LDP from~\cite{hornakova2020lifted}.

\subsection{Notation and Definitions.}
\begin{description} 
    \item[Flow network:] a directed acyclic graph $G = (V,E)$.
    \item[Start and terminal:] nodes $s,t \in V$.
    \item[Lifted graph:] a directed acyclic graph $G' = (V',E')$, where $V' = V \backslash\{s,t\}$.
    \item[The set of paths]
 starting at $v$ and ending in $w$ is
\begin{equation}
    vw\mhyphen\text{paths}(G) = \left\{ (v_1 v_2,\ldots,v_{l-1} v_l) : \begin{array}{c}v_i v_{i+1} \in E,\\ v_1 = v, v_l = w \end{array} \right\}\,.
\end{equation}
For a $vw\mhyphen$path $P$ its edge set is $P_E$ and its node set is $P_V$.

    \item[Reachability relation] for two nodes $v,w \in V$ is defined as $vw \in \mathcal{R}_G\Leftrightarrow vw\mhyphen\text{paths}(G) \neq \emptyset$.
    We assume that it is reflexive and $su \in \mathcal{R}_G, ut \in \mathcal{R}_G$  $\forall u \in V$, i.e.\ all nodes can be reached from  $s$ and all nodes can reach the sink node $t$.%
    % \item[Reachability relation] for two nodes $v,w \in V$ is defined as $vw \in \mathcal{R}_G\Leftrightarrow vw\mhyphen\text{paths}(G) \neq \emptyset$.
    % We assume that $su \in \mathcal{R}_G, ut \in \mathcal{R}_G$  $\forall u \in V$, i.e.\ all nodes can be reached from the source node $s$ and all nodes can reach the sink node $t$.
%
    \item[Flow variables:]
Variables $y \in \{0,1\}^E$ have value $1$ if flow passes through the respective edges.
%     \item[Node variables:]
% Node variables $z \in \{0,1\}^V$ denote flow passing through each node.
    \item[Node variables] $z \in \{0,1\}^V$ denote flow passing through each node. Values 0/1   forces paths to be node-disjoint.
    \item[Variables of the lifted edges] $E'$ are denoted by $y'\in\{0,1\}^{E'}$.
$y'_{vw}=1$ signifies that nodes $v$ and $w$ are connected via the flow $y$ in $G$.
Formally, 
\begin{equation}
\label{eq:lifted-edge-def}
y'_{vw} = 1 \Leftrightarrow \exists P\in vw\mhyphen\text{paths}(G): \forall ij\in P_E: y_{ij}=1 \,.
\end{equation}
\end{description}

\myparagraph{Lifted disjoint paths problem.}
Given edge costs $c \in \R^E$, node cost $d \in \R^V$ in flow network $G$ and edge cost $c' \in \R^{E'}$ for the lifted graph $G'$ the lifted disjoint paths problem is
\begin{equation}
\label{eq:lifted-disjoint-paths-problem}
\begin{array}{rl}
    \min\limits_{\substack{y \in \{0,1\}^E, y' \in \{0,1\}^{E'},\\ z \in \{0,1\}^{V}}} & \la c, y \ra + \la c',y' \ra + \la d, z\ra \\
    \text{s.t.} & y \text{ node-disjoint } s,t \text{-flow in } G, \\
                &z \text{ flow through nodes of }G\\
                & y, y' \text{ feasible according to } \eqref{eq:lifted-edge-def}
    \end{array}
\end{equation}

% Set $E'$ can be arbitrary.  Due to Formula~\eqref{eq:lifted-edge-def}, it makes sense to create a lifted edge $vw$ only if $vw\in \mathcal{R}_G$. Also, note that lifted edges between consecutive frames are not useful. We describe our choice in Sec.~\ref{sec:graph_construction}. 

Set $E'$ can be arbitrary.  It makes sense to create a~lifted edge $vw$ only if $vw\in \mathcal{R}_G$ due to Formula~\eqref{eq:lifted-edge-def} and only if $v$ and $w$ do not belong to neighboring frames. We describe our choice in Sec.~\ref{sec:graph_construction}. 

Other notation and abbreviations are in Appendix~\ref{ap:notation}.

% \begin{figure}
%     \centering
%     \includegraphics[width=0.4\textwidth]{illustrationNewNoT.pdf}
%     \caption{
%     Exemplary LDP: Black edges form the flow network $G$, \textcolor{blue}{blue} edges the lifted graph $G'$.
%     Bold edges have value $1$, while dashed ones are inactive and have value $0$.
%     The solution induces by the bold edges is feasible to constraints~\eqref{eq:lifted-disjoint-paths-problem}.
%     }
%     \label{fig:problemIllustration}
% \end{figure}

\roberto{Figure could be removed, there is nothing new after LDP..}
	\section{Lagrange Decomposition Algorithm for LDP}
\label{sec:Lagrange-decomposition}
Below we recapitulate  
Lagrange decomposition and the message passing primitive used in our algorithm (Sec.~\ref{sec:lagrange-decomposition-sub}).
Then, we propose a decomposition of the LDP problem~\eqref{eq:lifted-disjoint-paths-problem} into smaller but tractable subproblems (Sec.~\ref{sec:inout-subproblems}-\ref{sec:cut-subproblem}).
This decomposition is a dual task to an LP-relaxation of~\eqref{eq:lifted-disjoint-paths-problem}. Therefore, it provides a~lower bound that is iteratively increased by the message passing. We solve Problem~\eqref{eq:lifted-disjoint-paths-problem} in a simplified version of Lagrange decomposition framework developed in~\cite{Swoboda_2017_CVPR}.
Our heuristic for obtaining primal solutions uses the dual costs from the subproblems (Sec.~\ref{sec:primal-rounding}).
% The dual costs from the subproblems are used in our heuristic for obtaining primal solutions (Section~\ref{sec:primal-rounding}).
\subsection{Lagrange Decomposition}
\label{sec:lagrange-decomposition-sub}
We have an optimization problem $\min_{x \in \mathcal{X}} \langle c, x \rangle$ where $\mathcal{X}\subseteq \{0,1\}^n$ is a feasible set and $c\in \mathbb{R}^n$ is the objective vector. Its Lagrange decomposition is given by 
a~set of \emph{subproblems} 
$\mathcal{S}$
with associated \emph{feasible sets} 
$\mathcal{X}^{\mathsf{s}}\subseteq \{0,1\}^{d_\mathsf{s}}$ for each $\mathsf{s}\in \mathcal{S}$. 
Each coordinate $i$ of $\mathcal{X}^{\mathsf{s}}$ corresponds to one coordinate  of $\mathcal{X}$ via an injection $\pi_{\mathsf{s}}:[d_{\mathsf{s}}]\rightarrow [n]$
alternatively represented by a matrix
$A^{\mathsf{s}}\in \{0,1\}^{d_{\mathsf{s},n}}$ where $(A^{\mathsf{s}})_{ij}=1\Leftrightarrow \pi_\mathsf{s}(i)=j$.
%We denote the image of $\pi_{s}$ as $\img(\pi_s) = \{\pi_s(i): i \in [d_s]\}$.
% Moreover, $A_{\mathsf{s}}\in \{0,1\}^{d_{\mathsf{s}},n}$ defines selection of subproblem coordinates $ (A_{\mathsf{s}})_{ij}=1\Leftrightarrow \pi_\mathsf{s}(i)=j$.
For each pair of subproblems $\mathsf{s},\mathsf{s'}\in \mathcal{S}$ that contain a pair of coordinates $i,j$ such that $\pi_{\mathsf{s}}(i)=\pi_{\mathsf{s'}}(j)$, we have a~\emph{coupling constraint} $x^\mathsf{s}_i = x^{\mathsf{s}'}_j$ for each  $x^\mathsf{s} \in \mathcal{X}^\mathsf{s}$, $x^{\mathsf{s}'} \in \mathcal{X}^{\mathsf{s}'}$.

We require that every feasible solution $x\in \mathcal{X}$ is feasible for the subproblems, i.e.\ $\forall x\in \mathcal{X},\forall \mathsf{s}\in \mathcal{S}: A^{\mathsf{s}}x \in \mathcal{X}^s$.

We require that the objectives of subproblems are equivalent to the original objective, i.e.\ 
$\la c, x \ra = \sum_{s \in \mathcal{S}} \la \theta^s, A^{\mathsf{s}}x \rangle$ $\forall x \in \mathcal{X}$. Here, $\theta^\mathsf{s} \in \R^{d_{\mathsf{s}}}$ defines the \emph{objective} of subproblem $\mathsf{s}$.

The \emph{lower bound} of the Lagrange decomposition given the costs $\theta^{\mathsf{s}}$ for each $\mathsf{s}\in \mathcal{S}$ is 
\begin{equation}
\label{eq:lower-bound}
    \sum_{\mathsf{s} \in \mathcal{S}} \min_{x^\mathsf{s} \in \mathcal{X}^\mathsf{s}} \la \theta^\mathsf{s}, x^\mathsf{s} \ra\,.
\end{equation}

Given coupling constraint $x^\mathsf{s}_i = x^{\mathsf{s}'}_j$  and $\gamma \in \R$, a  sequence of operations of the form
$ \theta^\mathsf{s}_i \pluseq \gamma,\ \theta^{\mathsf{s}'}_j \minuseq \gamma$ is called a~\emph{reparametrization}.

% \begin{definition}[Lagrange decomposition]
% A Lagrange decomposition is given by a set of \emph{subproblems} 
% $\mathcal{S}$
% with associated \emph{feasible sets} 
% $\mathcal{X}^{\mathsf{s}}\subseteq \{0,1\}^{\abs{X^\mathsf{s}}}$ for each $ \mathsf{s} \in \mathcal{S}$ and
% \emph{coupling constraints} $\mathcal{C}^{\mathsf{s},\mathsf{s}'}$ for pairs of subproblems $\mathsf{s},\mathsf{s}' \in \mathcal{S}$ constraining $x^\mathsf{s} \in \mathcal{X}^\mathsf{s}, x^{\mathsf{s}'} \in \mathcal{X}^{\mathsf{s}'}$  to  $x^\mathsf{s}_i = x^{\mathsf{s}'}_j$ for all variable pairs $(i,j) \in \mathcal{C}^{\mathsf{s},\mathsf{s}'}$.
% For every subproblem an \emph{objective} is given by $\theta^\mathsf{s} \in \R^{\abs{\mathsf{s}}}$.
% Its \emph{lower bound} is given by
% \begin{equation}
% \label{eq:lower-bound}
%     \sum_{\mathsf{s} \in \mathcal{S}} \min_{x^\mathsf{s} \in \mathcal{X}^\mathsf{s}} \la \theta^\mathsf{s}, x^\mathsf{s} \ra\,.
% \end{equation}
% A sequence of operations of the form
% \begin{equation}
% \label{eq:reparametrization}
%     \theta^\mathsf{s}_i \pluseq \gamma,\quad 
%     \theta^{\mathsf{s}'}_j \minuseq \gamma
% \end{equation}
% for coupling constraint $x^\mathsf{s}_i = x^{\mathsf{s}'}_j$  and $\gamma \in \R$ is called a \emph{reparametrization}.
% \end{definition}
Feasible primal solutions are invariant under reparametrizations but the lower bound~\eqref{eq:lower-bound} is not.
%Below, we will explore reparametrization updates that are monotonically non-decreasing in the lower bound.
The optimum of the dual lower bound equals to the optimum of a convex relaxation of the original problem, see~\cite{guignard1987lagrangean}.

\myparagraph{Min-marginal message passing.}
Below, we describe reparametrization updates monotonically non-decreasing in the lower bound based on \emph{min-marginals}.
Given a variable $x^\mathsf{s}_i$ of a subproblem $\mathsf{s} \in  S$, the associated \emph{min-marginal} is
\begin{equation}
    \label{eq:min-marginal}
    m^{\mathsf{s}}_{i} =
    \min\limits_{x^\mathsf{s} \in \mathcal{X}^\mathsf{s}: x^\mathsf{s}_i = 1} \langle  \theta^{\mathsf{s}}, x^\mathsf{s} \rangle
    - \min\limits_{x^\mathsf{s} \in \mathcal{X}^\mathsf{s}: x^\mathsf{s}_i = 0} \langle  \theta^{\mathsf{s}}, x^\mathsf{s}  \rangle 
\end{equation}
i.e.\ the difference between the optimal solutions with the chosen variable set to $1$ resp.\ $0$.

\begin{proposition}[\cite{Swoboda_2017_CVPR}]
\label{prop:reparametrization}
Given a coupling constraints $x^\mathsf{s}_i = x^{\mathsf{s}'}_j$ and $\omega \in [0,1]$ the following operation is non-decreasing w.r.t.\ the dual lower bound~\eqref{eq:lower-bound}
    \begin{align}
        \label{eq:min-marginal-update}
        \theta^{\mathsf{s}}_{i} &\minuseq \omega \cdot m^{\mathsf{s}}_{i},& 
\theta^{\mathsf{s}'}_{j} &\pluseq \omega \cdot m^{\mathsf{s}}_{i}\,.
\end{align}

\end{proposition}

\myparagraph{The goal of reparametrization} is two-fold. 
(i)~Improving the objective lower bound to know how far our solution is from the optimum.
% (i) Obtaining the best possible lower bound of the objective function tells us how far our solution is from the optimum.
(ii)~Using reparametrized costs as the input for our primal heuristic yields high-quality primal solutions. The key components are efficient computations of (i)~optima of subproblems for obtaining lower bound~\eqref{eq:lower-bound}, (ii)~constrained optima for obtaining min-marginals \eqref{eq:min-marginal} and (iii)~a~primal heuristic using the reparametrized costs (Sec.~\ref{sec:primal-rounding}).
Lagrange decomposition has been used for other problems but the subproblem decomposition and minimization procedures are problem specific.
Therefore, developing them for LDP is an important contribution for solving LDP in a scalable way while keeping a~small gap to an optimum.
\subsection{Inflow and Outflow Subproblems}\label{sec:inout-subproblems}
For each node $v \in V$ of the flow graph, we introduce two subproblems:
An inflow and an outflow subproblem.
%one denoted by $x^{in}_v$ and $x^{out}_v$.
The subproblems contain all incoming resp.\ outgoing edges of node $v$ together with the corresponding node.
Formally, inflow resp.\ outflow subproblems contain the edges $ \delta^-_E(v) \cup \delta^-_{E'}(v)$, resp. $ \delta^+_E(w) \cup \delta^+_{E'}(w)\,.$
% \begin{align}
%     &\delta^+_E(v) \cup \delta^+_{E'}(v), & \delta^-_E(w) \cup \delta^-_{E'}(w)\,.
% \end{align}
Here, we adopt the standard notation where $\delta^-_E(v)$, resp.\ $\delta^+_E(v)$ denote all base edges incoming to $v$, resp. outgoing from $v$. 
Similarly, $\delta^-_{E'}(v), \delta^+_{E'}(v)$ denote lifted edges incoming to, resp.\ outgoing from $v$.

\myparagraph{The feasible set} $\mathcal{X}^{out}_v$ of the outflow subproblem for node $v$ is defined as
\begin{equation}
\label{eq:outflow-feasible-set}
    \left\{
\begin{array}{l}
   %(x,y,y') \in \{0,1\} \times \{0,1\}^{\delta^+_E(v)} \times \{0,1\}^{\delta^+_{E'}(v)} : \\
   z^{out}_v \in \{0,1\}, y^{out} \in \{0,1\}^{\delta^+_E(v)}, y'^{out} \in \{0,1\}^{\delta^+_{E'}(v)} : \\
   \ \ (z^{out}_v,y^{out},y'^{out}) = \mathbb{0}\ \vee \\
     \ \ \exists P \in vt\mhyphen\text{paths}(G) \text{ s.t. }
     \begin{array}[t]{c}
     z^{out}_v = 1 \\
     y^{out}_{vw} = 1 \Leftrightarrow vw \in P_E\\ 
     y'^{out}_{vu} = 1 \Leftrightarrow u \in P_V
     \end{array}
  \end{array}
    \right\}\,.
  \end{equation}
%The above constraints restrict the variables as follows:
Consequently, either there is no flow going through  vertex $v$ and all base and lifted edges
%adjacent to $v$
have label zero.
Alternatively, there exists a $vt\mhyphen$path $P$ in $G$ labeled by one.  
In this case the base edge adjacent to $v$ corresponding to the first edge in $P$ is one.
All lifted edges connecting $v$ with vertices of $P$ also have value one.
All other base and lifted edges are zero.
Each feasible solution of the outflow subproblem can be represented by a path $vt\mhyphen$path $P$.
The feasible set of the inflow subproblem  $\mathcal{X}^{in}_v$ is defined analogously.
We sometimes omit the  superscipts $out$ for better readability.

\myparagraph{Constraints between inflow and outflow subproblems.}
For node variables, we add the constraint $
    z^{in}_v = z^{out}_v$.
For an edge $vw \in E \cup E'$ we require the shared edge in the outflow subproblem of $v$ and in the inflow subproblem for $w$ to agree, i.e. $
        y_{vw}^{out} = y_{vw}^{in}$ if $vw \in E$ and  $    {y'}_{vw}^{out} = {y'}_{vw}^{in}$ if $vw \in E'$.

\myparagraph{Optimization of in- and outflow subproblems.} Given costs $\theta^{out}$, the optimal solution of an outflow problem for node $v$ can be computed by depth-first search on the subgraph defined by the vertices reachable from $v$.

\noindent The algorithms rely on the following data structures:
%\vspace{-\topsep}
\begin{itemize}
    \item lifted\_costs$[u]$ contains the minimum cost of all $ut\mhyphen$paths w.r.t.\ to costs of all lifted edges connecting $v$ with the vertices of the path. 
    \item next$[u]$ contains the best neighbor of vertex $u$ w.r.t.\ values in lifted\_cost. That is, next$[u]=\argmin_{w:uw\in \delta^+_E(u)}\text{lifted\_cost}[w]$
\end{itemize}
\vspace{-\topsep}

\begin{algorithm}
 \caption{Opt-Out-Cost}  
 \label{alg:outflow-minimization}
 \textbf{Input} start vertex $v$, edge costs $\tilde{\theta}$\\
 \textbf{Output} optimal value $\mathrm{opt}$, $\text{lifted\_cost}$  $\forall w:vw \in \delta^+_E(v)$ optimal solution for $vw$ active $\alpha_{vw}$
    \begin{algorithmic}[1]
    \For{$u\in V:vu\in\mathcal{R}_{G}$}
    %\State $\text{lifted\_cost}[u] = \infty$, $\text{visited}[u] = false$, next$[u]=t$
    \State $\text{lifted\_cost}[u] = \infty$, next$[u]=\emptyset$
    \EndFor
    %\State$\text{lifted\_cost}[r] =0 $, $\text{visited}[r] = true$, next$[r]=t$
    \State$\text{lifted\_cost}[t] =0 $,  next$[t]=t$
    \State Lifted-Cost-DFS-Out($v,v,\tilde{\theta},\text{lifted\_cost},\text{next}$)
    \State $\forall w:vw \in \delta^+_E(v): \alpha_{vw} =  \tilde{\theta}_v + \tilde{\theta}_{vw} + \text{lifted\_cost}[w] $
    \State $\mathrm{opt} = \min(\min_{vw \in \delta^+_E(v)} \alpha_{vw}, 0)$
    % \State $opt' = \min_{w:vw \in \delta^+_E(v)} \tilde{\theta}_v + \tilde{\theta}_{vw} + \text{lifted\_cost}[w] $
    % \State $opt = \min(opt', 0)$
    \end{algorithmic}
\end{algorithm}
\roberto{To differentiate variables from functions other latex algorithm packages use the style \textbackslash texttt, for instance \texttt{Lifted-Cost-DFS}.}

\begin{algorithm}
\caption{Lifted-Cost-DFS-Out}
\label{alg:outflow-minimization-dfs}
 \textbf{Input} $v,u,\tilde{\theta},\text{lifted\_cost},\text{next}$\\
 \textbf{Output} $\text{lifted\_cost},\text{next}$
\begin{algorithmic}[1]
\State $\alpha=0$
\For{$uw \in \delta^+_{E}(u)$}
%\If{$\text{visited}[w] = false$}
\IIf{$\text{next}[w] = \emptyset$} Lifted-Cost-DFS-Out($v,w,\tilde{\theta}$)
\If{lifted\_cost$[w]<\alpha$}
%\State $m=\min \{\text{lifted\_cost}[w],m\}$
\State $\alpha=\text{lifted\_cost}[w]$, $\text{next}[u]=w$
\EndIf
\EndFor
\IIf{next$[u]=\emptyset$} next$[u]=t$ %\EndIIf
\State $\text{lifted\_cost}[u]=\alpha+\tilde{\theta}'_{vu}$
%\State $\text{visited}[u] = true$\;
\end{algorithmic}
\end{algorithm}

Alg.~\ref{alg:outflow-minimization} and~\ref{alg:outflow-minimization-dfs} give a general dept first search (DFS) procedure that, given a vertex $v$, computes optimal 
%Alg.~\ref{alg:outflow-minimization} and~\ref{alg:outflow-minimization-dfs} give a general DFS-procedure that, given a vertex $v$, computes optimal
paths from all vertices reachable from $v$.
Alg.~\ref{alg:outflow-minimization} takes as input vertex $v$ and edge costs $\tilde{\theta}$. Its subroutine Alg.~\ref{alg:outflow-minimization-dfs} computes recursively for each vertex $u$ reachable from $v$ the value  $\text{lifted\_cost}[u]$.
The overall optimal cost $\min_{(z,y,y') \in \mathcal{X}^{out}_v} \la \tilde{\theta}, (z,y,y') \rangle$ of the subproblem is given by the minimum of node and base edge and lifted edges costs $\min_{vu\in \delta^+_{E}(v)}\tilde{\theta}^{out}_v + \tilde{\theta}^{out}_{vu} + \text{lifted\_cost}[u]$.
We achieve linear complexity by exploiting that subpaths of minimum cost paths are minimal as well.
The optimization for the inflow subproblem is analogous.

\myparagraph{Message passing for in- and outflow subproblems.}
We could compute one min-marginal~\eqref{eq:min-marginal}  by adapting Alg.~\ref{alg:outflow-minimization} and forcing an edge to be taken or not.
However, computing min-marginals one-by-one with performing operation~\eqref{eq:min-marginal-update} would be inefficient, since it would involve calling Alg.~\ref{alg:outflow-minimization} $\mathcal{O}(\abs{\delta^+_E(v)} + \abs{\delta^+_{E'}(v))}$ times.
Therefore, we present efficient algorithms for computing a sequence of min-marginals in Appendix~\ref{sec:outflow-subproblem-min-marginals}.
The procedures save computations by choosing the order of edges for computing min-marginals suitably and reuse previous calculations.

\subsection{Path Subproblems}\label{sec:path-subproblem}
The subproblem contains a~lifted edge $vw$ and a~path $P$ from $v$ to $w$ consisting of both base and lifted edges. They reflect that 
(i)~lifted edge $vw$ must be labelled $1$ if there exists an active path between $v$ and $w$, and
(ii)~there cannot be exactly one inactive lifted edge within path $P$ if $vw$ is active. The reason is that the inactive lifted edge divides $P$ into two segments that must be disconnected. This is contradictory to activating lifted edge $vw$  because it indicates a connection between $v$ and $w$.
% Path subproblems reflect that 
% (i)~a lifted edge must be labelled $1$ if there exists an active path between the lifted edge's endpoints and
% (ii)~there cannot be exactly one inactive lifted edge within a path if the lifted edge between the first and the last vertex of the path is active.
Path subproblems are similar to cycle inequalities for the multicut~\cite{chopra1993partition}.
% A path subproblem is composed of a lifted edge $vw$ and a $vw\mhyphen$path $P$ consisting of both base and lifted edges. %in a union of the base graph and the lifted graph $G\cup G'=(V,E\cup E')$.

In order to distinguish between base and lifted edges of path $P$, we use notation $P_E=P\cap E$ and $P_{E'}=P\cap E'$. 
For the purpose of defining the feasible solutions of path subproblems, we define strong base edges $E_0=\{vw\in E|vw\mhyphen\text{paths}(G)=\{vw\}\}$.
That is, base edge $vw$ is strong iff there exists no other $vw\mhyphen$path in graph $G$ than  $vw$ itself.
\todo[inline]{Use $:$ instead of $|$ in the set?}
\andrea{I actually prefer $|$}
\roberto{Both are fine. I just meant to make it consistent. I saw sometimes $:$ and sometimes $|$, but may be this is gone now..}

\myparagraph{The feasible set} $\mathcal{X}^P$ of the path subproblem for $vw\mhyphen$path $P$ is defined as 
\begin{align}
\nonumber y\in &\{0,1\}^{P_E},y'\in \{0,1\}^{P_{E'}\cup\{ vw\}}:\\
\label{eq:path-subproblem-eq-lifted}
\forall& kl\in P_{E'}\cup \{vw\}: \\
\nonumber
&\sum_{ij\in P_E} (1-y_{ij})+\sum_{ij\in P_{E'}\cup \{vw\}\setminus \{kl\}} (1-y'_{ij})\geq 1-y'_{kl}\,,\\
\label{eq:path-subproblem-eq-strong}
\forall& kl\in P_{E}\cap E_0:\\ 
\nonumber &\sum_{ij\in P_E\setminus kl} (1-y_{ij})+\sum_{ij\in P_{E'}\cup \{vw\}} (1-y_{ij})\geq 1-y_{kl}\,. 
\end{align}

Equation~\eqref{eq:path-subproblem-eq-lifted} requires that a lifted edge in $P_{E'}$ or $vw$ can be zero only if at least one other edge of the subproblem is zero.
%The same holds true for all strong base edges, see
Equation~\eqref{eq:path-subproblem-eq-strong} enforces the same for strong base edges.

\myparagraph{The optimization of path subproblems } 
%The optimization of the path subproblem
is detailed in Alg.~\ref{alg:path-subproblem-optimization} in the Appendix. The principle is as follows.
It checks whether 
there exists exactly one positive edge and whether it is either a lifted or a strong base edge.
If so, the optimal solution is either 
(i)~all edges except the two largest ones or
(ii)~all edges, whichever gives smaller objective value.
If the above condition does not hold, the optimal solution can be chosen to contain all negative edges.

We use a~variation  of the path optimization algorithm with an edge fixed to $0$ or $1$ for computing min-marginals.
%A variation of the path optimization algorithm with an edge fixed to $0$ or $1$ is used for computing min-marginals

\myparagraph{Cutting plane.}
Since there are exponentially many path subproblems, we add during the optimization only those that improve the relaxation.
 Details are in Appendix~\ref{sec:cutting-planes-path-subproblems}.

\subsection{Cut Subproblems}\label{sec:cut-subproblem}
The purpose of a cut subproblem is to reflect that a lifted edge $uv$ must be labelled 0 if there exists a cut of base edges  that separate $u$ and $v$ ($uv\mhyphen$cut) all labelled $0$.

\myparagraph{The feasible set.}
A cut subproblem consists of a lifted edge $uv$ and a $uv\mhyphen$cut $C = \{ij \in E| i \in A, j \in B\}$ where $A,B \subset V$  with $A \cap B = \emptyset$.
The space of feasible solutions $\mathcal{X}^C$ is defined as 
\begin{align}
\nonumber 
y'_{uv} &\in \{0,1\}, y \in \{0,1\}^C: \quad y'_{uv} \leq \sum_{ij\in C}y_{ij}\,,\\
\nonumber &\forall i\in A: \sum_{ij\in C} y_{ij}\leq 1 \,,\quad \forall j\in B: \sum_{ij\in C} y_{ij}\leq 1\,, \\
 &uv\in C \Rightarrow y'_{uv}\geq y_{uv}\,.
\end{align}
The constraints stipulate that
(i)~the lifted edge $uv$ is $0$ if all the edges in the cut are $0$,
(ii)~there exists at most one active outgoing resp.\ incoming edge for every vertex in $A$ resp.\ $B$ and
(iii)~if there is also base edge $uv\in C$ then whenever it is active, the lifted  edge $uv$ must be active.

\begin{algorithm}
\caption{Cut-Subproblem-Optimization}
\label{alg:cut-subproblem-optimization}
 \textbf{Input} Edge costs $\theta^C$\\
 \textbf{Output} optimal value $\mathrm{opt}$ of subproblem.
    \begin{algorithmic}[1]
%\State    \Comment{Construct linear assignment problem}\\
\State Define $\psi \in \R^{{A} \times {B}}$:
\State $\psi_{ij} = \begin{cases}
    \theta^C_{uv} + \theta'^C_{uv},& \text{if }ij=uv\wedge uv\in C \wedge \theta'^C_{uv} > 0 \\
    \infty,& \text{if }ij \notin C\\
    \theta^C_{ij},& \text{otherwise} 
\end{cases}$
 %  \State $\forall ij\in C: \psi_{ij} = \theta^C_{ij} $
 %  \State $\forall ij\in A\times B\setminus C: \psi_{ij}=0$
 %  \IIf{$uv\in C\wedge \theta^C_{uv}>0$} $\psi^C_{uv}+=\theta'^C_{uv}$ 
    \State $z^* \in \argmin\limits_{z \in \{0,1\}^{{A}\times{B}}} \sum\limits_{i\in A}\sum\limits_{j\in B} \psi_{ij}z_{ij}$,\ s.t.\  $z \mathbbmss{1} \leq \mathbbmss{1}, z^\top \mathbbmss{1} \leq \mathbbmss{1}$
    \State $\mathrm{opt}=\sum_{ij\in C}\psi_{ij} z^*_{ij}$
    \IIf{$\theta'^C_{uv}\geq 0$} return $\mathrm{opt}$
   % \If{$\forall ij\in A\times B: z_{ij}=0$}
    \If{$\exists kl\in C: z_{kl}=1$}
    \State return $\mathrm{opt}+\theta'^C_{uv}$
   % \IIf{$\alpha< \theta^C_{uv}$} $opt=$
    \Else
        \State $\alpha =\min_{ij\in C}\theta^C_{ij}$
    \IIf{$ |\theta'^C_{uv}|> \alpha$} return $\theta'^C_{uv}+\alpha$
    \EElse return $\mathrm{opt}$ 
    \EndIf
    \end{algorithmic}
\end{algorithm}

\myparagraph{Optimization of a~cut subproblem}  with respect to feasible set $\mathcal{X}^C$ is given by Alg.~\ref{alg:cut-subproblem-optimization}. 
Its key is to solve a linear assignment problem (LAP) \cite{ahuja1988network} between vertex sets $A$ and $B$.
The assignment cost $\psi_{ij}$ for $(i,j)\in  A\times B$ is the cut edge cost $\theta^C_{ij}$ if edge $ij$ belongs to $C$ and $\infty$ otherwise. 
In the special case of $uv\mhyphen\text{cut }C$ containing base edge $uv$ and the lifted edge cost $\theta'^C_{uv}$ being positive, the assignment cost $\psi_{uv}$ is increased by $\theta'^C_{uv}$. 

A candidate optimal labeling of cut edges is given by values of LAP variables $z_{ij}$.
If $\theta'^C_{uv}\geq 0$, the optimal value found by the LAP is the optimal value of the cut subproblem.
If it is negative, we distinguish two cases: 
(i)~If a cut edge $kl$ labeled by one exists, the lifted edge cost $\theta'^C_{uv}$ is added to the optimal value of LAP.
(ii)~Otherwise, we inspect whether it is better to activate the smallest-cost cut edge and the lifted edge $uv$ or keep all edges inactive.

We use a~variation  of Alg.~\ref{alg:cut-subproblem-optimization} with an edge variable restricted to be either $0$ or $1$ for computing min-marginals.
%We obtain min marginals of the cut subproblem by modifying Alg,~\ref{alg:cut-subproblem-optimization} such that an edge variable is restricted to be zero or one.

\andrea{I changed the notation of node variables to $z$ because of confusion with $x$ in general subproblem. However, in Alg. \ref{alg:cut-subproblem-optimization} we have now confusion of $z$.}
% \begin{algorithm}
% \caption{Cut-Subproblem-Optimization}
% \label{alg:cut-subproblem-optimization}
%  \textbf{Input} Edge costs $\theta^C$\\
%  \textbf{Output} optimal value $opt$ of subproblem.
%     \begin{algorithmic}[1]
% %\State    \Comment{Construct linear assignment problem}\\
% \State Define $\psi \in \R^{{A} \times {B}}$:
% \State $\psi_{ij} = \begin{cases}
%     \theta^C_{uv} + \theta'^C_{uv},& uv \in C, \theta'^C_{uv} > 0 \\
%     \infty,& ij \notin C\\
%     \theta^C_{ij},& \text{otherwise} 
% \end{cases}$
%  %  \State $\forall ij\in C: \psi_{ij} = \theta^C_{ij} $
%  %  \State $\forall ij\in A\times B\setminus C: \psi_{ij}=0$
%  %  \IIf{$uv\in C\wedge \theta^C_{uv}>0$} $\psi^C_{uv}+=\theta'^C_{uv}$ 
%     \State $z^* \in \argmin\limits_{z \in \{0,1\}^{{A} \times {B}}} \la \psi, z\ra$\quad s.t.\ \quad $z \mathbbmss{1} \leq 1, z^\top \mathbbmss{1} \leq 1$
%     \State $\mathrm{opt}=\sum_{ij\in C}\psi_{ij} z^*_{ij}$
%     \IIf{$\theta'^C_{uv}\geq 0$} return $\mathrm{opt}$
%   % \If{$\forall ij\in A\times B: z_{ij}=0$}
%     \If{$\exists kl\in C: z_{kl}=1$}
%     \State return $\mathrm{opt}+\theta'^C_{uv}$
%   % \IIf{$\alpha< \theta^C_{uv}$} $opt=$
%     \Else
%         \State $\alpha =\min_{ij\in C}\theta^C_{ij}$
%     \IIf{$ |\theta'^C_{uv}|> \alpha$} return $\theta'^C_{uv}+\alpha$
%     \EElse return $\mathrm{opt}$ 
%     \EndIf
%     \end{algorithmic}
% \end{algorithm}

\myparagraph{Cutting plane.} There are exponentially many cut subproblems.
Therefore, we add only those that improve the lower bound.
See Appendix~\ref{sec:cutting-planes-cut-subproblems} for details.

% \myparagraph{Cutting plane.} Similarly as for the path subproblems, there are exponentially many cut subproblems.
% Therefore, we add only those that improve the lower bound.
% See Section~\ref{sec:cutting-planes-cut-subproblems} in the Appendix for details.

\subsection{Message Passing}

The overall algorithm for optimizing the Lagrange decomposition is Alg.~\ref{alg:message-passing} in the Appendix.
First, inflow and outflow subproblems are initialized for every node.
%of the original graph $G=(V,E)$.
Then, for a number of iterations or until convergence, costs for each subproblems are adjusted iteratively by computing min-marginals and adjusting the reparametrization  proportionally to the min-marginal's value.
Additionally, every $k$-th iteration additional path and cut subproblems are separated and added to the Lagrange decomposition.  

\myparagraph{Solver complexity.} 
We need $\mathcal{O}(|E^{inp}|)$ space where $E^{inp}$ are all edges before graph sparsification. 
% A detailed time complexity analysis per algorithm is provided in Appendix~7.13.
%We provide detailed time complexity analysis in Appendix~\ref{appendix:runtime} analyses time complexity.
The most time consuming is computing lifted edges min-marginals  for each in/outflow subproblem. Alg.~\ref{ap:alg:all-lifted-mm} computes them for one outflow subproblem and it is linear in the number of detections per frame. 
% Overall, the most time consuming procedure of the solver is Algorithm 6 with a quadratic runtime in the number of detections in a frame. 
This significantly improves the complexity of to the optimal LDP solver LifT, making LDP applicable to large problem instances. 
See Appendix~\ref{appendix:runtime} for details.
%Details in Appendix~\ref{appendix:runtime}.

% \myparagraph{Solver complexity.} 
% We need $\mathcal{O}(|E^{inp}|)$ space where $E^{inp}$ are edges provided to the solver before graph sparsification. 
% We provide detailed time complexity analysis in Appendix~\ref{appendix:runtime}.
% The most time consuming procedure is calling Alg.~\ref{ap:alg:all-lifted-mm} 
% %(min-marginals of lifted edges in in/outflow subproblems)
% for each vertex where Alg.~\ref{ap:alg:all-lifted-mm} is linear in the number of detections per frame. 
% % Overall, the most time consuming procedure of the solver is Algorithm 6 with a quadratic runtime in the number of detections in a frame. 
% This significantly improves the time complexity bound compared to the optimal LDP solver LifT, making LDP applicable to large problem instances. 

\subsection{Primal Rounding}\label{sec:primal-rounding}
For computing primal solutions we solve a minimum cost flow (MCF) problem on the base edges and improve this initial solution with a local search heuristic.

Without lifted edges, the disjoint paths problem is an instance of MCF, which can be efficiently optimized via combinatorial solvers like the successive shortest path solver that we employ~\cite{ahuja1988network}.
We enforce node disjoint paths via splitting each node $u\in V$ into two nodes $u^{in},u^{out}\in V^{mcf}$ in the MCF graph $G^{mcf}=(V^{mcf},E^{mcf})$, adding an additional edge $u^{in}u^{out}$ to $E^{mcf}$ and setting capacity $[0,1]$ on all edges $E^{mcf}$. 
%This ensures at most one incoming resp.\ outgoing edge via flow conservation.
%All edges in $E^{mcf}$ are directed and have capacities in $[0,1]$.
Each node except $s$ and $t$ has demand $0$.
%  Alg.~\ref{alg:mcf-init} details how edge costs for the MCF  are calculated from the edge costs of inflow and outflow subproblems via calling Alg.~\ref{alg:outflow-minimization}.
% We obtain the cost of each flow edge $u^{out}v^{in}$
% from the inflow subproblem of $v$ and the outflow subproblem of $u$ using their minima where edge $uv$ is active.
Alg.~\ref{alg:mcf-init} calculates MCF edge costs from in/outflow subproblems using Alg.~\ref{alg:outflow-minimization}.
We obtain the cost of each flow edge $u^{out}v^{in}$
from the inflow subproblem of $v$ and the outflow subproblem of $u$ using their minima where edge $uv$ is active. This combines well the cost from base and lifted edges. 

%contains the best cost achievable by activating $uv$ in the outflow  subproblem of $u$ plus the best cost achievable by activation $uv$ in the inflow subproblem of $v$.
%So that cost of each edge $vw\in E$ is synthesized from costs of all lifted edges adjacent to $v$ and $w$. 

We describe the local search heuristic for improving the MCF solution in Alg.~\ref{alg:post-process-primal} in the Appendix.
% The local search heuristic for improving the MCF solution is described in Algorithm~\ref{alg:post-process-primal} in the Appendix.
It works with sets of disjoint paths.
First, paths are split if this leads to a decrease in the objective.
Second, merges are explored.
If a merge of two paths is not possible, we iteratively check whether cutting off one node from the first path's end or the second paths's beginning makes the connection possible.
If yes and the connection is decreasing the objective, the nodes are cut off and the paths are connected.

\begin{algorithm}
\caption{Init-MCF}
\label{alg:mcf-init}
    \begin{algorithmic}[1]
    \State  \begin{varwidth}[t]{\linewidth}
    $\forall u \in V\backslash \{s,t\}$: \par
     $(o,lc,\alpha^{in})$=Opt-In-Cost$(u,\theta^{in}_u)$\par
     $(o,lc,\alpha^{out})$=Opt-Out-Cost$(u,\theta^{out}_u)$
    \end{varwidth}
    \State $\forall u \in V\backslash \{s,t\}:$ $\theta^{mcf}_{s u^{in}} = \alpha^{in}_{su}$, $\theta^{mcf}_{u^{out} t} = \alpha^{out}_{ut}$
    \State 
    \begin{varwidth}[t]{\linewidth}
    $\forall u \in \{uv \in E| u \neq s, v \neq t\}:$
    $\theta^{mcf}_{u^{out} v^{in}} = \alpha^{out}_{uv} + \alpha^{in}_{uv}$
    \end{varwidth}
    \end{algorithmic}
\end{algorithm}

	\section{Experiments}
\label{sec:experiments}
We integrate our LDP solver into an MOT system (Appendix, Fig. \ref{fig:overallframework}) and show
%demonstrate 
on 
%several 
challenging datasets that higher order MOT is scalable to big problem instances.
%Additional evaluations are provided in Appendix, 
In the next sections, we describe our experimental setup and present results. We clarify the edge cost calculation and construction of the base and the lifted graph and their sparsification. %The tracking system is shown in Figure \ref{fig:overallframework}. 

\subsection{Pairwise Costs}
\label{sec:cost_classifier}
We use multi layer perceptrons (MLP) to predict the likelihood that two detections belong to the same trajectory.
We divide the maximal frame distance into $20$ intervals of equal length and train one separate MLP for each set of frame distances.
 We transform the MLP output to the cost of the edge between the detections and use it in our objective~\eqref{eq:lifted-disjoint-paths-problem}.
Negative cost indicates that two detections belong to the same trajectory. Positive cost reflects the opposite.

%\noindent{\textbf{MLP architecture:}} 
\myparagraph{MLP architecture.}
Each MLP consists of a fully connected layer with the same number of neurons as the input size, followed by a LeakyReLU activation \cite{maasrectifier} and a~fully connected single neuron output layer. We add sigmoid activation in the training. 
We describe our spatial and visual features used as the input in the paragraphs below.

%\noindent{\textbf{Spatial Feature:}}
\myparagraph{Spatial feature} uses bounding box information of two detections $v$ and $w$. We align the boxes such that their centers overlap. The similarity feature $\sigma_{vw,\text{Spa}} \in [0, 1]$ is the intersection-over-union between two aligned boxes.

\myparagraph{Appearance feature.}
%\label{section:visual_feature}
We create an appearance feature $F_{v}$ for each detection $v$  by training the method  \cite{zheng2019joint}  on the training set of the respective benchmark and additional data from \cite{zheng2015scalable, wei2018person, ristani2016MTMC}. The similarity feature $\sigma_{vw,\text{App}}$ between detection $v$ and $w$ given by \mbox{$\sigma_{vw,\text{App}}:=\max\{0,\la F_v, F_{w}\ra\}$} is used. A higher value indicates a higher similarity. 

\myparagraph{Global context normalization.}
The two features $\sigma_{vw,\text{Spa}}$, $\sigma_{vw,\text{App}}$ depend entirely on the nodes $v$ and $w$. 
To include global context, we append several normalized versions of the two features to the edge feature vector, inspired by \cite{hornakova2020lifted}.
Both features $\sigma_{ij,*}$ of edge $ij$ undergo a~five-way normalization. 
In each case, the maximum feature value from a~relevant set of edges is selected as the normalization value.
The normalization is done by dividing the two features $\sigma_{ij,*}$ by each of their five normalization values. This yields 10 values.
Another set of 10 values for edge $ij$ is obtained by dividing $\sigma^2_{ij,*}$ by each of the five normalization values.
Together with the two unnormalized features $\sigma_{ij,*}$, edge feature vectors have length 22.
See Appendix~\ref{appendix:feature_scaling} for details.

\myparagraph{Training.} 
We iteratively train our MLP on batches $B$ containing sampled edges. To compensate the imbalance between true positive and true negative edges, we use an $\alpha$-balanced focal loss \cite{focalLoss} with $\gamma=1$. We define the $\alpha$-weight $\alpha^{(g,\Delta f)}$ to weight the correct classification of edge $vw$ with ground truth flow value $g_{vw} \in \{0,1\}$, time distance $\Delta f$ between $v$  in frame $f_{v}$ and $w$ in frame $f_{w}$, and value $g \in \{0,1\}$ via {$\alpha^{(g,\Delta f)} := 1/| \{vw \in E: |f_{v} -f_{w}| = \Delta f, g_{vw} = g \} |\, .$}
We optimize the classifier using Adam with $l_r=0.1$, $\beta_1=0.9$, $\beta_2=0.999$ and $\epsilon=10^{-8}$. 
To reduce complexity while maintaining  variety during training, we introduce an extended sampling. Given a frame $f$, we create batches $B(f)$ by sampling detections from a fixed sequence of frame shifts starting at frame $f$ ensuring that all temporal distances $\Delta f$ are present in $B(f)$ (details in Appendix \ref{appendix:sampling}). We then subsample the $k$-nearest detections to a random generated image position with $k=160$, which sensitizes  training to crowded scenes. We train the MLP for $3$ epochs with batches $B(f)$ for all frames $f$ of the dataset.

\subsection{Graph Construction}
\label{sec:graph_construction}
We create the base and the lifted graph edges between detections with time distance up to $2$ seconds.
% We create the base and the lifted graph with edges between detections within time distance up to $2$ seconds.
We also add an edge from source $s$, and to sink $t$ to each detection.
In order to reduce computational complexity, we apply sparsification on both base and lifted graph as described later. 
\andrea{I would move first sentence of the Cost paragraph to graph sparsification. The rest of the paragraph should go to Section Cost Classifier which would be renamed to "Cost" or something.}

\myparagraph{Costs.}
%For each edge $vw\in E \cup E'$, negative edge cost indicates that two detections belong to the same trajectory, while positive cost reflect the opposite. 
We obtain base and lifted costs $c$ and $c'$ from the same MLP classifier (Sec.~\ref{sec:cost_classifier}). 
Due to decreasing classification accuracy with increasing frame distance $\Delta f$, we multiply the costs by a decay weight $\omega_{\Delta f} \coloneqq (10 \cdot \Delta f + 0.1)^{-1}$, so that edges representing long temporal distances have lower weight.
Edges from $s$ and  to $t$ have costs zero. 
% Outgoing edges from $s$ and incoming edges to $t$ have their costs fixed to zero. 

Finally, we use simple heuristics to find pairs that are obviously matching or non-matching.
We set the corresponding costs to be high in absolute value, negative for matching and positive for non-matching, thereby inducing soft constrains. 
An obvious match is given by a nearly maximal feature similarity. Detection pairs are obviously non-matching, if the displacement between their bounding boxes is too high.
%Obviously disconnected detections are determined by implausible high displacements between detection boxes.
See  Appendix~\ref{appendix:high_confident} for details.
% Finally, we use simple heuristics to find pairs that should be obviously connected or obviously disconnected. We set the corresponding costs to be high in absolute value. These costs are negative for obvious joints and positive for obvious cuts, thereby inducing soft-constrains. Edges are obviously connected, if the corresponding feature similarity is nearly optimal. Detections are obviously disconnected, if the displacement between their bouning boxes is too high.
% %Obviously disconnected detections are determined by implausible high displacements between detection boxes.
% See Sec. \ref{appendix:high_confident} in the Appendix for details.

\myparagraph{Sparsification.}
%\myparagraph{Sparsification.}
The base edges are an intersection of two edge sets. The first set contains for every $v \in V'$ edges to its $3$ nearest (lowest-cost) neighbors from every subsequent time frame. 
%The second set contains for every $v \in V'$ edges to its $3$ nearest neighbors from every preceding  frame. 
The second set selects for every vertex the best edges to its preceding frames analogically.
%Edges longer than 6 frames  must additionally comply with maximal cost threshold $3.0$.
Moreover, edges longer than 6 frames  must have costs lower than $3.0$.
To avoid double counting of edge costs, we subsequently set costs of all base edges between non-consecutive frames to zero, so that only  lifted edges maintain the costs. 
If a~lifted edge has cost around zero, it is not discriminative and we remove it, unless it overlaps with a~(zero-valued) base edge. 
% For lifted graph, edges having cost around zero  are not discriminative. We remove such edge unless it overlaps with a~(zero-valued) base edge. 

\begin{table*}
\center
\def\arraystretch{0.9}%  1 is the default, change whatever you need
\small{
    \caption{Comparison of ApLift with the best performing solvers w.r.t.\ MOTA metric on the MOT challenge. $\uparrow$ higher is better,  $\downarrow$ lower is better. The two rightmost columns: average number of frames per sequence and the average number of detections per frame for dataset.}

    \begin{tabular}{c c  c c c c c c c c c c }

     \toprule
     & Method  & MOTA$\uparrow$  & IDF1$\uparrow$  & MT$\uparrow$  & ML$\downarrow$  & FP$\downarrow$ & FN$\downarrow$ & IDS$\downarrow$ & Frag$\downarrow$ & Frames & Density \\   
     \toprule
     
     \parbox[t]{3mm}{\multirow{3}{*}{\rotatebox[origin=c]{90}{MOT20}}} 
& ApLift (ours) & $\mathbf{58.9}$ & $56.5$ & $\mathbf{513}$ & $\mathbf{264}$ & $17739$ & $\mathbf{192736}$ & $2241$ & $2112$ & \multirow{3}{*}{1119.8} & \multirow{3}{*}{170.9} \\ % frames 4479
& MPNTrack~\cite{braso2020learning} & $57.6$ & $\mathbf{59.1}$ & $474$ & $279$ & $16953$ & $201384$ & $\mathbf{1210}$ & $\mathbf{1420}$ & & \\ 
& Tracktor++v2~\cite{bergmann2019tracking} & $52.6$ & $52.7$ & $365$ & $331$ & $\mathbf{6930}$ & $236680$ & $1648$ & $4374$ & & \\ 
     \midrule

     \parbox[t]{3mm}{\multirow{4}{*}{\rotatebox[origin=c]{90}{MOT17}}} 
& CTTrackPub~\cite{zhou2021tracking} &  $\mathbf{61.5}$ & $59.6$ & $621$ & $752 $ & $\mathbf{14076}$ & $200672$ & $2583$ & $4965$ & \multirow{4}{*}{845.6} &\multirow{4}{*}{31.8} \\ % frames 17757
& ApLift (ours) & $60.5$ & $\mathbf{65.6}$ & $\mathbf{798}$ & $\mathbf{728}$ & $30609$ & $\mathbf{190670}$ & $1709$ & $2672$ & & \\ 
& Lif\_T~\cite{hornakova2020lifted} &  $60.5$ &$ \mathbf{65.6} $& $637 $ & $791$  & $14966$ & $206619$ & $1189 $& $3476$ & & \\ 
& MPNTrack~\cite{braso2020learning} & $58.8$ & $61.7$ & $679$ & $788$ & $17416$ & $213594$ & $\mathbf{1185}$ & $\mathbf{2265}$ & &  \\

     \midrule
     
     \parbox[t]{3mm}{\multirow{4}{*}{\rotatebox[origin=c]{90}{MOT16}}} 
&ApLift (ours) & $\mathbf{61.7}$ & $\mathbf{66.1}$ & $\mathbf{260}$ & $\mathbf{237}$ & $9168$ & $\mathbf{60180}$ & $495$ & $802$ & \multirow{4}{*}{845.6} & \multirow{4}{*}{30.8}\\ % 5919 frames
&Lif\_T~\cite{hornakova2020lifted}  &  $61.3 $& $64.7$ & $205$ & $258$ & $4844$ & $65401$ & $389$ & $1034$ & & \\ 
& MPNTrack~\cite{braso2020learning} & $58.6$ & $61.7$ & $207$ & $258$ & $4949$ & $70252$ & $\mathbf{354}$ & $\mathbf{684}$ & & \\ 
& GSM~\cite{ijcai2020-74} & $57.0$ & $55.0$ & $167$ & $262$ & $\mathbf{4332}$ & $73573$ & $475$ & $859$ & & \\ 

     \midrule
     
     \parbox[t]{3mm}{\multirow{4}{*}{\rotatebox[origin=c]{90}{MOT15}}} 
& Lif\_T~\cite{hornakova2020lifted}  &  $\mathbf{52.5}$ & $\mathbf{60.0}$ & $244 $ & $186$ & $6837$ & $21610$ & $730$ & $1047$ & \multirow{4}{*}{525.7} & \multirow{4}{*}{10.8} \\ % 5783 frames
& MPNTrack~\cite{braso2020learning} & $ 51.5 $& $58.6$ & $225$ & $187$ & $7260$ & $21780$ & $\mathbf{375}$ & $\mathbf{872}$ & &  \\ 
& ApLift (ours) & $51.1$ & $59.0$ & $\mathbf{284}$ & $\mathbf{163}$ & $10070$ & $\mathbf{19288}$ & $677$ & $1022$ & & \\ 
& Tracktor15~\cite{bergmann2019tracking} & $ 44.1 $& $46.7$ & $130$ & $189$ & $\mathbf{6477}$ & $26577$ & $1318$ & $1790$ & &  \\ 

     \bottomrule
    \end{tabular}

\label{tab:mot}
}
\vspace{-5mm}
\end{table*}

%\vspace{-18.2mm}
\subsection{Inference}
For fair comparison to state of the art, we filter and refine detections using tracktor~\cite{bergmann2019tracking} as in \cite{hornakova2020lifted}. Different to \cite{hornakova2020lifted}, we apply tracktor to recover missing detections before running the solver.
% Before we apply our method, we follow \cite{hornakova2020lifted,braso2020learning} and use tracktor \cite{bergmann2019tracking} to filter and refine input detections. Different to \cite{hornakova2020lifted}, we apply the recovery before running the solver.

While we solve MOT15/16/17 on global graphs,  we solve MOT20 in time intervals in order to decrease memory consumption and runtime.
First, we solve the problem on non-overlapping adjacent intervals and fix the trajectories in the interval centers. Second, we solve the problem on a new set of intervals where each of them covers unassigned detections in two initial neighboring intervals and enables connections to the fixed trajectory fragments.
We use the maximal edge length of 50 frames in MOT20. Therefore, 150 is the minimal interval length such that all edges from a~detection are used when assigning the detection to a~trajectory. This way, the solver has sufficient context for making each decision.
Intervals longer than 200 frames increase the complexity significantly for MOT20, therefore we use interval length 150 in our experiments.

 % Details in Section \ref{sec:intervalSolution} in the Appendix. 

\myparagraph{Post-processing.} We use simple heuristics to check if base edges over long time gaps correspond to plausible motions, and split trajectories if necessary.
Finally, we use linear interpolation to recover missing detections within a trajectory. 
%See details in Appendix \ref{appendix:post_processing}).  
%Further details on inference can be found in the
Appendix~\ref{sec:inference} contains further details on inference.

% \myparagraph{Solver complexity} 
% We need $\mathcal{O}(|E^{inp}|)$ space where $E^{inp}$ are edges provided to the solver before graph sparsification. 
% % A detailed time complexity analysis per algorithm is provided in Appendix~7.13.
% We provide detailed time complexity analysis in Appendix~\ref{appendix:runtime}.
% The most time consuming procedure is calling Alg.~\ref{ap:alg:all-lifted-mm} for each vertex where Alg.~\ref{ap:alg:all-lifted-mm} is linear in the number of detections per frame. 
% % Overall, the most time consuming procedure of the solver is Algorithm 6 with a quadratic runtime in the number of detections in a frame. 
% This significantly improves the time complexity bound compared to the optimal LDP solver LifT, making LDP applicable to large problem instances. 
%To demonstrate this, we compare 
% We compare
% the runtime of our solver against  two step procedure of LifT for a~sample sequence in Table \ref{tabel:runtime}. 
% With increasing problem complexity, our solver outperforms LifT w.r.t.\ runtime while achieving similar results w.r.t. tracking metrics. Note that better objective value in terms of LDP~\eqref{eq:lifted-disjoint-paths-problem} does not automatically translate to better MOT metrics due to imperfect input costs. Therefore, we may sometimes observe slightly worse IDF1 score for better LDP solutions.  We provide comparison of our solver against optimal (one step) LifT on MOT17 train in Appendix~\ref{appendix:runtime}. 

\subsection{Tracking Evaluation}
 
We evaluate our method on four standard MOT benchmarks.
The MOT15/16/17 benchmarks \cite{MOTChallenge2015,MOT16} contain semi-crowded videos sequences filmed from a~static or a~moving camera. 
MOT20 \cite{MOTChallenge20} comprises crowded scenes with considerably higher number of  frames  and detections per frame, see Tab.~\ref{tab:mot}.
% MOT20 \cite{MOTChallenge20} comprises crowded scenes with an average number of  frames and average number of detections per frame considerably higher than for MOT15/16/17. 
% Note that the number of edges in our graphs grows quadratically with the number of detections per frame. Therefore,  it was crucial to make the tracker scalable to these massive data. 
The challenge does not come only with the data size. Detectors make more errors in crowded scenes due to frequent occlusions and appearance features are less discriminative as the distance of people to the camera is high. 
Using higher order information helps in this context. However, the number of edges in our graphs grows quadratically with the number of detections per frame. Therefore, it is crucial to make the tracker scalable to these massive data. 
We use the following ingredients to solve the problems: (i)~fast but accurate method for obtaining edge costs, (ii)~approximate LDP solver delivering high-quality results fast, (iii)~preprocessing heuristics, (iv)~interval solution keeping sufficient context for each decision. 

We use training data of the corresponding dataset for training and the public detections for training and test. 
\vspace{-6mm}
\begin{center}
\begin{table}
\setlength{\tabcolsep}{3pt}
\def\arraystretch{0.8}
%\caption{Ablations on MOT17 train without postprocessing.
\caption{Influence of lifted graph sparsification, message passing and using zero base costs on MOT17 train without postprocessing.%Message passing is disabled for \textit{w/o MP}, for (i) w/o Base edge costs set to zero (3) w/o Base edge costs set to zero and sparser lifted graph (4) w/o message passing}
}
\label{tab:ablations}
\small{
\begin{tabularx}{\columnwidth}{cccccccc}
\toprule
 \multirow{2}{*}{$E'$} & MP & Base & \multirow{2}{*}{IDF1$\uparrow$}&  \multirow{2}{*}{MOTA$\uparrow$}
& \multirow{2}{*}{FP$\downarrow$}& \multirow{2}{*}{FN$\downarrow$}&  \multirow{2}{*}{IDS$\downarrow$} \\
&steps & cost & & & & & \\
\midrule
Dense        & $82$ & Zero      & $\mathbf{71.0}$ & $\mathbf{66.3}$& $2826$  & $\mathbf{109263}$ & $1369$ \\
Dense        & $0$  & Zero      & $70.3$ & $\mathbf{66.3}$ & $2832$  & $109265$ & $1354$        \\
Dense        & $82$ & Orig.     & $69.8$ &$\mathbf{66.3}$& $\mathbf{2824}$  & $109266$ & $1355$    \\
Sparse       & $82$ & Orig.     & $69.1$ & $\mathbf{66.3}$ & $2825$  & $\mathbf{109263}$ & $\mathbf{1316}$     \\
\bottomrule
\end{tabularx}
}
\vspace{-2.5mm}
\end{table}
%\vspace*{-25mm}

\end{center}
\vspace{-8mm}
\begin{center}
\begin{table}
\def\arraystretch{0.8}
\small{
\caption{Runtime and IDF1 comparison of LDP solvers: ApLift (ours)  with 6, 11, 31 and 51 iterations and LifT\cite{hornakova2020lifted} (two step procedure) on first $n$ frames of sequence \textit{MOT20-01} from MOT20.}
\label{tabel:runtime}

\begin{tabular}{c l c c c c c}
%\cline{1-7}
\toprule

$n$ & Measure      & LifT & Our6 & Our11 & Our31 & Our51   \\ \toprule
%\cline{1-7}
%\multirow{2}{*}{30}     & IDF1         & 84.6 & 85.2   & 85.2    & 85.2    & 85.2   \\ % &  &  \\ %\cline{1-7}
%     & time {[}s{]} & 38   & 0      & 1       & 3       & 6    \\ \midrule%   &  &  \\ \cline{1-7}
\multirow{2}{*}{50}     & IDF1$\uparrow$         & $80.6$ & $\mathbf{83.3}$   & $\mathbf{83.3}$    & $81.5$    & $81.5$   \\ % &  &  \\ %\cline{1-7}
     & time {[}s{]} & 272  & 2      & 4       & 16      & 35     \\ \midrule% &  &  \\ \cline{1-7}
%\multirow{2}{*}{75}     & IDF1         & 79.4 & 82.6   & 82.6    & 81.5    & 81.5   \\ % &  &  \\ %\cline{1-7}
%     & time {[}s{]} & 396  & 8      & 14      & 53      & 121    \\ \midrule% &  &  \\ \cline{1-7}
\multirow{2}{*}{100}    & IDF1$\uparrow$         & $80.4$ & $\mathbf{82.5}$   & $\mathbf{82.5}$    & $81.6$    & $81.6$   \\ % &  &  \\ %\cline{1-7}
    & time {[}s{]} & $484$  & $14$     & $24$      & $97$      & $218$    \\ \midrule% &  &  \\ \cline{1-7}
\multirow{2}{*}{150}    & IDF1$\uparrow$         & $78.1$ & $\mathbf{81.0}$   & $\mathbf{81.0}$    & $79.8$    &$ 79.8$   \\ % &  &  \\ %\cline{1-7}
    & time {[}s{]} & $1058$ & $25$     & $46$      & $192$     & $431$    \\ \midrule% &  &  \\ \cline{1-7}
\multirow{2}{*}{200}    & IDF1$\uparrow$         & $73.2$ & $\mathbf{75.4}$   & $\mathbf{75.4}$    & $74.6$    & $74.6$  \\ %  &  &  \\ %\cline{1-7}
    & time {[}s{]} &  $2807$ & $36 $    & $66$      & $277$     & $616$   \\ \bottomrule% &  &  \\ \cline{1-7}
\end{tabular}
}
\vspace{-5.2mm}
\end{table}
\end{center}
\vspace{-10mm}

% We trained our method based on the training data of the corresponding dataset, and use the public detections for training and test. 

We compare our method using standard MOT metrics. MOTA \cite{motametric} and IDF1 \cite{IDF1metric} are considered  the most representative  as they incorporate other metrics (in particular recall and precision).
% We compare our method using the most representative standard MOT metrics MOTA \cite{motametric} and IDF1 \cite{IDF1metric}.
%Thereby, 
IDF1 is more penalized by inconsistent trajectories. We also report mostly tracked (MT) and mostly lost trajectories (ML), false negatives (FN) and false  positives (FP), ID switches (IDS) and fragmentations (Frag) as provided by the evaluation protocols \cite{motametric} of the benchmarks.

% The comparison to top peer-reviewed methods on test sets is shown in Table \ref{tab:mot}.

Tab.~\ref{tab:mot} shows the comparison to the best (w.r.t.\ MOTA)  peer-reviewed methods  on test sets.
Our approximate solver achieves almost the same results on  MOT15/16/17 as the optimal LDP solver~\cite{hornakova2020lifted}, while using simpler features.
%
% MOTA and IDF1 are considered the most relevant metrics in MOT as they incorporate other metrics (in particular recall and precision). Here, ApLift performs comparable or better than current state-of-the-art methods.
% Overall, our method performs on par with state-of-the-art on all evaluated benchmarks. 
%
Overall, our method performs on par with state  of the art on all evaluated benchmarks, especially in MOTA and IDF1. %that are considered the most relevant MOT metics.
Our complete results and videos are publicly available\footnote{\scriptsize \url{https://motchallenge.net/method/MOT=4031&chl=13}}.
% We trained our method based on the training data of the corresponding dataset, and use the public detections for training and test. 
%
The proposed method achieves overall low FN values but slightly high FP values. 
%It is hard to identify a~tracker component that influences a~specific metric.
%In our case, 
FP/FN are mostly affected by preprocessing the input detections and interpolation in the post-processing. The impact of post-processing (trajectory splits and interpolations) on MOT20, which causes FP but reduces FN and IDS, is analyzed in the Appendix (Tab.~\ref{tab:postprocessing}).  
% MOTA and IDF1 are considered the most relevant metrics in MOT as they incorporate other metrics (in particular recall and precision). Here, ApLift performs comparable or better than current state-of-the-art methods.

Tab.~\ref{tab:ablations} shows the influence of various settings on the performance of MOT17 train.
%In particular, using dense or sparsified lifted graph, using message passing and setting the base cost to zero (see Sec.~5.1) or keep the original values. 82 MP iterations to correspond to the setup from Tab.~\ref{tab:mot}.
While we usually set the base edge costs to zero (Sec.~\ref{sec:graph_construction}), we need to keep them when using the sparsified lifted graph.
% Note that we need to keep the base edge costs when using the sparsified lifted graph.
% While we set it to zero usually (Sec.~\ref{sec:graph_construction}).
%Sec.~\ref{sec:graph_construction} explains how base edge costs are set to zero otherwise.
Both, message passing and dense lifted edges improve IDF1 and IDS. However, MOTA, FN and FP remain almost unchanged.

Finally, we compare
the runtime of our solver against the two step version of LifT for a~sample sequence in Tab.~\ref{tabel:runtime}.  
With increasing problem complexity, our solver outperforms LifT w.r.t.\ runtime while achieving similar IDF1. 
Counter-intuitively, as we progress towards increasingly better optimization objective values, the tracking metrics can slightly decrease due to imperfect edge costs.
We compare our solver against optimal (one step) LifT on MOT17 train in Appendix~\ref{appendix:runtime}. 

%Note that better LDP objective value~\eqref{eq:lifted-disjoint-paths-problem} does not automatically lead to better MOT metrics due to imperfect input costs. Therefore, we may sometimes observe slightly worse IDF1 for better LDP solutions.  

	\section{Conclusion}
%We made the NP-hard LDP concept applicable for processing massive sequences of MOT20 by the combination of an approximate solver, efficient input costs and subdivision of data keeping sufficient context for each decision.
We demonstrated that the NP-hard LDP model is applicable  for processing massive sequences of MOT20. The combination of an approximate LDP solver, efficiently computable costs and subdivision of data keeping sufficient context for each decision make this possible.
 %We proved that the NP-hard LDP concept can be used for processing massive sequences of MOT20. The combination of an approximate LDP solver, input costs based on effective features and subdivision of data providing sufficiently broad context for each  decision make this possible.

% I

% want

% too 

% see

% how

% much 

% space 

% we

% have

% to

% write

% the

% conclusion
	\section{Acknowledgements}
This work was supported by the Federal Ministry of Education and Research (BMBF), Germany, under the project LeibnizKILabor (grant no.\ 01DD20003), the Center for Digital Innovations
(ZDIN) and the Deutsche Forschungsgemeinschaft (DFG) under Germany’s Excellence Strategy
within the Cluster of Excellence PhoenixD (EXC 2122).

% This work has been supported by the Federal Ministry of Education and Research (BMBF), Germany, under the project LeibnizKILabor (grant no. 01DD20003), the Center for Digital Innovations
% (ZDIN) and the Deutsche Forschungsgemeinschaft (DFG) under Germany’s Excellence Strategy
% within the Cluster of Excellence PhoenixD (EXC 2122).

	%\clearpage

	{\small
		\bibliographystyle{ieee_fullname}
		\bibliography{literature}
	}
	
	% \clearpage
	% \newpage
	
	%% DEACTIVATE main page title if you want title for appendix
	%\title{Making Higher Order MOT Scalable: An Efficient Approximate Solver for Lifted Disjoint Paths \\ Supplemental Material}
	%\maketitle
	%\maketitle
%\begin{abstract}
%Fill in content of Abstract
%  \end{abstract}

\begin{figure*}[t]
    \centering
    \includegraphics{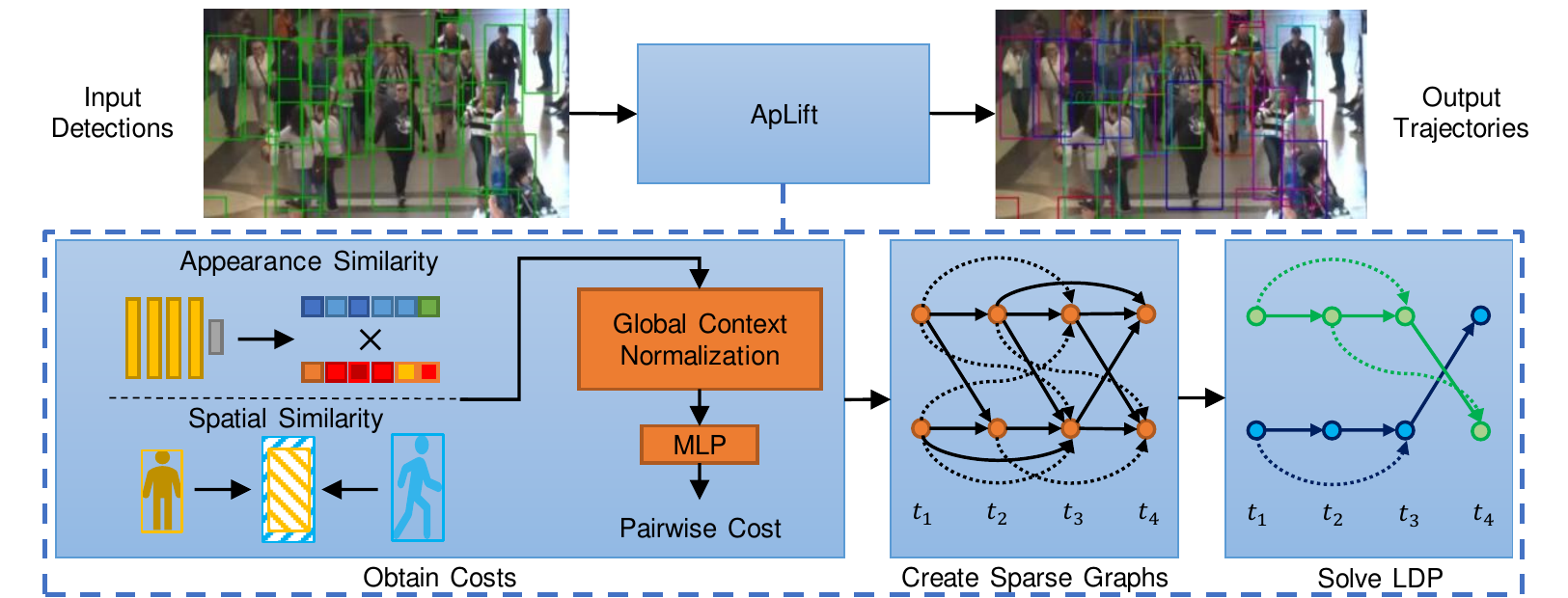}
    \caption{Overview of the ApLift framework. Input detections are used to obtain pairwise costs by an MLP with spatial and appearance features. Based on the costs, two sparse graphs are constructed and passed to our proposed approximate LDP solver. Dashed arrows represent lifted edges and solid arrows base edges. In figure \textit{Solve LDP} equally colored nodes and edges belong to the same trajectory.}
    \label{fig:overallframework}
\end{figure*}
\newpage

\section{Appendix}
This Appendix contains details about our approximate LDP solver and the whole MOT framework used in ApLift. We depicts this framework in Figure~\ref{fig:overallframework}.

\myparagraph{Appendix outline.}  We start with providing additional notation and abrreviations list in Section~\ref{ap:notation}. Sections~\ref{sec:outflow-subproblem-min-marginals}-\ref{sec:appendix:primal} present the message passing solver implementation and the algorithms used for it. Sections~\ref{appendix:feature_scaling}-\ref{appendix:post_processing} present details about processing of the tracking data. Finally, Section~\ref{appendix:runtime} discusses theoretical runtime of the solver and Section~\ref{sec:qualitative-results} presents examples of qualitative results. The Appendix is rather extensive, especially its algorithmic part. Therefore, we provide its section outline within the context of the whole method bellow.

\myparagraph{LDP solver outline. }Figure~\ref{fig:algorithms} contains a~scheme of all algorithms used  in our LDP solver. The algorithms are stated either in the main paper or in this Appendix. The solver performs an~explicitly given number of message passing iterations. Section~\ref{sec:message-passing} describes the full solver run and an overview of all methods used within one message passing iteration. Once in five iterations, new primal solution is computed (Sections~\ref{sec:primal-rounding} and~\ref{sec:appendix:primal}). Once in twenty iterations, new subproblems are separated and added to the problem. These are path and cut subproblems (see Sections~\ref{sec:path-subproblem} and~\ref{sec:cut-subproblem}). Methods for their separations are described in Sections~\ref{sec:cutting-planes-path-subproblems} and~\ref{sec:cutting-planes-cut-subproblems}.

\myparagraph{Message passing. } Messages are sent between the subproblems. Each subproblem creates messages to be sent by computing min-marginals of its variables. Section~\ref{sec:outflow-subproblem-min-marginals} presents algorithms used for obtaining min-marginals of inflow and outflow subproblems. The algorithms allow us to efficiently obtain min-marginals of all lifted or all base edges of a subproblem at once. 
Messages from cut and path subproblems are obtained by modifications of the respective algorithms for their optimization. See Section~\ref{sec:cut-subproblem} in the main text for the cut subproblem optimization and Section~\ref{sec:path-optimization} for path subproblem optimization. 

\myparagraph{Tightening by separation. }We create the new path and cut subproblems in order to tighten the LP relaxation of the problem \eqref{eq:lifted-disjoint-paths-problem}. Section~\ref{sec:separation-lower-bound} discusses the guaranteed lower bound improvement achieved by separating the new subproblems using algorithms in Sections~\ref{sec:cutting-planes-path-subproblems} and ~\ref{sec:cutting-planes-cut-subproblems}. 

\myparagraph{Tracking. }The proposed tracking framework contains additional processing steps, which are briefly mentioned in the main paper. A detailed description and additional evaluation data is provided in this appendix. To construct the graph we calculate costs based on two features as described in Section~\ref{sec:cost_classifier} and add multiple scalings which details can be found in Section~\ref{appendix:feature_scaling}. We also determine very confident edges and set their cost based on heuristics explained in Section~\ref{appendix:high_confident}. Furthermore, additional implementation and training details for the classifier are presented in Sections~\ref{appendix:classifier} and~\ref{appendix:sampling}. The efficient inference based on interval solutions is provided in Section~\ref{sec:intervalSolution}. Finally we show details for the post-processing based on heuristics in Section~\ref{appendix:post_processing}.  

% Once in five iterations, new primal solution is computed. Section~\ref{sec:primal-rounding} in the main text and Section~\ref{sec:appendix:primal} in the Appendix present methods for obtaining the primal solution. 

\begin{figure*}
    \centering
    \hspace{-1.2cm}\includegraphics[width=1.07\textwidth]{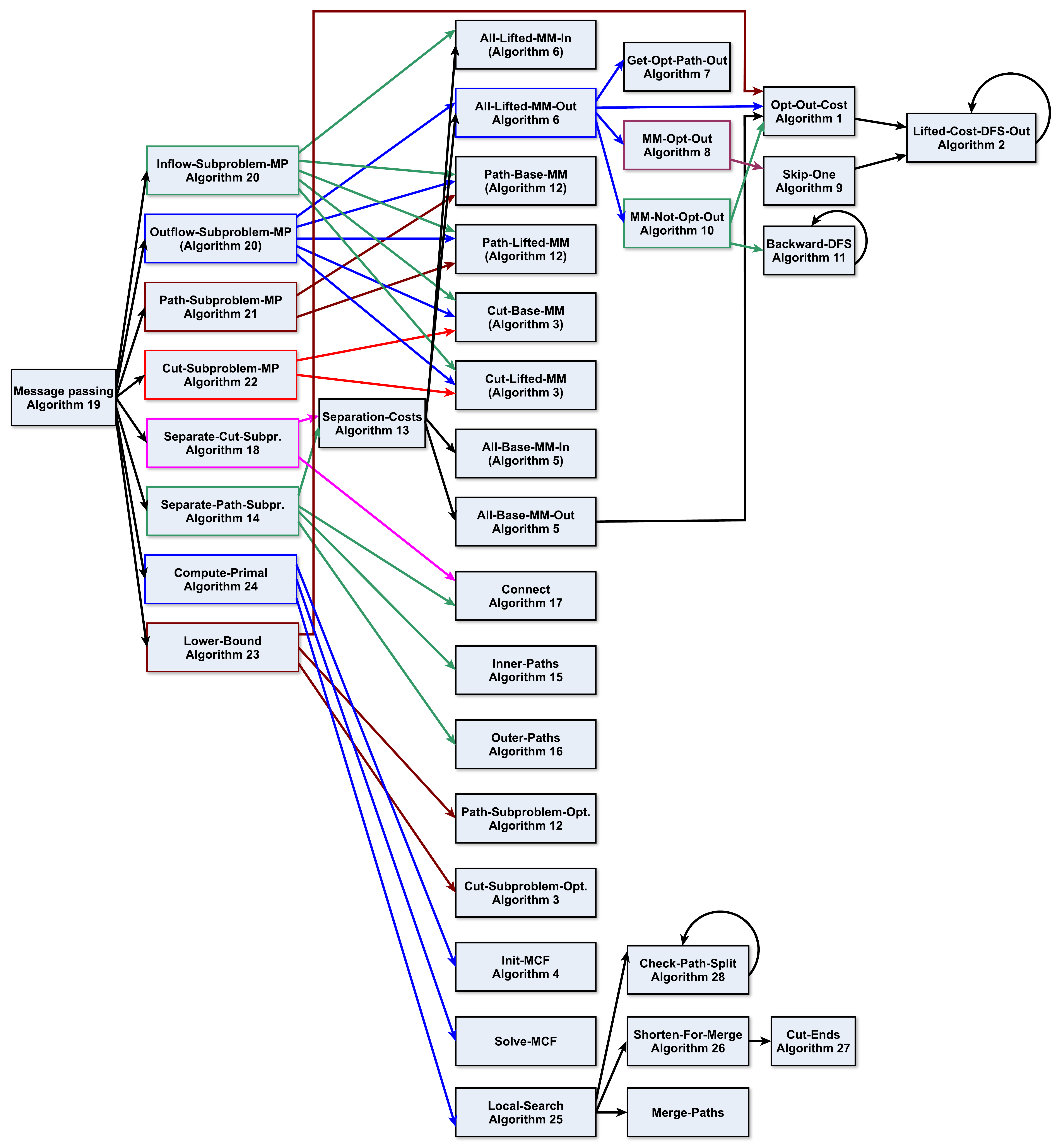}
    \caption{The scheme of our message passing algorithm and all its subroutines described in this work. An arrow from Algorithm~$X$ to Algorithm~$Y$ means that Algorithm~$X$ calls Algorithm~$Y$. Algorithms in brackets denote that their modifications are used as the respective procedures. Abbreviation MP means Messate-Passing. Some algorithms for inflow subproblems are omitted for clarity because they are analogical to outflow subproblem algorithms.}
    \label{fig:algorithms}
\end{figure*}

\subsection{Additional notation and abbreviations}\label{ap:notation}
\begin{itemize} 
    \item$x\in A^B$ denotes a~mapping $x:B\rightarrow A$.
    \item $[n]$ denotes the set of numbers $\{1,2,\dots,n\}$.
    \item LP: Linear programming.
    \item MP: Message passing.
    \item DP: disjoint paths problem.
    \item LDP: Lifted disjoint paths.
    \item DFS: Depth first search.
    \item MPLP: Max Product Linear Programming.
\end{itemize}

\subsection{Min-Marginals for Inflow and Outflow Subproblems }
\label{sec:outflow-subproblem-min-marginals}

We detail routines for computing min-marginals for all base edges at once (Algorithm~\ref{ap:alg:base-min-marginals}) and all lifted edges at once (Algorithm~\ref{ap:alg:all-lifted-mm}). All the stated algorithms assume outflow subproblems. Modification to inflow subproblems is done via proceeding in the oposite edge direction. 

Iteratively computing min-marginals and performing operation~\eqref{eq:min-marginal-update} would be inefficient, since it would involve calling Algorithm~\ref{alg:outflow-minimization} $\mathcal{O}(\abs{\delta^+_E(v)} + \abs{\delta^+_{E'}(v))}$ times.
To speed up iterative min-marginal updates, we can reuse computations as done in Algorithm~\ref{ap:alg:base-min-marginals} for base edges and in Algorithm~\ref{ap:alg:all-lifted-mm} for lifted edges.

Algorithm~\ref{ap:alg:base-min-marginals} for computing base edge min-marginals uses the fact that lifted edge costs do not change and therefore Algorithm~\ref{alg:outflow-minimization} needs to be called only once.
For lifted edges, Algorithm~\ref{ap:alg:all-lifted-mm} interleaves min-marginal computation and reparametrization updates~\eqref{eq:min-marginal-update} such that computations can be reused.
We introduce auxiliary variables $\gamma'_{vw}$ in line~\ref{alg:delta-init} that keep track of future reparametrization updates. 

For the min-marginals, we will need slight extensions of Algorithm~\ref{alg:outflow-minimization} and a method to additionally compute a labeling that attains the optimum. These methods are given in Algorithm~\ref{ap:alg:outflow-minimization} and~\ref{ap:alg:get-opt-path}.

In Algorithm~\ref{ap:alg:all-lifted-mm}, path $P^*$ representing the optimal solution of the outflow problem is found by calling Algorithm~\ref{alg:outflow-minimization} followed Algorithm~\ref{ap:alg:get-opt-path}. 
Then, Algorithm~\ref{ap:alg:mm-lifted-opt} computes min-marginals for the lifted edges that are active in the optimal solution.
In the end of Algorithm \ref{ap:alg:all-lifted-mm}, min-marginals are computed for those lifted edges that are not active in the optimal solution. 

% For the min-marginals, we will need slight extensions of Algorithm~\ref{alg:outflow-minimization} and a method to additionally compute a labeling that attains the optimum. These methods are given in Algorithm~\ref{ap:alg:outflow-minimization} and~\ref{ap:alg:get-opt-path}.

% Calling  on vertices of the optimal path $P^*$ in the order given by the path allows to reuse most of the values stored in lifted\_cost$[u]$. 
% In the end of Algorithm \ref{ap:alg:all-lifted-mm}, min marginals are computed for those lifted edges that are not active in the optimal solution. 

For computing min-marginals of edges that are active in the optimal solution, we need as a subroutine Algorithm~\ref{ap:alg:outflow-minimization}, an extended version of Algorithm~\ref{alg:outflow-minimization}.
Algorithm~\ref{ap:alg:outflow-minimization} restricts the vertices taken into consideration during the optimization.
In particular, a special vertex $r$ is given that is to be excluded from the optimization.
Values $\text{lifted\_cost}[u]$ are reused for those vertices $u$ where $ur\notin \mathcal{R}_G$ because these values are not affected by excluding vertex $r$. 

Min-marginals for vertices inactive in the optimal solution are computed by Algorithms \ref{ap:alg:mm-lifted-not-opt} and \ref{ap:alg:BU-main}. The algorithms rely on structure back\_cost which is an analogy of lifted\_cost. Structure back\_costs$[u]$ contains the minimum cost of all $vu\mhyphen$paths w.r.t. to the costs of all lifted edges connecting $v$ with the vertices of the path plus the cost of the first base edge of the path. Note that lifted\_costs$[u]$ is defined analogically but contains the minimum cost of all $ut\mhyphen$paths. 
Therefore, the minimal solution where a lifted edge $vu\in E'$ is active can be obtained as follows:
\begin{multline}
    \min\limits_{(z,y,y') \in \mathcal{X}^{out}_{v}: y'_{vu}=1} \langle 
    (z_v, y, {y'}), \theta^{out} \rangle=\\
    =\text{lifted\_cost}[u]+\text{back\_cost}[u]-\tilde{\theta}'_{vu}
\end{multline}
The cost of lifted edge $\tilde{\theta}'_{vu}$ must be subtracted because it is involved in both values $\text{lifted\_cost}[u]$ and $\text{back\_cost}[u]$.

Algorithm \ref{ap:alg:BU-main} performs two tasks simultaneously. First, it is a DFS procedure for computing back\_cost. Contrary to Algorithm \ref{alg:outflow-minimization-dfs} that performs DFS for obtaining lifted\_cost, Algorithm \ref{ap:alg:BU-main} proceeds in the opposite edge direction. It again uses the fact that a subpath of a minimum-cost path must be minimal. Second, it directly computes min marginal for already processed vertex $u$ on Line~\ref{alg:line:gamma} and involves this change in setting back\_cost$[u]$ on Line~\ref{alg:line:back-cost}.  

\myparagraph{Speeding up DFS: }All the algorithms for obtaining optimal solution or min-marginals of inflow and outflow subproblems call DFS procedures. It can be considered that the order of processing the relevant nodes reachable form the central node is always the same. Therefore, we call DFS for each inflow and outflow subproblem only once during their initialization and store the obtained list of processed nodes. The full DFS in Algorithm~\ref{alg:outflow-minimization-dfs} is replaced by traversing the precomputed node list in the forward direction. Algorithm~\ref{ap:alg:BU-main} is replaced by traversing this node list in the backward direction. 

\begin{algorithm}[h]
\caption{All-Base-MM-Out$(v,\tilde{\theta})$}
\label{ap:alg:base-min-marginals}
\textbf{Input} start vertex $v$, costs $\tilde{\theta}$ \\
\textbf{Output} base edge min-marginals $\gamma_{vu}\ \forall vu\in \delta^+_{E}(v)$
\begin{algorithmic}[1]
\State $(\mathrm{opt},\text{lifted\_cost},\alpha)=$Opt-Out-cost$(v,\theta^{out})$ 
%\State $\forall vw\in\delta_E^+(v): m_{vw}=\tilde{\theta}_v+  \tilde{\theta}_{vw} + \text{lifted\_cost}[w]$
\State $e^* = \argmin\limits_{vw \in \gamma^+_E}\{\alpha_{vw}\}$, $e^{**} = \argmin\limits_{vw \in \gamma^+_E \backslash \{e^*\}}\{\alpha_{vw}\}$
\State $\forall  vu\in\delta^+(v):  \gamma_{vu}=\alpha_{vu}-\min(\alpha_{e^{**}},0)$
\end{algorithmic}
\end{algorithm}

\begin{algorithm}[h]
    \caption{All-Lifted-MM-Out$(v,\tilde{\theta})$}
    \label{ap:alg:all-lifted-mm}
    \textbf{Input} starting vertex $v$, $\tilde{\theta}$\\
    \textbf{Output} lifted edge min-marginals $\gamma'_{vu}\ \forall vu\in \delta^+_{E'}(v)$
    \begin{algorithmic}[1]
    \State $(\mathrm{opt},\text{lifted\_cost},\alpha,\text{next})=$Opt-Out-cost$(v,\theta^{out})$ 
    \State $P^*_V$=Get-Opt-Path-Out$({\theta}^{out},\alpha,\text{next})$
   % \State $P^*=(v_1...v_n)$ where $v_{i+1}=$best\_neighbor$[v_i]$
%    \State $P^*=$ $vt\mhyphen$ path $P$ representing $opt$
%    \State Compute $P^* \in vt\mhyphen \text{path}(G)$, $P_E =  (v_1,\ldots,v_k)$ such that $y_{v v_2} = 1$, $y'_{vu} = 1 \Leftrightarrow u \in P_V$.
    %\State $\gamma_{vw}=0\ \forall vw\in \delta^+_{E}(v)$,
    \State $ \forall vw\in \delta^+_{E'}(v): \gamma'_{vw}=0$
    \label{alg:delta-init}
    \State $(\mathrm{opt},\gamma')=$MM-Opt-Out$(v,P^*_V,\mathrm{opt},\gamma',\tilde{\theta})$
    \State $\gamma'=$MM-Not-Opt-Out$(v,\mathrm{opt},\gamma',\tilde{\theta})$  
    \end{algorithmic}
\end{algorithm}

\begin{algorithm}[h]
 \caption{Get-Opt-Path-Out}
 \label{ap:alg:get-opt-path}
  \textbf{Input} costs $\tilde{\theta}$, vector $\alpha$ such that $\forall vw\in \delta^+_E(v):\alpha_{vw}$ is the optimal value if $vw$ is active, next\\
 \textbf{Output} min cost path $P^*_V$ 
    \begin{algorithmic}[1]
  %  \State $\alpha_{vw} = \min_{w:vw \in \delta^+_E(v)} \tilde{\theta}_v + \tilde{\theta}_{vw} + \text{lifted\_cost}[w] $
    \State $w^*=  \argmin_{w:vw \in \delta^+_{E}(v)} \alpha_{vw}$
   \If{$\alpha_{vw^*}<0$}
   \While{$w^*\neq t$}
   \State $P^*_V\leftarrow w^*$
   \State $w^*=\text{next}[w^*]$
   \EndWhile
  % \State $P^*_V=(v_1,\dots,v_n)$: $v_{i+1}=\text{next}[v_i],v_n=t$
   \Else
   \State  $P^*_V=\emptyset$
   \EndIf
    \end{algorithmic}
\end{algorithm}

\begin{algorithm}[h]
    \caption{MM-Opt-Out}
    \textbf{Input} starting vertex $v$, optimal path $P^*_V=(v_1,\ldots,v_k)$, value of optimal path $\mathrm{opt}$, $\gamma'$, costs $\tilde{\theta}$\\
    \textbf{Output} updated cost of optimal path $\mathrm{opt}$, new reparametrization updates $\gamma'$
    \label{ap:alg:mm-lifted-opt}
    \begin{algorithmic}[1]
     \ForAll {$v_i=  v_1,\ldots,v_k : vv_i\in \delta^+_{E'}(v)$}
    \State $\alpha=\text{Skip-One}(v,v_i,\tilde{\theta}-(\mathbb{0},\gamma'),\text{lifted\_cost},\text{next})$
    \State $\gamma'_{vv_i}=\mathrm{opt}-\alpha$
    \State $\mathrm{opt}=\alpha$
    \EndFor
    \end{algorithmic}
\end{algorithm}

\begin{algorithm}[h]
 \caption{Skip-One}  
 \label{ap:alg:outflow-minimization}
 \textbf{Input}  $v$, ignored vertex $r$, $\tilde{\theta}$, lifted\_cost, next\\
 \textbf{Output} optimal value $\mathrm{opt}$
    \begin{algorithmic}[1]
    \For{$u\in V:vu\in\mathcal{R}_G\wedge ur\in \mathcal{R}_G$}
    %\State $\text{lifted\_cost}[u] = \infty$, $\text{visited}[u] = false$, next$[u]=t$
    \State $\text{lifted\_cost}[u] = \infty$, next$[u]=\emptyset$
    \EndFor
    %\State$\text{lifted\_cost}[r] =0 $, $\text{visited}[r] = true$, next$[r]=t$
    \State$\text{lifted\_cost}[r] =0 $,  next$[r]=t$
    \State Lifted-Cost-DFS-Out($v,v,\tilde{\theta},\text{lifted\_cost},\text{next}$)
    \State $\forall w:vw \in \delta^+_E(v): \alpha_{vw} =  \tilde{\theta}_v + \tilde{\theta}_{vw} + \text{lifted\_cost}[w] $
    \State $\mathrm{opt} = \min(\min_{w:vw \in \delta^+_E(v)} \alpha_{vw}, 0)$
    \end{algorithmic}
\end{algorithm}

\begin{algorithm}[h]
%\SetAlgorithmName{Opt-Outflow}{Opt-Outflow}{List of algorithms}
%\mylabel{alg:opt-outflow}{\textbf{opt-outflow}}  
    \caption{MM-Not-Opt-Out}
    \label{ap:alg:mm-lifted-not-opt}
        \textbf{Input} $v$, current optimum $\mathrm{opt}$, reparametrization update $\gamma'$, $\tilde{\theta}$\\
    \textbf{Output} changed reparametrization update $\gamma'$
    \begin{algorithmic}[1]
    \State $(\mathrm{opt},\text{lifted\_cost})=$Opt-Out-cost$(v,\tilde{\theta}-(\mathbb{0},\gamma'))$ 
    \ForAll{$u: vu\in \mathcal{R}_G$}
    \If{$u\in P^*_V$}
    \State $\text{visited}[u] = true$
    \State $\text{back\_cost}[u]=\mathrm{opt}-\text{lifted\_cost}[u]$
    \IIf{$vu\in E'$} back\_cost$[u]\pluseq\tilde{\theta}'_{vu}-\gamma'_{vw}$ 
    \Else 
    \State $\text{visited}[u] = false$
    \If{$vu\in \delta^+_{E}(v)$}
    \State $\text{back\_cost}[u]=\tilde{\theta}_{vu}$
    \Else
    \State $\text{back\_cost}[u]=\infty$
    \EndIf
    \EndIf
    \EndFor
    \ForAll{$vu \in \delta_{E'}^+(v)$}
    \If{$\text{visited}[u]=false$}
    \State Backward-DFS($v,u,\tilde{\theta},\gamma',\mathrm{opt},\text{back\_cost}$)
    \EndIf
    \EndFor
    \end{algorithmic}
\end{algorithm}

\begin{algorithm}[h]
    \caption{Backward-DFS }
    \label{ap:alg:BU-main}
    \textbf{Input} $v,u,\tilde{\theta},\gamma',\mathrm{opt},\text{back\_cost}$\\
    \textbf{Output}  $\gamma'$, back\_cost
    \begin{algorithmic}[1]
   % \State $\tilde{\theta}=\theta^{out}-\gamma$
        \State $\alpha=\text{back\_cost}[u]$
        \For{$wu \in \delta^-_{E}(u): vw\in \mathcal{R}_G$}
        \If{$\text{visited}[w] = false$}
        \State Backward-DFS($v,w,\tilde{\theta},\gamma',\text{back\_cost}$)
        \EndIf
        \State$ \alpha=\min \{\text{back\_cost}[w],\alpha\}$
        \EndFor
        \If{$vu\in E'$}
        %\State $bc=m+\theta'^{out}_{vu}$
        \State $\mathrm{opt}_{u}=\alpha+\text{lifted\_cost}[u]$
        \State $\gamma'_{vu}=\mathrm{opt}_u-\mathrm{opt}$\label{alg:line:gamma}
        \State $\text{back\_cost}[u]=\alpha+\tilde{\theta}_{vu}-\gamma'_{vu}$\label{alg:line:back-cost}
        %\State $\theta'^{out}_{vu}-=\delta'_{vu}$
        \Else
        \State $\text{back\_cost}[u]=\alpha$
        \EndIf
        \State $\text{visited}[u] = true$
    \end{algorithmic}
\end{algorithm}

%\subsection{Minimum Cost Flow Subproblem}

 \vfill\null
 \newpage
% \vfill\null
% \newpage
% \vfill\null
% \newpage
\subsection{Optimization of path subproblems. }\label{sec:path-optimization}
We denote by $\theta^P$ the edge costs in subproblem of $vw\mhyphen$path $P$. The optimization over the feasible set $\mathcal{X}^P$ w.r.t.\ costs $\theta^P$ is detailed in Algorithm~\ref{alg:path-subproblem-optimization}.
It checks whether 
there exists exactly one positive edge and whether it is either a lifted or a strong base edge (Line~\ref{algl:path-subproblem-opt-E+=1}).
If so, the optimal solution is either 
(i)~all edges except the two largest ones (Line~\ref{algl:path-subproblem-opt-except-two-largest}) or
(ii)~all edges (Line~\ref{algl:path-subproblem-opt-all-edges}), whichever gives smaller objective value.
If the above condition does not hold, the optimal solution can be chosen to contain all negative edges (Line~\ref{algl:path-subproblem-opt-all-negative-edges}).
%The optimization is usually simple.

\begin{algorithm}[h]
\caption{Path-Subproblem-Optimization}
\label{alg:path-subproblem-optimization}
 \textbf{Input} Edge costs $\theta^P$\\
 \textbf{Output} optimal value $\mathrm{opt}$ of subproblem.
    \begin{algorithmic}[1]
\State $E^+=\{ kl \in P_{E'} \cup vw|\theta'^P_{kl}>0\} \cup \{kl\in P_{E}| \theta^P_{kl} > 0\}$ 
%   \If{$E^+=\emptyset$}
%   \State return $\sum\limits_{ij \in P_E} \theta^P_{ij} + \sum\limits_{ij \in P_{E'} \cup \{vw\}} \theta'^P_{ij}$
   \If{$E^+=\{kl\}\wedge kl\in P_{E'}\cup vw\cup E_0$}\label{algl:path-subproblem-opt-E+=1} 
    \State $\alpha = \min \{\min\limits_{ij \in P_E\setminus E^+} \abs{\theta^P_{ij}},\min\limits_{ij \in P_{E'}\cup vw\setminus E^+} \abs{\theta'^P_{ij}}\}$ 
    \State $\beta = \begin{cases} \theta'^P_{kl},& kl \in P_{E'}\cup vw \\ \theta^P_{kl}, & kl \in P_E \end{cases}$
%      \If{$kl\in P_{E'}$}
%   \State $\beta=\theta'^P_{kl}$
%   \Else
%   \State  $\beta=\theta^P_{kl}$
%   \EndIf
   \If{$\alpha<\beta$} 
   \State  $\mathrm{opt} = \sum\limits_{ij \in P_E\setminus E^+} \theta^P_{ij} + \sum\limits_{ij \in P_{E'} \cup vw\setminus E^+} \theta'^P_{ij} +\alpha$\label{algl:path-subproblem-opt-except-two-largest} 
   \Else
   \State $\mathrm{opt} = \sum\limits_{ij \in P_E} \theta^P_{ij} + \sum\limits_{ij \in P_{E'} \cup vw} \theta'^P_{ij}$\label{algl:path-subproblem-opt-all-edges}
   \EndIf
   \Else 
     \State $\mathrm{opt} = \sum\limits_{ij \in P_E\setminus E^+} \theta^P_{ij} + \sum\limits_{ij \in P_{E'} \cup vw\setminus E^+} \theta'^P_{ij}$ \label{algl:path-subproblem-opt-all-negative-edges}
       \EndIf
       \State return $\mathrm{opt}$
    \end{algorithmic}
\end{algorithm}
\todo[inline]{ Set-Notation in Alg 5 not consistent. Use $:$ ? }

A variation of Algorithm~\ref{alg:path-subproblem-optimization} with a specified edge fixed to either $0$ or $1$ is used for computing min-marginals

\subsection{Separation for Path Subproblems}
\label{sec:cutting-planes-path-subproblems}

The path subproblem separation procedure is described in Algorithm~\ref{alg:path-sep}.
The algorithm finds paths together with a lifted edge connecting the start and the end point of the path such that exactly one lifted edge has positive cost, while all the remaining base and lifted edges have negative cost.

First, lifted and base edge costs are obtained in Algorithm~\ref{alg:extract-MM} by computing min-marginals of inflow and outflow factors.
Second, a graph with an empty edge set $E^1$ is created.
Then, edges with negative costs are added to $E^1$ in ascending order.
After adding an edge, we check whether separating path subproblems with edge costs leading to lower bound improvement is possible.
Such a factor must contain the newly added edge, one positive lifted edge and edges that already are in the edge set $E^1$.

Algorithm~\ref{alg:inner-path-sep} separates those paths subproblems where the only positive edge is the one connecting the path's endpoints. Algorithm~\ref{alg:outer-path-sep} separates those path subproblems where the only positive edge is one of the edges within the path. Algorithm~\ref{alg:connect} updates connectivity structures by adding edge $ij$ to the edge set $E^1$. 

Each path subproblems has a guaranteed lower bound improvement, see Proposition~\ref{prop:path-subproblem-guaranteed-lower-bound-improvement}.
We add each found path subproblem to priority queue $Q$, where we sort w.r.t.\ the guaranteed lower bound improvement.
After searching for path subproblems, we add the $k$ best path subproblems from queue $Q$ to the optimization problem.

%In order to obtain $\theta'^P_{uv}<0$ (resp. $\theta^P_{uv}<0$)., the lower bound of either $\theta^{out}_v+\theta^{in}_{u}$   

%\textcolor{blue}{Explain that this way is the weakest negative edge maximal in absolute value from all alternative paths.}  
\begin{algorithm}[h]
    \caption{Separation-Costs}
    \label{alg:extract-MM}
    \textbf{Input} Current cost in inflow and outflow factors $\theta^{in},\theta^{out}$ \\
    \textbf{Output} Cost reparametrization $\forall uv\in E: \tilde{\theta}_{uv}$, $\forall uv\in E':\tilde{\theta}'_{uv}$
    \begin{algorithmic}[1]
    \State $\forall uv\in E: \tilde{\theta}_{uv}=0$, $\forall uv\in E': \tilde{\theta}'_{uv}=0$
    \ForAll{$u\in V\setminus \{s,t\}$}
    \State $\forall uv\in \delta^+_{E}(u): \gamma^{out}_{uv}=0$
    \State $\forall uv\in \delta^-_{E}(u): \gamma^{in}_{vu}=0$
    \State $\gamma'^{out}_u=0.5\cdot$All-Lifted-MM-Out$(u,\theta^{out}_u)$
    \State $\gamma'^{in}_u=0.5\cdot$All-Lifted-MM-In$(u,\theta^{in}_u)$
    \State $\gamma^{out}_u=$All-Base-MM-Out$(u,\theta^{out}_u-(\gamma^{out}_u,\gamma'^{out}_u))$
    \State $\gamma^{in}_u=$All-Base-MM-In$(u,\theta^{in}_u-(\gamma^{in}_u,\gamma'^{in}_u))$
    \State $\forall uv\in \delta^+_{E}(u): \tilde{\theta}_{uv}\pluseq\gamma^{out}_{uv}$
    \State $\forall uv\in \delta^-_{E}(u): \tilde{\theta}_{vu}\pluseq\gamma^{in}_{vu}$
    \State $\forall uv\in \delta^+_{E'}(u): \tilde{\theta}'_{uv}\pluseq\gamma'^{out}_{uv}$
    \State $\forall uv\in \delta^-_{E'}(u): \tilde{\theta}'_{vu}\pluseq\gamma'^{in}_{vu}$
    \EndFor
    \end{algorithmic}
\end{algorithm}

\begin{algorithm}[h]
    \caption{Separate-Path-Subproblem}
    \label{alg:path-sep}
    \textbf{Input }Cost threshold $\varepsilon$
    \begin{algorithmic}[1]
    % \ForAll{$vw\in E'$}
    % \State $\theta'_{vw}=m'^{out}_{vw}+m'^{in}_{vw}$
    % \EndFor
    % \ForAll{$vw\in E$}
    % \State $\theta_{vw}=m^{out}_{vw}+=m^{in}_{vw}$
    % \EndFor
    \State $\tilde{\theta}=$Separation-Costs$(\theta^{in},\theta^{out})$
    \State $G^1=(V,E^1=\emptyset)$
    \State $E^-=\{vw\in E|\tilde{\theta}_{vw}<-\varepsilon\}\cup \{vw\in E'|\tilde{\theta}'_{vw}<-\varepsilon\} $
    \State $E'^+=\{vw\in E'|\tilde{\theta}'_{vw}>\varepsilon\} $
    \State $\forall v\in V: \text{desc}[v]=\{v\},  \text{pred}[v]=\{v\}$
    \State Priority-Queue $Q=\emptyset$
    \ForAll{$ij\in E^-$ ascending in $\tilde{\theta}$}
    \IIf {$ij\in E$} $c_{ij}=\tilde{\theta}_{ij}$
    \EElse $c_{ij}=\theta'_{ij}$
   \State Inner-Paths$(i,j,c_{ij},\text{pred},\text{desc},E^+,E^1,Q)$
       \State Outer-Paths$(i,j,c_{ij},\text{pred},\text{desc},E^+,E^1,Q)$
         \State Connect$(i,j,\text{pred},\text{desc},E^1)$
     \EndFor
    \end{algorithmic}
\end{algorithm}

\begin{algorithm}[h]
    \caption{Inner-Paths}
    \label{alg:inner-path-sep}
    \textbf{Input} $i,j,c_{ij},\text{pred},\text{desc},E^+,E^1,Q$
    \begin{algorithmic}[1]
     \ForAll{$p\in \text{pred}[i]$}
    \ForAll{$d\in \text{desc}[j]$}
    \If{$pd\in E^+$}
    \State $P_1=$Find-Path$(p,i,E^1)$
    \State $P_2=$Find-Path$(j,d,E^1)$
    \State $P=(P_1,ij,P_2)$
    \State $priority=\min\{|c_{ij}|,\tilde{\theta}'_{pd}\}$
    \State $Q\leftarrow($Path-Problem$(P),priority)$\EndIf
    \EndFor
    \EndFor
    \end{algorithmic}
\end{algorithm}

\begin{algorithm}[h]
    \caption{Outer-Paths}
    \label{alg:outer-path-sep}
    \textbf{Input }$i,j,c_{ij},\text{pred},\text{desc},E^+,E^1,Q$
    \begin{algorithmic}[1]
      \ForAll{$p\in \text{pred}[j]$}
    \ForAll{$d\in \text{desc}[i]$}
    \If{$dp\in E^+$}
  \State $P_1=$Find-Path$(i,d,E^1)$
    \State $P_2=$Find-Path$(p,j,E^1)$
    \State $P=(P_1,ij,P_2)$
    \State $priority=\min\{|c_{ij}|,\tilde{\theta}'_{dp}\}$
    \State $Q\leftarrow($Path-Problem$(P),priority)$
    \EndIf
    \EndFor
    \EndFor
    \end{algorithmic}
\end{algorithm}

\begin{algorithm}[h]
    \caption{Connect}
    \label{alg:connect}
    \textbf{Input }$i,j,\text{pred},\text{desc},E^1$
    \begin{algorithmic}[1]
     \ForAll{$p\in \text{pred}[i]$}
    \ForAll{$d\in \text{desc}[j]$}
    \State $\text{desc}[p]\leftarrow d$
    \State $\text{pred}[d]\leftarrow p$
    \State $E^1\leftarrow ij$
    \EndFor
    \EndFor
    \end{algorithmic}
\end{algorithm}

\subsection{Separation for Cut Subproblems}
\label{sec:cutting-planes-cut-subproblems}
Algorithm~\ref{alg:cut-sep} separates cut subproblems.
The algorithm finds cuts consisting of base edges with positive costs and a lifted edge having endpoints on both sides of the cut and negative cost.
Similarly as for the path subproblem separation, lifted and base edge costs are obtained by computing min-marginals of inflow and outflow factors in Algorithm~\ref{alg:extract-MM}.
Each edge $uv\in E'^-$  is a  candidate lifted edge for a~$uv\mhyphen$cut factor.

The edge set $E^1$ initially contains all base edges with cost lower than $\varepsilon$.
The remaining base edges are added to $E^1$ in ascending order.
Whenever a newly added edge $ij$ causes a connection between $u$ and $v$  where $uv\in E'^-$, a $uv\mhyphen$cut $C$ is separated.
We select the cut $C$ to contain only those edges that do not belong to $E^1$.
This ensures that $ij$ is the weakest cut edge.
In the same time, $C$ is the best $uv\mhyphen\text{cut}$ with respect to the cost of the weakest cut edge. 

The found cut factors are added to a priority queue where the priority represents guaranteed lower bound improvement (see Proposition~\ref{prop:cut-subproblem-lower-bound-improvement}) after adding the factor to our problem.

\begin{algorithm}[h]
    \caption{Separate-Cut-Subproblem}
    \label{alg:cut-sep}
     \textbf{Input }Cost threshold $\varepsilon$
    \begin{algorithmic}[1]
    % \ForAll{$vw\in E'$}
    % \State $\theta'_{vw}=m'^{out}_{vw}+m'^{in}_{vw}$
    % \EndFor
    % \ForAll{$vw\in E$}
    % \State $\theta_{vw}=m^{out}_{vw}+m^{in}_{vw}$
    % \EndFor
    \State $\tilde{\theta}=$Separation-Costs$(\theta^{in},\theta^{out})$
    \State $E'^-= \{vw\in E'|\tilde{\theta}'_{vw}<-\varepsilon\} $
    \State $E^-= \{vw\in E|\tilde{\theta}'_{vw}<\varepsilon\} $, $E^+=E\setminus E^-$
    \State $ E^1= E^-$, $G^1=(V,E^1)$
        \State Priority-Queue $Q=\emptyset$
    \ForAll{$ij\in E^+$ ascending in $\tilde{\theta}$}
    \ForAll{$u\in \text{pred}[i]$}
    \ForAll{$v\in \text{desc}[j]$}
    \If{$uv\in E'^-$}
    \State $C$= cut between $u,v$ using edges $E\setminus E^1$
    %\State $C$=Find-Cut-Edges$(u,v,E\setminus E^1)$
    \State $priority=\min\{\tilde{\theta}_{ij},|\tilde{\theta}'_{uv}|\}$
    \State $Q\leftarrow($Cut-Problem$(C,u,v),priority)$
    \EndIf
    \EndFor
    \EndFor
    \State Connect$(i,j,\text{pred},\text{desc},E^1)$
    \EndFor
    \end{algorithmic}
\end{algorithm}

\subsection{Tightening Lower Bound Improvement}\label{sec:separation-lower-bound}
In order to show that the separation procedures in Algorithms~\ref{alg:path-sep} and~\ref{alg:cut-sep} lead to relaxations that improve the lower bound we show the following:
(i)~Certain reparametrization used in the above algorithms are non-decreasing in the lower bound.
(ii)~Separation procedures find new subproblems such that w.r.t.\ the above reparametrization, a guaranteed lower bound improvement can be achieved.

Points (i) and (ii) guarantee that the same lower bound achievement w.r.t.\ the original reparametrization can be found.
The special reparametrization chosen helps empirically to find good subproblems.

\begin{lemma}
\label{lemma:generalLB}
Let $\mathsf{s}\in \mathcal{S}$ be a subproblem, $\theta$ its cost and $L_{\mathsf{s}}(\theta)$ its lower bound for cost $\theta$. 
Given a cost reparametrization $\gamma$ such that 
\begin{enumerate}
    \item  $\forall i\in [d_{\mathsf{s}}]$
    \begin{equation}\gamma_{i}=
\begin{cases}
\leq 0 &\text{if } \exists x^*\in \argmin_{x\in \mathcal{X}^{\mathsf{s}}} \langle\theta,x\rangle: x^*_i=1\\
\geq 0 &\text{if } \exists x^*\in \argmin_{x\in \mathcal{X}^{\mathsf{s}}} \langle\theta,x\rangle: x^*_i=0
\end{cases}
\end{equation}
\item $\argmin\limits_{x\in \mathcal{X}^{\mathsf{s}}} \langle\theta,x\rangle \subseteq \argmin\limits_{x\in \mathcal{X}^{\mathsf{s}}}\langle\theta-\gamma,x\rangle$
\end{enumerate}
and a~coordinate-wise scaled reparametrization $\gamma(\omega)$ defined by coefficients $\omega\in [0,1]^{\mathsf{s}}$ where $\forall i\in [d_{\mathsf{s}}]: \gamma(\omega)_i=\omega_i\gamma_i$, it holds:
\begin{enumerate}
    \item The lower bound of $\mathsf{s}$ after reparametrization $ \gamma(\omega)$ is $L_{\mathsf{s}}(\theta-\gamma(\omega))=L_{\mathsf{s}}(\theta)-\sum_{i\in [d_{\mathsf{s}}]:\gamma_i<0}\omega_i\gamma_i$.
    \item $\argmin\limits_{x\in \mathcal{X}^{\mathsf{s}}} \langle\theta,x\rangle\subseteq \argmin\limits_{x\in \mathcal{X}^{\mathsf{s}}} \langle\theta-\gamma(\omega),x\rangle$
\end{enumerate}
\end{lemma}
\begin{proof}
We start with evaluating $\langle \theta-\gamma(\omega),x^*\rangle$ where $x^*\in \argmin_{x\in\mathcal{X}^{\mathsf{s}}} \langle\theta,x\rangle$.
% \begin{align}
% \nonumber\langle \theta-\gamma(\omega),x^*\rangle&=\sum_{i=1}^{ d_\mathsf{s}}(\theta_i-\omega_i\gamma_i)x^*_i=\\
% \nonumber&=\sum_{i=1}^{ d_\mathsf{s}}\theta_ix^*_i-\sum_{i\in [d_{\mathsf{s}}]:\gamma_i<0}\omega_i\gamma_i x^*_i=\\
% &=L_{\mathsf{s}}(\theta)-\sum_{i\in \label{eq:proof1:lb1}\mathsf{s}:\gamma_i<0}\omega_i\gamma_i
% \end{align}
\begin{align}
\nonumber\langle \theta-\gamma(\omega),x^*\rangle&=\sum_{i\in [d_\mathsf{s}]}(\theta_i-\omega_i\gamma_i)x^*_i=\\
\nonumber&=\sum_{i\in [d_\mathsf{s}]}\theta_ix^*_i-\sum_{i\in [d_\mathsf{s}]:\gamma_i<0}\omega_i\gamma_i x^*_i=\\
\label{eq:proof1:lb1}&=L_{\mathsf{s}}(\theta)-\sum_{i\in [d_\mathsf{s}]:\gamma_i<0}\omega_i\gamma_i
\end{align}
$\forall x\in \mathcal{X}, \forall x^*\in \argmin_{x\in X^{\mathsf{s}}}\la \theta, x \ra:$
\begin{align}
\nonumber\langle \theta-\gamma(\omega),x\rangle&=\sum_{i\in [d_{\mathsf{s}}]}(\theta_i-\omega_i\gamma_i)x_i=\\
\nonumber&=\sum_{i\in [d_{\mathsf{s}}]}(\theta_i-\gamma_i)x_i+
\sum_{i\in [d_{\mathsf{s}}]}(1-\omega_i)\gamma_i x_i\geq \\
\nonumber&\geq L_{\mathsf{s}}(\theta-\gamma)+\sum_{i\in [d_{\mathsf{s}}]:\gamma_i<0}(1-\omega_i)\gamma_i x_i=\\
\nonumber&=\sum_{i\in [d_{\mathsf{s}}]}(\theta_i-\gamma_i)x^*_i+
\sum_{i\in [d_{\mathsf{s}}]}(1-\omega_i)\gamma_i x^*_i=\\
\label{eq:proof1:lb2}&=\langle \theta-\gamma(\omega),x^*\rangle
\end{align}
Formula \ref{eq:proof1:lb2} proves Point 2 of Lemma~\ref{lemma:generalLB}. Formulas~\ref{eq:proof1:lb1} and \ref{eq:proof1:lb2} together prove Point 1.
\end{proof}

\begin{lemma}
Variables $(\gamma^{out}_u,\gamma'^{out}_u)$, resp. $(\gamma^{in}_u,\gamma'^{in}_u)$ in Algorithm \ref{alg:extract-MM} satisfy the requirements of Lemma~\ref{lemma:generalLB} for each outflow resp. inflow subproblem of vertex $u$. 
\end{lemma}
\begin{proof}
Both Algorithms~\ref{ap:alg:base-min-marginals} and \ref{ap:alg:all-lifted-mm} output reparametrization variables that satisfy the requirements of Lemma~\ref{lemma:generalLB}. We have, for an outflow subproblem of node $u$:
\begin{multline}
\argmin\limits_{(y,y')\in \mathcal{X}^{out}_{u}}\langle \theta,(y,y')\rangle\subseteq \argmin\limits_{(y,y')\in \mathcal{X}^{out}_{u}}\langle \theta-(0,\gamma'^{out}_u),(y,y')\rangle\\
\subseteq    \argmin\limits_{(y,y')\in \mathcal{X}^{out}_{u}}\langle \theta-(\gamma^{out}_u,\gamma'^{out}_u),(y,y')\rangle
\end{multline}
Therefore, also $(\gamma^{out}_u,\gamma'^{out}_u)$ together satisfy the requirements of Lemma~\ref{lemma:generalLB}. Analogically, for the inflow subproblems.  
\end{proof}

\myparagraph{Costs in the new path and cut subproblems.} One edge is typically shared among multiple newly added path and cut subproblems. Therefore, the available cost reparametrizations $\tilde{\theta}$ and $\tilde{\theta}'$ from Algorithm~\ref{alg:extract-MM} must be redistributed to the newly added subproblems. We denote the set of all newly added path subproblems resp. cut subproblems in tightening iteration $i$ by $\mathcal{P}^i$ resp. $\mathcal{C}^i$.
% \begin{equation}
% \omega_{uv} = \frac{1}{\abs{\{\mathsf{s}\in \mathcal{S}^i : uv\in \mathsf{s}\}}},
% \omega'_{uv} = \frac{1}{\abs{\{\mathsf{s}\in \mathcal{S}^i : uv\in \mathsf{s}\}}}        
% \end{equation}
For each base resp. lifted edge $uv$, we sum up the number of newly added path and cut subproblems that contain $uv$.
\begin{align}
\nonumber N_{uv} = &{\abs{\{P\in \mathcal{P}^i : uv\in P_E\}}+\abs{\{C\in \mathcal{C}^i : uv\in C_E\}}},\\
\nonumber N'_{uv} = &{\abs{\{P\in \mathcal{P}^i : uv\in P_{E'}\}}
+\abs{\{P\in \mathcal{P}^i: P\text{ is a }uv\mhyphen\text{path}\}}}\\
&+\abs{\{C\in \mathcal{C}^i: C\text{ is a }uv\mhyphen\text{cut}\}}\,.       
\end{align}
 Then, we define coefficient $\omega_{uv}$ resp. $\omega_{uv}'$ for each base edge $uv\in E$ resp. lifted edge $uv\in E'$ that belongs to a newly added subproblem  as 
\begin{equation}
\omega_{uv} = \frac{1}{N_{uv}}\,,\quad
\omega'_{uv} = \frac{1}{N'_{uv}}\,.       
\end{equation}
Finally, for each newly added path subproblem $P$ resp. cut subproblem $C$, we set the cost of base edge $uv\in E$ to $\theta^P_{uv}=\omega_{uv}\tilde{\theta}_{uv}$, resp. $\theta^C_{uv}=\omega_{uv}\tilde{\theta}_{uv}$. Analogically, for the lifted edges. 

\myparagraph{Cost update in in/outflow subproblems. }If we use an edge $uv$ for creating one or more path and cut suproblems, it is necessary to update its cost in the inflow subproblem of vertex $v$ and the outflow subproblem of vertex $u$ accordingly. For instance, we update the cost of base edge $uv$ in the outflow subproblem of $u$ as follows $\theta^{out}_{uv}\minuseq \gamma_{uv}^{out}$. Where we adopt the notation from Algorithm~\ref{alg:extract-MM}. Note that $\tilde{\theta}_{uv}=\gamma^{in}_{uv}+\gamma_{uv}^{out}$. Therefore,  the total cost of edge variable $uv$ is preserved.

\begin{proposition}[Guaranteed lower bound improvement of path subproblem]
\label{prop:path-subproblem-guaranteed-lower-bound-improvement}
If a path subproblem corresponding to $vw\mhyphen$path $P$ separated by Algorithm~\ref{alg:path-sep} is added to the subproblem set $\mathcal{S}$, the guaranteed improvement of the global lower bound is $\min\{\min_{uv\in P_E}|\theta^P_{uv}|,\min_{uv\in P_{E'}\cup vw}|\theta'^P_{uv}|\}$, where $\theta^P$ is the reparametrized cost used for the path factor initialization.  
\end{proposition}
\begin{proof} 
Algorithm~\ref{alg:path-sep} separates only those subproblems that contain exactly one lifted edge with cost $\theta'^P_{uv}>\varepsilon$ and the rest of the edges have cost lower than $-\varepsilon$. The reparametrized costs of the path factor are fractions of cost reparametrizations obtained by Algorithm~\ref{alg:extract-MM}. 
We have 
\begin{align}
 \nonumber \forall uv\in& P_E:\\
\nonumber &\theta^P_{uv}=\omega_{uv} \cdot( \gamma^{out}_{uv}+\gamma^{in}_{uv})\,,\\
 \nonumber&\theta^{out}_{uv}\minuseq\omega_{uv}\cdot \gamma^{out}_{uv},\, \theta^{in}_{uv}\minuseq\omega_{uv}\cdot \gamma^{in}_{uv},\,\\
 \label{eq:subproblem-cost-init}
 \forall uv\in& P_{E'}:\\
\nonumber &\theta'^P_{uv}=\omega'_{uv} \cdot( \gamma'^{out}_{uv}+\gamma'^{in}_{uv})\,,\\
 \nonumber&\theta'^{out}_{uv}\minuseq\omega'_{uv}\cdot \gamma'^{out}_{uv},\, \theta'^{in}_{uv}\minuseq\omega'_{uv}\cdot \gamma'^{in}_{uv}
\end{align}

We evaluate the change of the lower bounds of all relevant inflow and outflow factors after reparametrization given by Formula~\ref{eq:subproblem-cost-init}. According to Lemma~\ref{lemma:generalLB}, we have
\begin{align}
\Delta& L^{out}+\Delta L^{in}=\\
\nonumber&-\sum_{uv\in P_E:\gamma^{out}_{uv}<0}\omega_{uv}\cdot \gamma^{out}_{uv}
-\sum_{uv\in P_{E'}\cup vw:\gamma'^{out}_{uv}<0}\omega'_{uv}\cdot \gamma'^{out}_{uv}\\
\nonumber&-\sum_{uv\in P_E:\gamma^{in}_{uv}<0}\omega_{uv}\cdot \gamma^{in}_{uv}
-\sum_{uv\in P_{E'}\cup vw:\gamma'^{in}_{uv}<0}\omega'_{uv}\cdot \gamma'^{in}_{uv}\geq\\
\nonumber&-\sum_{uv\in P_E:\gamma^{out}_{uv}+\gamma^{in}_{uv}<0}\omega_{uv}\cdot( \gamma^{out}_{uv}+ \gamma^{in}_{uv})-\\
\nonumber &-\sum_{uv\in P_{E'}\cup vw:\gamma'^{in}_{uv}+\gamma'^{out}_{uv}<0}\omega'_{uv}\cdot(\gamma'^{out}_{uv}+ \gamma'^{in}_{uv})=\\
\nonumber &-\sum_{uv\in P_E:\theta^{P}_{uv}<0}\omega_{uv}\cdot \theta^{P}_{uv}
-\sum_{uv\in P_{E'}\cup vw:\theta'^{P}_{uv}<0}\omega'_{uv}\cdot \theta'^{P}_{uv}
\end{align}
Let $kl\in P_{E'}$ be the only lifted edge with positive cost in the path subproblem. We set $\alpha = \min \{\min\limits_{ij \in P_E} \abs{\theta^P_{ij}},\min\limits_{ij \in P_{E'}\cup vw\setminus kl} \abs{\theta'^P_{ij}}\}$ as in Algorithm~\ref{alg:path-subproblem-optimization}.
If we denote by $\Delta L^P$ the lower bound of the path subproblem, the global lower bound change after adding the path subproblem is:
\begin{equation}
\Delta L= \Delta L^{out}+\Delta L^{in}+\Delta L^P 
 \end{equation}
 If $\alpha<\theta'^P_{kl}$
 \begin{align}
 \nonumber \Delta &L^{out}+\Delta L^{in}+\Delta L^P\geq\\
 &-\sum_{uv\in P_E:\theta^{P}_{uv}<0}\omega_{uv}\cdot \theta^{P}_{uv}
-\sum_{uv\in P_{E'}:\theta'^{P}_{uv}<0}\omega'_{uv}\cdot \theta'^{P}_{uv}+\\
 \nonumber&+\sum_{uv\in P_E:\theta^{P}_{uv}<0}\omega_{uv}\cdot \theta^{P}_{uv}
+\sum_{uv\in P_{E'}:\theta'^{P}_{uv}<0}\omega'_{uv}\cdot \theta'^{P}_{uv}+\\
 \nonumber&+\alpha =\alpha
\end{align}
 If $\alpha\geq\theta'^P_{kl}$
 \begin{align}
  \nonumber \Delta &L^{out}+\Delta L^{in}+\Delta L^P\geq\\
 &-\sum_{uv\in P_E:\theta^{P}_{uv}<0}\omega_{uv}\cdot \theta^{P}_{uv}
-\sum_{uv\in P_{E'}\cup vw:\theta'^{P}_{uv}<0}\omega'_{uv}\cdot \theta'^{P}_{uv}+\\
 \nonumber&+\sum_{uv\in P_E:\theta^{P}_{uv}<0}\omega_{uv}\cdot \theta^{P}_{uv}
+\sum_{uv\in P_{E'}:\theta'^{P}_{uv}}\omega'_{uv}\cdot \theta'^{P}_{uv}=\\
 \nonumber& =\theta'^P_{kl}
\end{align}

\end{proof}

\begin{proposition}[Guaranteed lower bound improvement of cut subproblem]
\label{prop:cut-subproblem-lower-bound-improvement}
If a subproblem corresponding to $vw\mhyphen$cut $C$ separated by Algorithm~\ref{alg:cut-sep} is added to the subproblem set $\mathcal{S}$, the guaranteed improvement of the global lower bound is $\min\{\min_{uv\in C}\theta^C_{uv},|\theta'^C_{vw}|\}$. Where $\theta^C$ is the reparametrized cost used for the cut factor initialization.  
\end{proposition}
\begin{proof}
%\myparagraph{Proof of Lemma \ref{lemma:cut-subproblem-lower-bound-improvement}}
We obtain the reparametrized cost $\theta^C$ for the cut subproblem analogically as in Formula \ref{eq:subproblem-cost-init} for the path subproblem. Note that Algorithm~\ref{alg:cut-sep} ensures that all cut edges in the separated cut subproblem have positive cost and the lifted edge $vw$ has negative cost. Using the same arguments as in the proof of Proposition~\ref{prop:cut-subproblem-lower-bound-improvement}, we obtain the lower bound change of inflow and outflow factors after separating the cut subproblem:

\begin{align}
\nonumber\Delta& L^{out}+\Delta L^{in}=\\
\nonumber&-\sum_{uv\in C:\gamma^{out}_{uv}<0}\omega_{uv}\cdot \gamma^{out}_{uv}
-\sum_{uv\in C:\gamma^{in}_{uv}<0}\omega'_{uv}\cdot \gamma^{in}_{uv}\\
&-\omega'_{vw}\gamma'^{out}_{vw}-\omega'_{vw}\gamma'^{in}_{vw}\geq\\
\nonumber&-\sum_{uv\in C:\gamma^{out}_{uv}+\gamma^{in}_{uv}<0}\omega_{uv}\cdot( \gamma^{out}_{uv}+ \gamma^{in}_{uv})-\omega'_{vw}\gamma'^{out}_{vw}\\
\nonumber&-\omega'_{vw}\gamma'^{in}_{vw}=-\omega'_{vw}\gamma'^{out}_{vw}-\omega'_{vw}\gamma'^{in}_{vw}=-\theta'^C_{vw}
\end{align}

Algorithm~\ref{alg:cut-subproblem-optimization} shows how we obtain the lower bound of the cut subproblem. We set $\theta^C_{ij}=\argmin_{uv\in C}\theta^C_{uv}$. If $\theta^C_{ij}<|\theta'^C_{vw}|$, we get the overall lower bound improvement
\begin{align}
\nonumber\Delta& L^{out}+\Delta L^{in}+\Delta L^{C}\geq -\theta'^C_{vw}+\theta'^C_{vw}+\theta^C_{ij}=\theta^C_{ij}\,.
\end{align}
If $\theta^C_{ij}\geq |\theta'^C_{vw}|$, the lower bound of the cut subproblem is $0$ and the overall lower bound improvement is
\begin{equation}
\Delta L^{out}+\Delta L^{in}+\Delta L^{C}\geq -\theta'^C_{vw}\,.
\end{equation}
\end{proof}

\subsection{Message Passing}\label{sec:message-passing}
One solver run consists of subproblems initialization and a~number of message passing iterations. Algorithm~\ref{alg:message-passing} details the whole run. Algorithms~\ref{alg:lower-bound}-\ref{alg:cut-subproblem-message-passing}  present methods that are called within one iteration.

The number of iterations is predetermined by an input parameter. We use typically tens or maximally one hundred iterations in our experiments.

Algorithm~\ref{alg:message-passing} sends in each iteration, messages between all subproblems in the subproblem set $\mathcal{S}$. Each subproblem creates messages to be sent by computing min-marginals of its variables  (see Formula~\eqref{eq:min-marginal}). These min-marginals are re-scaled and redistributed between other subproblems that contain the respective variables. These operations are called reparametrization. See Section~\ref{sec:Lagrange-decomposition} for details.

\begin{algorithm}
\algnewcommand{\LineComment}[1]{\State \(\triangleright\) #1}
    \caption{Message Passing}
    \label{alg:message-passing}
    \textbf{Input} Graphs $G=(V,E)$ and $G'=(V,E')$, costs $\theta \in \R^{V\cup E\cup E'}$\\
    \textbf{Output} Best found primal solution $(z,y,y')^{\mathsf{ub}}$, lower bound $LB$
    \begin{algorithmic}[1]
    \State \textbf{Initialization:}
   \For{$v \in V$}
   \State Add inflow subproblem for node $v$.
   \State \begin{equation*}
       \forall uv \in \delta^-_E(v):\quad  \theta^{in}_{uv}=\begin{cases}
        \theta_{uv}& \text{if } v=s \\
       \frac{1}{2}\theta_{uv} & \text{otherwise}
       \end{cases}
   \end{equation*}
  % \State $\forall u \in \delta^-_E(v)$: $\theta^{in}_{uv} = \frac{1}{2}\theta_{uv}$.
   \State $\forall uv \in \delta^-_{E'}(v)$: $\theta'^{in}_{uv} = \frac{1}{2}\theta'_{uv}$.
   \State $\theta^{in}_{v} = \frac{1}{2}\theta_{v}$.
   \State Add outflow subproblem for node $v$ with analoguous costs.
   \EndFor
   \State $\mathcal{C} = \emptyset$
   \State $\mathcal{P} = \emptyset$
   \State \textbf{Lagrange decomposition optimization}
   \For{$\mathrm{iter} =1,\ldots ,\mathrm{max\_iter}$}
   \State \textbf{Forward Pass:}
   %\State Call MCF-reparametrization
   \For{$u=u_1,\ldots,u_{\abs{V}}$} 
   \State Inflow-Subproblem-Message-Passing$(u)$
   \State Outflow-Subproblem-Message-Passing$(u)$
   \EndFor
   \For{$P\in \mathcal{P}$}
   \State Path-Subproblem-Message-Passing$(P)$
   \EndFor
   \For{$C\in \mathcal{C}$}
   \State Cut-Subproblem-Message-Passing$(C)$
   \EndFor
   \State \textbf{Backward Pass}:
   \State Revert order of nodes and perform above iteration.
   \If{$\mathrm{iter} \equiv 0  \mod k$}
   \State Separate-Cut-Subproblem$(\varepsilon)$
   \State Separate-Path-Subproblems$(\varepsilon)$
   \State Add  cut and path subproblems to $\mathcal{C}$ and $\mathcal{P}$
    \EndIf
     \If{$\mathrm{iter} \equiv 0  \mod l$}
   \State $(z,y,y')=$Compute-Primal$(\mathcal{S},\theta)$
   \If{$\la \theta, (z,y,y')\ra < \la \theta,(z,y,y')^{\mathsf{ub}} \ra$}
   \State $(z,y,y')^{\mathsf{ub}}=(z,y,y')$
    \EndIf
    \EndIf
    \State $LB=$Lower-Bound
   \EndFor
    \end{algorithmic}
\end{algorithm}

Algorithm~\ref{alg:lower-bound} computes lower bound of the LDP objective by summing up lower bounds of all subproblems. The cost reparametrization realized via our message passing procedures ensures that the lower bound is non-decreasing during the computaiton.

Algorithm~\ref{alg:inflow-subproblem-message-passing} shows sending messages from the inflow subproblem of node $u$. Algorithm~\ref{alg:path-subproblem-message-passing} shows sending messages from a path subproblem. Algorithm~\ref{alg:cut-subproblem-message-passing} presents sending messages from a cut subproblem.

\begin{algorithm}[h]
\caption{Inflow-Subproblem-Message-Passing}
\label{alg:inflow-subproblem-message-passing}
\textbf{Input }central vertex  $u$ of the subproblem
\begin{algorithmic}[1]
  \State $\gamma'^{in} = \text{All-Lifted-MM-In}(u,\theta^{in})$.
  \State $\omega = 1$
  \For{$vu \in \delta^-_{E}(u), P \in \mathcal{P} : vu \in P_E$}
  \State $\gamma^{P}_{vu} = \text{Path-Base-Min-Marginal}(u,v, \theta^{in})$
  \State $\omega \pluseq 1$
  \EndFor
  \For{$vu \in \delta^-_{E'}(u), P \in \mathcal{P} : vu \in P_{E'}$}
  \State $\gamma'^{P}_{vu} = \text{Path-Lifted-Min-Marginal}(u, v,\theta^{in})$
  \State $\omega \pluseq 1$
  \EndFor
    \For{$vu \in \delta^-_{E'}(u), P \in \mathcal{P} : P\in vu\mhyphen\text{paths}(G)$}
  \State $\gamma'^{P}_{vu} = \text{Path-Lifted-Min-Marginal}(u, v,\theta^{in})$
  \State $\omega \pluseq 1$
  \EndFor
  \For{$vu \in \delta^-_{E}(u), C \in \mathcal{C} : vu \in C_{E}$}
  \State $\gamma^{C}_{vu} = \text{Cut-Base-Min-Marginal}(u,v, \theta^{in})$
  \State $\omega \pluseq 1$
  \EndFor
  \For{$vu \in \delta^-_{E'}(u), C \in \mathcal{C} : C \text{ is a } vu \mhyphen\text{Cut}$}
  \State $\gamma'^{C}_{vu} = \text{Cut-Lifted-Min-Marginal}(u,v, \theta^{in})$
  \State $\omega \pluseq 1$
  \EndFor
   
  \For{$vu \in \delta^-_{E'}(u)$}
  \State 
  $\theta^{in}_{vu} \minuseq \frac{1}{\omega} \cdot \gamma'^{in}_{vu},\quad
  \theta^{out}_{vu} \pluseq \frac{1}{\omega} \cdot \gamma'^{in}_{vu}$ 
  \EndFor
  \For{$vu \in \delta^-_{E}(u), P \in \mathcal{P} : vu \in P_E$}
  \State 
  $\theta^{in}_{vu} \minuseq \frac{1}{\omega} \cdot \gamma^{P}_{vu},\quad
  \theta^{P}_{vu} \pluseq \frac{1}{\omega} \cdot \gamma^{P}_{vu}$
  \EndFor
  \For{$vu \in \delta^-_{E'}(u), P \in \mathcal{P} : vu \in P_{E'}$}
  \State 
  $\theta'^{in}_{vu} \minuseq \frac{1}{\omega} \cdot \gamma'^{P}_{vu},\quad
  \theta'^{P}_{vu} \pluseq \frac{1}{\omega} \cdot \gamma'^{P}_{vu}$
  \EndFor
    \For{$vu \in \delta^-_{E'}(u), P \in \mathcal{P} : P\in vu\mhyphen\text{paths}(G)$}
  \State 
  $\theta'^{in}_{vu} \minuseq \frac{1}{\omega} \cdot \gamma'^{P}_{vu},\quad
  \theta'^{P}_{vu} \pluseq \frac{1}{\omega} \cdot \gamma'^{P}_{vu}$
  \EndFor
  \For{$vu \in \delta^-_{E}(u), C \in \mathcal{C} : vu \in C_{E}$}
  \State 
  $\theta^{in}_{vu} \minuseq \frac{1}{\omega} \cdot \gamma^{C}_{vu},\quad
  \theta^{C}_{vu} \pluseq \frac{1}{\omega} \cdot \gamma^{C}_{vu}$
  \EndFor
  \For{$vu \in \delta^-_{E'}(u), C \in \mathcal{C} : C \text{ is a } vu \mhyphen\text{Cut}$}
  \State 
  $\theta'^{in}_{vu} \minuseq \frac{1}{\omega} \cdot \gamma'^{C}_{vu},\quad
  \theta'^{C}_{vu} \pluseq \frac{1}{\omega} \cdot \gamma'^{C}_{vu}$
  \EndFor
\end{algorithmic}
\end{algorithm}

% \begin{algorithm}[h]
% \caption{Path-Subproblem-Message-Passing}
% \label{alg:path-subproblem-message-passing}
% \textbf{Input:} $uv \mhyphen\text{Path } P \in \mathcal{P}$
% \begin{algorithmic}[1]
%   % \State $\gamma^P = \text{Path-Min-Marginals}(P,\theta^P)$
%   \State $\omega^P = \frac{1}{2\cdot \abs{P_E} + 2\cdot \abs{P_{E'}}}$
%   \For{$kl \in P_E$}
%   \State $\gamma_{kl}^P=$Path-Min-Marginal$(P,\theta^P,kl)$
%   \State $\theta^P_{kl} \minuseq 2\omega^P\cdot \gamma^P_{kl}$
%   \State
%   $\theta^{in}_{kl} \minuseq \omega^P \cdot \gamma^P_{kl}$,\quad
%   $\theta^{out}_{kl} \minuseq \omega^P \cdot \gamma^P_{kl}$
%   \EndFor
%   \For{$kl \in P_{E'}$}
%   \State $\gamma_{kl}^P=$Path-Min-Marginal$(P,\theta^P,kl)$
%   \State $\theta^P_{kl} \minuseq 2\omega^P\cdot\gamma^P_{kl}$
%   \State
%   $\theta^{in}_{kl} \minuseq \omega^P \cdot \gamma^P_{kl}$,\quad
%   $\theta^{out}_{kl} \minuseq \omega^P \cdot \gamma^P_{kl}$
%   \EndFor
% \end{algorithmic}
% \end{algorithm}

\begin{algorithm}[h]
\caption{Path-Subproblem-Message-Passing}
\label{alg:path-subproblem-message-passing}
\textbf{Input:} $uv \mhyphen\text{Path } P \in \mathcal{P}$
\begin{algorithmic}[1]
  \State $\gamma^P = \text{Path-Min-Marginals}(P,\theta^P)$
  \State $\omega^P = \frac{1}{2\cdot \abs{P_E} + 2\cdot \abs{P_{E'}}}$
  \For{$kl \in P_E$}
  \State $\theta^P_{kl} \minuseq \gamma^P_{kl}$
  \State
  $\theta^{in}_{kl} \minuseq \omega^P \cdot \gamma^P_{kl}$,\quad
  $\theta^{out}_{kl} \minuseq \omega^P \cdot \gamma^P_{kl}$
  \EndFor
  \For{$kl \in P_{E'}$}
  \State $\theta^P_{kl} \minuseq \gamma^P_{kl}$
  \State
  $\theta^{in}_{kl} \minuseq \omega^P \cdot \gamma^P_{kl}$,\quad
  $\theta^{out}_{kl} \minuseq \omega^P \cdot \gamma^P_{kl}$
  \EndFor
\end{algorithmic}
\end{algorithm}

% \begin{algorithm}[h]
% \caption{Cut-Subproblem-Message-Passing}
% \label{alg:cut-subproblem-message-passing}
% \textbf{Input:} $uv \mhyphen\text{Cut } C \in \mathcal{C}$
% \begin{algorithmic}[1]
%   %\State $\gamma^C = \text{Cut-Min-Marginals}(C,\theta^C)$
%   \State $\omega^C = \frac{1}{2\cdot \abs{C} + 2}$
%   \For{$kl \in C$}
%   \State $\gamma^C_{kl} = \text{Cut-Min-Marginal}(C,\theta^C,kl)$
%   \State $\theta^C_{kl} \minuseq 2\omega^C \cdot \gamma^C_{kl}$
%   \State 
%   $\theta^{in}_{kl} \pluseq \omega^C \cdot \gamma^C_{kl}$,\quad
%   $\theta^{out}_{kl} \pluseq \omega^C \cdot \gamma^C_{kl}$
%   \EndFor
%   \State $\gamma'^C_{uv} = \text{Cut-Min-Marginal}(C,\theta^C,uv)$
%   \State $\theta'^C_{uv} \minuseq 2 \omega^C \cdot \gamma'^C_{uv}$
%   \State $\theta^{in}_{uv} \pluseq \omega^C \cdot \gamma'^C_{uv}$, \quad
%   $\theta^{out}_{uv} \pluseq \omega^C \cdot \gamma'^C_{uv}$
% \end{algorithmic}
% \end{algorithm}

\begin{algorithm}[h]
\caption{Cut-Subproblem-Message-Passing}
\label{alg:cut-subproblem-message-passing}
\textbf{Input:} $uv \mhyphen\text{Cut } C \in \mathcal{C}$
\begin{algorithmic}[1]
  \State $\gamma^C = \text{Cut-Min-Marginals}(C,\theta^C)$
  \State $\omega^C = \frac{1}{2\cdot \abs{C_E} + 2}$
  \For{$kl \in C_E$}
  \State $\theta^C_{kl} \minuseq 2\omega^C \cdot \gamma^C_{kl}$
  \State 
  $\theta^{in}_{kl} \pluseq \omega^C \cdot \gamma^C_{kl}$,\quad
  $\theta^{out}_{kl} \pluseq \omega^C \cdot \gamma^C_{kl}$
  \EndFor
  \State $\theta'^C_{uv} \minuseq 2 \omega^C \cdot \gamma'^C_{uv}$
  \State 
  $\theta^{in}_{uv} \pluseq \omega^C \cdot \gamma'^C_{uv}$, \quad
  $\theta^{out}_{uv} \pluseq \omega^C \cdot \gamma'^C_{uv}$
\end{algorithmic}
\end{algorithm}

\begin{algorithm}
\caption{Lower-Bound}
\label{alg:lower-bound}
\textbf{Input }Subproblems $\mathcal{S}$\\
\textbf{Output }Lower bound value $LB$
\begin{algorithmic}[1]
\State $LB = 0$
\For{$u \in V \backslash \{s,t\}$}
\State $LB \pluseq \text{Opt-In-Cost}(u,\theta^{in})$
\State $LB \pluseq \text{Opt-Out-Cost}(u,\theta^{out})$
\
\EndFor
\For{$P \in \mathcal{P}$}
\State $LB \pluseq \text{Path-Subproblem-Optimization}(P,\theta^P)$
\EndFor
\For{$C \in \mathcal{C}$}
\State $LB \pluseq \text{Cut-Subproblem-Optimization}(C,\theta^C)$
\EndFor

\end{algorithmic}
\end{algorithm}
\vfill\null
\newpage
\vfill\null
\newpage
\vfill\null
\newpage
%\clearpage
\subsection{Primal Solution and Local Search}\label{sec:appendix:primal}
Algorithm~\ref{alg:primal} summarizes the whole procedure for obtaining a~primal solution. As stated in Section~\ref{sec:primal-rounding}, we obtain an initial primal solution by solving MCF problem.

Given a feasible solution of the LDP, Algorithm~\ref{alg:post-process-primal} improves it by splitting and merging paths. While we obtain the costs for MCF from base and lifted edges costs in inflow and outflow factors (Algorithm~\ref{alg:mcf-init}), the local search procedure uses original input costs of base and lifted edges.

%We use the fact that a solution of lifted disjoint paths problem can be uniquely represented by a set of disjoint paths in graph $G=(V,E)$. Algorithm~\ref{alg:post-process-primal} operates over the set $\mathcal{P}$ of disjoint paths resulting from the initial primal solution. The post-processing consists of two steps: paths splitting and path merging.

Algorithm~\ref{alg:split-path} finds candidate split point of each path and recursively splits the path if the split leads to decrease of the objective function.

For each vertex of each path, function $split$ evaluates the cost of splitting the path after the vertex:
\begin{align}
\forall v_{j}&\in P_V=(v_{1},\dots v_{n}): \\
\nonumber&split(v_{j},P)=-\sum_{\substack{k\leq j, l>j,\\ v_{k}v_{l}\in E'}}\theta'_{v_{k}v_{l}}-\theta_{v_{j}v_{j+1}}+\theta_{sv_{j+1}}+\theta_{v_{j}t}
\end{align}

The second step of the primal solution post-processing by Algorithm~\ref{alg:post-process-primal} is merging paths.
Before the path merging itself, some candidate pairs of paths need to be shortened at their ends in order to enable their feasible merging. 

Algorithm~\ref{alg:primal-find-cut-ends-candidates} identifies pairs of those paths whose merging should lead to objective improvement but that cannot be connected directly due to missing base edge between their endpoints. In order to identify the desired paths pairs, several functions are used.

Function $l^+(P_1,P_2)$ resp $l^-(P_1,P_2)$ is the sum of positive resp. negative lifted edges from path $P_1$ to path $P_2$. Function $l(P_1,P_2)$ sums all lifted edges from $P_1$ to $P_2$. 
\begin{align}
\nonumber\forall P_1&,P_2\in \mathcal{P}\\
%&lift(P_1,P_2)=\sum_{uv\in E':u\in P_1,v\in P_2}\theta'_{uv}\\
\nonumber&l^+(P_1,P_2)=\sum_{uv\in E':u\in P_1,v\in P_2,\theta'_{uv\geq 0}}\theta'_{uv}\\
&l^-(P_1,P_2)=\sum_{uv\in E':u\in P_1,v\in P_2,\theta'_{uv< 0}}\theta'_{uv}\\
\nonumber&l(P_1,P_2)=l^+(P_1,P_2)+l^-(P_1,P_2)
\end{align}
We use the above values in functions $merge$ and $merge_{\tau}$ that evaluate the gain of merging two paths. Threshold $\tau\leq 1$ constraints the ratio between the positive and the negative part of lifted cost function $l$ that is considered acceptable for merging two paths.
%\textcolor{red}{TODO: in and out costs to merge}
\begin{align}
\forall P_1=&(v_1,\dots,v_n),P_2=(u_1,\dots,u_m)\in \mathcal{P}\\
\nonumber merge&(P_1,P_2)=\begin{cases}
\theta_{v_nu_1}+l(P_1,P_2) &\text{if } v_nu_1\in E\\
\infty &\text{otherwise}
\end{cases}
\end{align}
% \begin{align}
% \forall P_1=&(v_1,\dots,v_n),P_2=(u_1,\dots,u_m)\in \mathcal{P}\\
% \nonumber merge&(P_1,P_2)=\begin{cases}
% \theta_{v_nu_1}+l(P_1,P_2) &\text{if } v_nu_1\in E\\
%  \quad-\theta_{su_1}-\theta_{v_nt}&  \\
% \infty &\text{otherwise}
% \end{cases}
% \end{align}

\begin{align}
\nonumber\forall P_1=(&v_1,\dots,v_n),P_2=(u_1,\dots,u_m)\in \mathcal{P}\\
merge_{\tau}&(P_1,P_2)=\\
\nonumber=&\begin{cases}
\infty &\text{if } v_nu_1\notin E\ \vee \\
 &\ l^+(P_1,P_2)>\tau|l^-(P_1,P_2)|\\
\theta_{v_nu_1}+l(P_1,P_2)&\text{otherwise}
%\\\quad -\theta_{su_1}-\theta_{v_nt}
\end{cases}
\end{align}

Algorithm~\ref{alg:cut-ends} is applied on all paths pairs found by Algorithm~\ref{alg:primal-find-cut-ends-candidates}. It inspects whether shortening of one or both paths leads to a feasible connection that ensures a desired objective improvement. It iteratively removes either the last vertex of the first path or the first vertex of the second path and checks if a~connection is possible and how much it costs.

The last part of Algorithm~\ref{alg:post-process-primal} considers merging paths.
We use formula $merge_\tau(P_i,P_j)-out(P_i)-in(P_j)$ to evaluate whether merging two paths is beneficial. Here $in(P_j)$ denotes input cost to the first vertex of $P_j$ and $out(P_i)$ denotes output cost from the last vertex of $P_i$. We state the full formula just for completeness. We set the input and the output costs to zeros in our experiments.
Using $merge_\tau$ ensures that we connect the paths only if the ratio between the positive lifted cost $l^+$ and negative lifted cost $l^-$ between the paths is below the acceptable threshold.

%Splitting of paths is done by calling Algorithm~\ref{alg:split-path} that recursively splits input path $P$ until no further splitting is beneficial. 

%TODO: maybe one more case: if positive lifted cost is not within a certain portion of negative cost. Use something like $merge_{thr}$

\begin{algorithm}
\caption{Compute-Primal}
\label{alg:primal}
\textbf{Input } Subproblems $\mathcal{S}$, original costs $\theta \in \mathbb{R}^{|V'|+|E|+|E'|}$\\
\textbf{Output }Primal solution $(z,y,y')$
    \begin{algorithmic}[1]
    \State Init-MCF
    \State Obtain primal solution of MCF  $y^{mcf}\in \{0,1\}^{E^{mcf}}$
    \State Set $(z,y)$ accroding to $y^{mcf}$
    % \State $\forall uv\in E: u\neq s, v\neq t: y_{uv}=y^{mcf}_{u^{out}v^{in}}$
    %   \State 
    %     \begin{varwidth}[t]{\linewidth}$\forall u \in V\backslash \{s,t\}$:
    %  $y_{su}=y^{mcf}_{su^{in}}$, $y_{ut}=y^{mcf}_{u^{out}t}$, \par
    %  $x_u=y^{mcf}_{u^{in}u^{out}}$
    %  \end{varwidth}
     \State $y'=$Adjust-Lifted-Solution$(z,y)$
     \State $(z,y,y')=$Local-Search$(z,y,y')$
    \end{algorithmic}
\end{algorithm}

\begin{algorithm}
\caption{Local-Search}
\label{alg:post-process-primal}
\textbf{Input } Input primal solution $z,y,y'$\\
\textbf{Output } Improved primal solution $z,y,y'$
    \begin{algorithmic}[1]
    \State Obtain set of disjoint paths $\mathcal{P}=\{P_1,\dots, P_n\}$ from $y$
    \ForAll{$P\in \mathcal{P}$}
    \State $\mathcal{P}=$Check-Path-Split$(P_i,\mathcal{P})$
    \EndFor
   \State $\mathcal{P}=$Shorten-For-Merge$(\mathcal{P})$
    \While{$true$}
    \State  \begin{align*}(P_1,P_2)=\argmin_{(P_i,P_j)\in \mathcal{P}\times\mathcal{P}}&merge_\tau(P_i,P_j)\\&-out(P_i)-in(P_j)\end{align*}
    \If{$merge_\tau(P_1,P_2)-out(P_1)-in(P_2)<0$}
    \State $\mathcal{P}=\text{Merge-Paths}(P_1,P_2,\mathcal{P})$ 
    \Else
    \State break
    \EndIf
    \EndWhile
    \State $(z,y,y')$=Set-From-Paths$(\mathcal{P})$
\end{algorithmic}
\end{algorithm}

\begin{algorithm}[h]
\caption{Shorten-For-Merge}
\label{alg:primal-find-cut-ends-candidates}
\textbf{Input } Set of paths $\mathcal{P}$\\
\textbf{Output }Updated set of paths $\mathcal{P}$
    \begin{algorithmic}[1]
    \ForAll{$P_1=(v_1,\dots,v_n)\in \mathcal{P}$}
    \State $P=\argmin\limits_{P_2=(u_1,\dots,u_m)\in \mathcal{P}:v_nu_1\in E}merge(P_1,P_2)$
    %\State $P'=\argmin\limits_{\substack{P_2=(u_1,\dots,u_m)\in \mathcal{P}:v_nu_1\notin E,\\ l^+(P_1,P_2)\leq\tau|l^-(P_1,P_2)|}}l(P_1,P_2)$
     \State $P'=\argmin\limits_{P_2=(u_1,\dots,u_m)\in \mathcal{P}:v_nu_1\notin E}l(P_1,P_2)$
    \If{$l(P_1,P')<merge(P_1,P)\wedge l(P_1,P')<0$}
    \State $P^*=P'$, $c=l(P_1,P')$
    \Else
    \State $P^*=P$, $c=merge(P_1,P)$
    \EndIf
    \If{pred$[P^*]=\emptyset\vee \text{score}[P^*]>c$}
    \State pred$[P^*]=P_1$, score$[P^*]=c$
    \EndIf
    \EndFor
    \ForAll{$P_2=(u_1,\dots,u_m)\in \mathcal{P}$}
    \If {pred$[P_2]$=$P_1=(v_1,\dots,v_n)\wedge v_nu_1\notin E$}
    \State $\mathcal{P}=$Cut-Ends$(P_1,P_2,\mathcal{P})$
    \EndIf
    \EndFor
\end{algorithmic}
\end{algorithm}

% \begin{algorithm}
% \caption{Get-Cut-Ends-Candidates}
% \label{alg:primal-find-cut-ends-candidates}
% \textbf{Input:} Set of paths $\mathcal{P}$\\
% \textbf{Output: }Set of paths pairs cut\_ends\_candidates
%     \begin{algorithmic}[1]
%     \ForAll{$P_1=(v_1,\dots,v_n)\in \mathcal{P}$}
%     \State $P=\argmin\limits_{P_2=(u_1,\dots,u_m)\in \mathcal{P}:v_nu_1\in E}merge(P_1,P_2)$
%     %\State $P'=\argmin\limits_{\substack{P_2=(u_1,\dots,u_m)\in \mathcal{P}:v_nu_1\notin E,\\ l^+(P_1,P_2)\leq\tau|l^-(P_1,P_2)|}}l(P_1,P_2)$
%      \State $P'=\argmin\limits_{P_2=(u_1,\dots,u_m)\in \mathcal{P}:v_nu_1\notin E}l(P_1,P_2)$
%     \If{$l(P_1,P')<merge(P_1,P)\wedge l(P_1,P')<0$}
%     \State desc$[P_1]=P'$
%     \Else
%     \State desc$[P_1]=\emptyset$
%     \EndIf
%     \EndFor
%     \ForAll{$P_2=(u_1,\dots,u_m)\in \mathcal{P}$}
%     \State $P=\argmin\limits_{P_1=(v_1,\dots,v_n)\in \mathcal{P}:v_nu_1\in E}merge(P_1,P_2)$
%     \State $P'=\argmin\limits_{P_1=(v_1,\dots,v_n)\in \mathcal{P}:v_nu_1\notin E}l(P_1,P_2)$
%     \If{$l(P',P_2)<merge(P,P_2)\wedge \text{desc}[P']=P_2$}
%     \State cut\_ends\_candidates.insert$(P',P_2)$
%     \EndIf
%     \EndFor
%     \State return cut\_ends\_candidates
% \end{algorithmic}
% \end{algorithm} 

\begin{algorithm}[t]
\caption{Cut-Ends }
\label{alg:cut-ends}
\textbf{Input } $P_1=(v_1,\dots,v_m),P_2=(u_1,\dots,u_m)$, $\mathcal{P}$, $i_{max}$\\
\textbf{Output } New set of paths $\mathcal{P}$
    \begin{algorithmic}[1]
    \State $c_1=\infty, c_2=\infty$
    \While{$i_1+i_2< i_{max}$}
    \State $P'_1=(v_1,\dots,v_{n-i_1})$, $P'_2=(u_{1+i_2},\dots,u_m)$
    \State $P''_1=(v_1,\dots,v_{n-i_1-1})$, $P''_2=(u_{2+i_2},\dots,u_m)$
    \If{$merge(P'_1,P''_2)+ merge(P''_1,P'_2)<\infty$}
    \State \begin{varwidth}[t]{\linewidth} $\alpha_1=merge(P'_1,P''_2)+split(P_1,v_{n-i_1})+split(P_2,u_{1+i_2})$
    \end{varwidth}
    \State \begin{varwidth}[t]{\linewidth} $\alpha_2=merge(P''_1,P'_2)+split(P_1,v_{n-i_1-1})+split(P_2,u_{i_2})$
    \end{varwidth}
    \IIf{$\alpha_1<\alpha_2$} $c_1=i_1-1,c_2=i_2$ 
    \EElse $c_1=i_1,c_2=i_2-1$
    \State break
    \ElsIf{$merge(P'_1,P''_2)<\infty$}
    \State $c_1=i_1-1,c_2=i_2$ 
    \State break
    \ElsIf{$merge(P''_1,P'_2)<\infty$}
    \State  $c_1=i_1,c_2=i_2-1$
    \State break 
    \Else
      \State \begin{varwidth}[t]{\linewidth} $\alpha_1=l(P'_1,P''_2)+split(P_1,v_{n-i_1})+$\\$split(P_2,u_{1+i_2})$
    \end{varwidth}
    \State \begin{varwidth}[t]{\linewidth} $\alpha_2=l(P''_1,P'_2)+split(P_1,v_{n-i_1-1})+$\\$split(P_2,u_{i_2})$
    \end{varwidth}
    \IIf{$\alpha_1<\alpha_2$} $i_2\plusplus$ 
    \EElse $i_1\plusplus$
    \EndIf
    \EndWhile
    \If{$c_1\neq \infty\wedge c_2\neq \infty$}
    \State $P'_1=(v_1,\dots,v_{n-c_1})$,  $P'_2=(u_{1+c_2},\dots,u_m)$
    \If{$merge_\tau(P'_1,P'_2)<\infty$}
    \If{$c_1>0$} 
    \State $(P_{11},P_{12})=$Split-Path$(P_1,v_{n-c_1})$
    \State $\mathcal{P}$.remove$(P_1)$
    \State $\mathcal{P}$.insert$(P_{11})$, $\mathcal{P}$.insert$(P_{12})$
    \EndIf
      \If{$c_2>0$} 
    \State $(P_{21},P_{22})=$Split-Path$(P_2,u_{c_2})$
    \State $\mathcal{P}$.remove$(P_2)$
    \State $\mathcal{P}$.insert$(P_{21})$, $\mathcal{P}$.insert$(P_{22})$
    \EndIf
    \EndIf
    \EndIf
    \State return $\mathcal{P}$
    \end{algorithmic}
\end{algorithm}

\begin{algorithm}[H]
\caption{Check-Path-Split }
\label{alg:split-path}
\textbf{Input } Input path $P$, set of all paths $\mathcal{P}$\\
\textbf{Output } Set of paths $\mathcal{P}$
    \begin{algorithmic}[1]
    \State $v_{m}=\argmax_{v_{j}\in P_V} split(v_{j},P)$
    \If{$split(v_{m},P)<0$} 
    \State $(P_{1},P_{2})=$Split-Path$(P,v_{m})$
    \State $\mathcal{P}$.remove$(P)$, $\mathcal{P}$.insert$(P_{1})$, $\mathcal{P}$.insert$(P_{2})$
    \State $\mathcal{P}=$Check-Path-Split$(P_{1},\mathcal{P})$
    \State $\mathcal{P}=$Check-Path-Split$(P_{2},\mathcal{P})$
    \EndIf
    \State return $\mathcal{P}$
    \end{algorithmic}
\end{algorithm} 

\newpage
\subsection{Global Context Normalization}
\label{appendix:feature_scaling}
Our tracking system employs a global context normalization to obtain accurate features between detections (see Section \ref{sec:cost_classifier}). This section elaborates the implementation details.

Global context normalization puts similarity measurements into global perspective to form more meaningful feature values. For instance, global illumination changes will likely decrease measured appearance similarities
between any pair of detections. Likewise, a scene where all people are far away from the camera will most likely result in less confident appearance similarity measurements. 
In both cases, positive matching pairs should have higher similarities than negative matching pairs but the absolute similarity values are reduced by the global effects.
These and further effects make the interpretation of the similarity in absolute terms less meaningful. 
Global context normalization compensates such effects.

To this end, let $\Omega_{k} = \{\sigma_{vw,k} : vw \in E\}$ comprise all computed similarity measurements for feature $k \in \{\text{Spa},\text{App}\}$ defined in Section \ref{sec:cost_classifier}.
For each feature $k$ and similarity measurement $\sigma_{vw,k} \in \Omega_{k}$, we define  sets $\mathrm{GC}_{i,k}$ with $i \in [5]$. Each set $\mathrm{GC}_{i,k}$ induces two global context normalization features:
$\sigma_{vw,k} \cdot \max(\mathrm{GC}_{i,k})^{-1}$ and $\sigma_{vw,k}^{2} \cdot \max(\mathrm{GC}_{i,k})^{-1}$.

The sets are defined w.\ r.\ t.\ $\sigma_{vw,k} \in \Omega_{k}$ as follows:
\begin{align}
   \mathrm{GC}_{1,k} &= \{s_{vn,k} \in \Omega_{k} : n \in B\} \, . \\
   \mathrm{GC}_{2,k} &= \{s_{mw,k} \in \Omega_{k} : m \in B\}\, . \\
   \mathrm{GC}_{3,k} &= \{s_{vn,k} \in \Omega_{k} : n \in B \text{ and } f_n = f_w\}\, . \\
   \mathrm{GC}_{4,k} &= \{s_{nw,k} \in \Omega_{k} :  n \in B \text{ and } f_n = f_v\}\, . \\
   \mathrm{GC}_{5,k} &= \{s_{mn,k} \in \Omega_{k} :   m,n \in B\}\, .
\end{align}
where $f_x$ denotes the frame of detection $x$ and $B$ the batch as defined in section \ref{sec:cost_classifier}. The sets $\mathrm{GC}_{1,k}$ and $\mathrm{GC}_{2,k}$ result in a normalization of the similarity score $\sigma_{vw,k}$ over all outgoing or incoming edges to $v$ or $w$, respectively. The set $\mathrm{GC}_{3,k}$ results in a normalization over all similarity scores for outgoing edges from $v$ to a detection in frame $f_{w}$. Analogously, the set $\mathrm{GC}_{4,k}$ collects all edges from a detection of frame $f_{v}$ to node $w$. Finally, $\mathrm{GC}_{5,k}$ normalizes the similarity score over all existing scores in the batch $B$.

\subsection{Multi Layer Perceptron (MLP)}
\label{appendix:classifier}
As reported in Section~\ref{sec:cost_classifier} we use a lightweight and scalable MLP to obtain edge costs. We use multiple instances of the same MLP structure. Each MLP is trained on edges that have a specific range of temporal gaps (more details in Section~\ref{sec:cost_classifier}).

The MLP architecture is based on two fully connected (FC) layers. The input is a $22$-dimensional vector (features with corresponding global context normalizations). LeakyReLU activation \cite{maasrectifier} is used for the first FC layer. The final layer (FC) has one neuron, whose output represents the cost value. For training, an additional sigmoid activation is added. The structure of the neural network is visualized in Figure~\ref{fig:classifier}.

% The classifier architecture is based on neural network layers. It consists of two consecutive fully connected layers (FC). The input of the first FC are the features and scalings, resulting at input size of 22. It has the same number of neurons as the input size, followed by a LeakyReLU activation \cite{maasrectifier}. The second FC has one neuron, which output represent the cost value. For training, an additional sigmoid activation is added. The structure of the neural network is visualized in figure~\ref{fig:classifier}. 

\begin{figure}[t]
    \centering
    \includegraphics{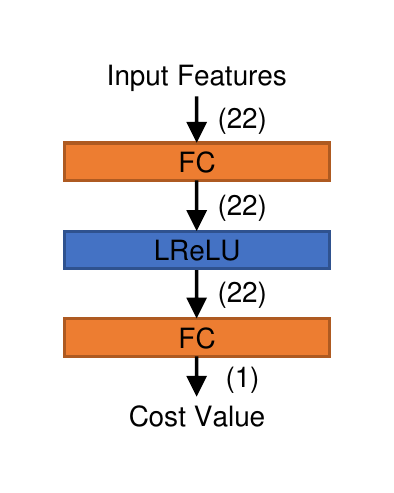}
    \caption{Visualisation of the proposed neural MLP. FC denotes fully connected layers, LReLU denotes LeakyReLU activation \cite{maasrectifier} and values in parenthesis denote the dimension of the corresponding tensors.}
    \label{fig:classifier}
\end{figure}

\subsection{Batch Creation Using Fixed Frame Shifts}
\label{appendix:sampling}
In this section, we elaborate on our batch creation method using fixed frame shifts (see Paragraph \textit{Training} of Section~\ref{sec:cost_classifier}).

A batch $B$ contains a set of edges with corresponding edge costs. Important to note is that we use the same batch creation strategy to form batches for training as well as inference. During training, batches are augmented with corresponding ground truth labels.
A carefully chosen batch creation strategy is crucial to ensure accurate and scalable training and inference. 

In order to obtain accurate predictions by our MLPs (Section~\ref{sec:cost_classifier}), the distribution of a batch should represent the characteristics of the training data. In particular, a batch should comprise edges covering all permissible temporal gaps between detections. Furthermore, the distribution of a batch influences the costs of all edges contained in the batch by the global context normalization (Section \ref{appendix:feature_scaling}). Thus also during inference, a batch should comprise edges covering all permissible temporal gaps between detections. 
It is thus important to employ the same batch creation strategy for training and inference. 

A na\"ive  strategy is thus to define a batch on all frames within a range $f$ up to $f + \Delta f_\text{max}$, where $f$ is a starting frame and $\Delta f_\text{max}$ defines the maximal permissible time gap. During training, one could then sample detections from these frames and randomly create true positive and false positive edges. However, such an approach is not tractable during inference for long sequences with many detections and long time gaps. To see this, consider the  sequence MOT20-05 of the MOT20 dataset~\cite{MOTChallenge20}. It contains $3315$ frames with $226$ detections per frame on average. We use a maximal permissible time gap of $2$ seconds, which correspond to $50$ frames for sequence MOT20-05. The number of edges per batch and for the entire sequence can then be roughly estimated\footnote{With $226$ detections per frame, there are about $226\sum_{i=1}^{49} 226 (50-i) = 62568100$ edges per batch. With $3315$ frames there are more than
$\frac{3315}{50}62568100 = 4148265030$ edges. Here we assumed non-overlapping batches, thereby missing many connections.} with  $63 \cdot 10^{6}$ and  $4 \cdot 10^9$, respectively, which is intractable. 

%Also, this strategy leads to an unbalancing between edges with small and large $\Delta f$, which can negatively affect the accuracy of the classifier.    

To decrease the amount of edges per batch while ensuring that batches consists of samples containing all permissible temporal gaps, we adapt batch creation to our needs: For each start frame $f$ of a batch, we subselect the frames to be considered within the range $f,\ldots, f_\text{max}$. During training, we then subsample detections from these frames. During inference, we utilize all detections of these frames to form our batch.

To this end, we define a sequence of frame shifts that is used to create the frame subselection. Using only few frame shift makes the approach more computationally efficient. Yet we must ensure that all edges are computed at least once during inference. That is if $B(f)$ denotes the batch created according to our strategy with starting frame $f$, then $\bigcup_{f} B(f)=E$ must contain all edges.

To ensure that we always cover temporal gaps of up to $2$ seconds, the frame shifts depend on the maximal  permissible temporal gaps (measured in frames).

For a start frame $f$ and $\Delta f_\text{max}=50$, we define the frame shift set $\mathrm{Sh}(\Delta f_\text{max})$ as
\begin{equation}
\mathrm{Sh}(50)=\{0, 1, 2, 3, 4, 5, 6, 7, 8, 17, 26, 35, 44, 50\} \, .
\end{equation}
For $\Delta f_\text{max}=60$, we define the frame shift set $\mathrm{Sh}(\Delta f_\text{max})$ as
\begin{equation}
\longset{\mathrm{Sh}(60)}{
    0, 1, 2, 3, 4, 5, 6, 7, 14, 21, 28, 36, 37, 38, 40, 46, 53, 60
    } 
\end{equation}
Then for each start frame $f$, a batch is created using the frames $\{f+f_{\mathrm{shift}} : f_{\mathrm{shift}} \in \mathrm{Sh}(f_{\mathrm{max}}) \}$.
To ensure that all edges are computed at least once in the inference stage, one need to calculate batches with start frames $f\in\{1-\Delta f_\text{max}, \ldots, \Delta f_\text{max}\}$. 
Compared to the na\"ive batch creation strategy,  our utilized batch creation results in  $226\sum_{i=1}^{13} 226\cdot (14-i) = 4647916$ edges for MOT20-20, using the same assumptions as before. The number of edges to be computed is thus significantly lower.

\subsection{Determining obviously matching and non-matching  detection pairs}
\label{appendix:high_confident}
During the graph construction in Section \ref{sec:graph_construction}, we employ a simply strategy to detect edges that represent obviously matching or obviously non-matching detection pairs. Corresponding edge costs are set such that they induce must-links or cannot-links as soft constraints. Details are described in this section. 

%\myparagraph{Obviously false edges.}
\myparagraph{Obviously non-matching pairs.}
We use optical flow and the object size to calculate the maximal plausible displacement $d_\text{max}(v, w)$ and velocities $\mathrm{v}_{x,\text{max}}(v, w)$ and $\mathrm{v}_{y,\text{max}}(v, w)$ between two detections $v$ and $w$. 
If $v$ is a detection in frame $f_v$ and $w$ a detection in $f_w$, we define \begin{equation}
    d_\text{max}(v, w) = k_d + \sum_{i=f_v}^{f_w - 1} \max(O(i, {i+1}))
\end{equation}
where $\max(O(i, i+1))$ is the maximal magnitude of the optical flow between the frames $i$ and $i+1$ and $k_d$ is a security tolerance, which we set to $175$ pixel according to experiments on the training data. If the distance $d(v, w)$ between the center points of detections $v$ and $w$  is greater than $d_\text{max}(v,w)$, the detection pair given by $vw$ is regarded as obviously non-matching. 
We also assume that the maximal velocity of a person is limited. With the height $h$ and width $b$ of the bounding boxes, we define the maximal velocities \begin{equation}
    \mathrm{v}_{x,\text{max}}(v, w) = b \cdot k_e
\end{equation}
and \begin{equation}
    \mathrm{v}_{y,\text{max}}(v, w) = h \cdot \frac{k_e}{2}
\end{equation}
in $x$ and $y$-direction. The factor $k_e$ is set to $k_e=0.3$ according to experiments on the training data. In sequences with moving cameras, the factor is increased to $k_e=0.8$ and decreased to $k_e=0.12$ in sequences with static camera and aerial viewpoint. If the velocity $\mathrm{v}_x(v, w)$ or $\mathrm{v}_y(v, w)$ between the detections $v$ and $w$  is greater than corresponding $\mathrm{v}_{x,\text{max}}(v,w)$ or $\mathrm{v}_{y,\text{max}}(v,w)$, the connection given by $vw$ is regarded as obviously non-matching. To avoid wrong interpretations caused by noise in the velocity calculation, we set $k_e$ to a high value for detection pairs with small temporal distances. 

If a detection pair $vw$ is regarded as obviously non-matching, we induce a cannot-link soft constraint on $vw$ by setting its costs to a negative value with a high absolute value, \ie  $c_{vw} \ll 0$. 

%\myparagraph{Obviously true edges.}
\myparagraph{Obviously matching pairs.}
We induce must-link soft constraints on edges, considering only connections between consecutive frames.

An edge $vw \in E$ with an appearance similarity score $\sigma_{vw,\text{App}}$ close to the maximal achievable  score (which is $2$ in our implementation) is considered an obviously matching pair. We infer from the training data $k_{s} = 1.95$ as threshold, so that edges with $\sigma_{vw,\text{App}} > k_{s}$ are regarded as obviously matching by setting their costs accordingly.

In addition, if two boxes between consecutive frames have a high overlap, we induce a must-link soft constraint on the corresponding edge. In more detail, the intersection over union between the detections must be at least $0.5$. However, such spatial measurements are affected by camera motions, thus potentially leading to wrong interpretations. In order to induce link soft-constraints only in confident cases,  we employ a simple camera motion compensation beforehand. To this end, we calculate for a considered edge the mean magnitude given by the optical flow between the frames of the respective detections. Before we compute the intersection over union, we translate one of the boxes in horizontal direction by the approximated camera motion, if this decreases the intersection over union. This procedure lowers the likelihood of creating wrong must-links caused by camera motion. Camera motion compensation needs to be applied 
only to sequences filmed from a non-static camera. Optical flow can be used to detect if a scene has a static camera setup.  Note that MOT20 contains only scenes filmed from a static camera.

\subsection{Inference}
\label{sec:inference}
\subsubsection{Interval Solution}\label{sec:intervalSolution}
This section explains how we solve MOT20 using solutions of its intervals. First, we solve the problem in independent subgraphs containing detections and edges from time intervals $[il+1,(i+1)l]$ for  $i\in\{0,1\dots,n\}$, where $l=3\cdot t_\text{max}$, and  $t_\text{max}$ is the maximum temporal  edge length.
 We fix resulting trajectories in the centres of intervals, namely in time intervals $[il+t_{max}+1,(i+1)l-t_{max}]$ for $i \in \{1,\dots,n-1 \}$. Second, we solve the problem in time intervals covering the end of one initial interval and the beginning of the subsequent interval while allowing connections with the fixed trajectory fragments. The cost of a connection between a detection and a trajectory fragment is obtained as the sum of costs between the unassigned detection and the detections within the trajectory fragment.

\subsubsection{Post-Processing}
\label{appendix:post_processing}
We perform post-processing on the result provided by our solver.
As it is common, we recover missing detections within a computed trajectory using linear interpolation. 
We also correct  wrong connections that mostly stem from situations which currently cannot be correctly resolved by current features, independent of the tracking system, \eg pairwise features computed  over very long temporal gaps and ambiguous feature information due to multiple people appearing within one detection box. 

Consequently, we apply these strategies only to MOT20. As soon as one of these methods detects a connection as false, the corresponding trajectory is split into two new trajectories. Table~\ref{tab:postprocessing} shows the effect of the post-processing on MOT20 train set.
\begin{center}
\hspace{-4pt}\begin{table}[]
\caption{Results with and w/o post-processing on MOT20 train set.}
\label{tab:postprocessing}
\small{
\begingroup
\setlength{\tabcolsep}{5.5pt} % Default value: 6pt
\begin{tabularx}{\columnwidth}{ccccccc}
\toprule
        & MOTA$\uparrow$ & IDF1$\uparrow$ & TP$\uparrow$ & FP$\downarrow$ & FN $\downarrow$ & IDS $\downarrow$ \\ % & Recall& Precision  
\midrule
w       & \textbf{74.4}  & 62.8 & \textbf{863203}  & 15778 & \textbf{271411} & \textbf{3511}    \\ %  76,1 & 98,2
w/o     & 72.3  & \textbf{63.6}& 833473  & \textbf{8462} & 301141 & 4201   \\ % 73,5 &  99,0
\bottomrule
%\vspace*{-6mm}
\end{tabularx}
\endgroup
}
\end{table}
\end{center}
We noticed an accumulation of wrong connections, where one end of a trajectory (\ie its first or last detection) is connected to the successive detection using a skip-connection over a long temporal gap, and the connection is wrong.
This might be explained by a  combination of misleading visual features (\eg caused by partial occlusion), not very informative spatial features (due to the high temporal gap) and missing lifted edges, because only one detection is existent at the end of the trajectory. To keep only reliable connections, we split trajectories if only one detection is existing at the start or end of a trajectory, followed by a temporal gap of at least 10 frames.

We also handle cases at the borders of a tracking scene. If a person leaves the scene, and another person enters the scene 
at a position close by
after a short time, the tracking system sometimes joins the trajectories of the two persons. We explain this behaviour by the high visual similarity between partially visual persons at image borders. If a person leaves a scene, normally just one leg, one arm or the head is visible for some frames. However, a single body party is not very discriminative and thus can look similar to a body part of another person. In addition, the spatio-temporal information will indicate a likely match in this scenario. Due to the temporal gap, no or not many meaningful lifted edges are existing which could give contradicting signals. To eliminate this kind of errors, we split trajectories between two detections, if the temporal gap is greater or equal to 10 and both detections are located at the image border.

For all detections which are connected over a temporal time gap greater or equal to 10 frames (skip edges), we perform a motion sanity check and split the corresponding trajectory if its motion is not plausible. To this end, we first determine the highest velocity of obviously correct trajectories, or parts of trajectories with a minimal length of 10 frames (to avoid random noise issues). Then,
we split connections at these skip edges, if their velocity is higher than the determined velocity.

Additionally, we verify that motion between trajectory parts are consistent. To this end, we consider the motion described by a trajectory, using the part before a connection, and compare it with the resulting motion described by the trajectory, using the part after the connection. If the velocities differ by a factor of 5 or greater or if the angle differs more than $\pi/2$, the trajectory is split.

\subsection{Solver Runtime}
\label{appendix:runtime}
Our solver can compute one interval of MOT20 (150 frames) or an entire sequence of MOT17 with  less than 20GB RAM, using a single CPU core.

Subsequently, we analyze the runtime in detail, by performing a theoretical analysis in Section~\ref{sec:theoretical_runtime}, 
followed by a comparison with an existing LDP solver in Section~\ref{sec:runtime_compared_with_lift}.
% followed by an extensive performance comparison with an existing LDP solver in Section~\ref{sec:runtime_compared_with_lift}. 

\subsubsection{Computational Complexity}
%\subsubsection{Theoretical Time Complexity}
\label{sec:theoretical_runtime}
The solver terminates if one of the following conditions is satisfied. Either the lower bound is equal to the objective value of the best primal solution, i.e.\ optimum has been found. Or the maximum number of message passing iterations has been reached. The optimum was not found in our experiments, so the letter condition applied.

The runtime of the solver is, therefore, determined by the input parameter denoting the maximum number of iterations. The dependence on number of iterations is not exactly linear because the problem size grows with the number of path and cut subproblems added to set of subproblems $\mathcal{S}$ via cutting plane separation (see Sections~\ref{sec:cutting-planes-path-subproblems} and~\ref{sec:cutting-planes-cut-subproblems}).

An overview of the whole solver run and the tasks performed within one its iteration is given in Section~\ref{sec:message-passing}. The runtime of the tasks is given by the runtime of computing min-marginals of the subproblems.

% Cutting plane procedures and primal rounding are not called in each iteration. Therefore, we describe their complexity in separate paragraphs.

We discuss the complexity of the used algorithms in the paragraphs bellow. They all have a polynomial complexity. Therefore, the overall runtime of the solver is polynomial too.

In order to compute messages between the inflow and the outflow subproblem, we apply Algorithm~\ref{ap:alg:all-lifted-mm}. Min-marginals for messages between the path subproblems and the in/outflow subproblems are obtained for one shared variable at the time.  The same holds for exchanging messages between the cut subproblems and the in/outflow subproblems. This is done by calling restricted versions of optimization algorithms of the path and cut subproblems (Algorithms~\ref{alg:path-subproblem-optimization} and \ref{alg:cut-subproblem-optimization}). For in/outflow subproblems, we use one call of Algorithm~\ref{alg:outflow-minimization} followed by either Algorithm~\ref{ap:alg:outflow-minimization} or Algorithm~\ref{ap:alg:BU-main} limited to single variable reparametrization.

\myparagraph{Messages between inflow and outflow subproblems. }Messages between inflow and outflow subproblems are realized on lifted edge variables by calling Algorithm~\ref{ap:alg:all-lifted-mm}. Many subroutines employ full or partial DFS on all nodes reachable from the central node within the relevant time gap. In these cases, we use precomputed node order instead of complete DFS as described in the last paragraph of Section~\ref{sec:outflow-subproblem-min-marginals}. One call of the full DFS (Algorithms~\ref{alg:outflow-minimization} and~\ref{ap:alg:BU-main}) requires to process all vertices reachable from $v$ within maximal time gap for edge length $(\Delta f_{max})$. This comprises $L_{max}$ video frames (we use $L_{max}=50$ or $60$). Let us denote by $n$ the maximum number of detections in one frame. The complete DFS processes maximally $nL_{max}$ vertices. Incomplete DFS used in Algorithm~\ref{ap:alg:outflow-minimization} processes in each step vertices in $L$ layers. In the worst case, this is done for all relevant layers $L=1,\dots,L_{max}$.  Processing one vertex requires to check its neighbors in the base graph. Their amount is bounded by $K L_{max}$ where $K=3$. See Sparsification paragraph in Section~\ref{sec:experiments}. Putting all together, the complexity of Algorithm~\ref{ap:alg:all-lifted-mm} for one subproblem is $\mathcal{O}(nL_{max}^3)$. We have two subproblems for each (lifted) graph vertex, yielding complexity $\mathcal{O}(|V'|nL_{max}^3)$ for sending messages between all inflow and outflow subproblems in one message passing iteration.

\myparagraph{Messages from path subproblems.} Obtaining min marginal for one edge variable of a path subproblem requires two calls of restricted  Algorithm~\ref{alg:path-subproblem-optimization} whose complexity is linear in the number of path edges. So, min-marginals for all path edges are obtained in $\mathcal{O}(|P|^2)$.

\myparagraph{Messages from cut subproblems. } Min marginal of one variable of a cut subproblems is obtained by adjusting its optimization Algorithm~\ref{alg:cut-subproblem-optimization}. The complexity is given by the complexity of the employed linear assignment problem which can be solved in polynomial time.  

\myparagraph{Cutting plane procedures.}  The cutting plane algorithms are called each 20 iterations. We allow to add maximally $0.5 \cdot |\mathcal{S}_0|$ new factors during one separation call, where $\mathcal{S}_0$ is the initial set of subproblems containing only inflow and outflow factors. So it holds, $|\mathcal{S}_0|=2|V'|$. Once added, the subproblems influence the runtime via taking part in the message passing (see Section~\ref{sec:message-passing}). Cutting plane itself (Sections~\ref{sec:cutting-planes-path-subproblems} and~\ref{sec:cutting-planes-cut-subproblems}) contains sorting of subsets of base or lifted edges which has complexity $\mathcal{O}(|E^-|\log |E^-|)$ (resp. $\mathcal{O}(|E^+|\log |E^+|)$). The other algorithms run in quadratic time w.r.t.\ number of vertices within relevant time distance to the currently processed edge.

\myparagraph{Primal solution. }We compute new primal solution in each five iterations. We use Algorithm~\ref{alg:mcf-init} for obtaining base edge costs. Then, we use successive shortest paths algorithm for solving  minimum cost flow problem and finally local search heuristic given by Algorithm~\ref{alg:post-process-primal}, see Section~\ref{sec:primal-rounding}. The complexity of solving MCF is the complexity of successive shortest path algorithm which is polynomial. Local search  heuristic requires to compute and update cummulative costs between candidate paths. They can be computed in time linear in the number of lifted edges $\mathcal{O}(|E'|)$.  MCF costs are obtained by calling Algorithm~\ref{alg:outflow-minimization}. Its complexity is discussed above.

\subsubsection{Comparison with LifT}
\label{sec:runtime_compared_with_lift}
We perform several experiments for comparing our solver with an optimal solver for lifted disjoint paths LifT~\cite{hornakova2020lifted}.

\myparagraph{LifT global solution vs. two-step procedure. }LifT is based on ILP solver Gurobi. It solves the LDP problem optimally. However, it is often not able to solve the problem on the full graphs. Therefore, LifT uses a two-step procedure. First, solutions are found on  small time intervals to create tracklets. Second, the problem is solved on tracklets. This approach simplifies the problem significantly but the delivered solutions are not globally optimal anymore. We have observed that using our input costs, LifT is able to solve some problem sequences globally without the two step-procedure. Therefore, we compare our solver with  LifT using both the two-step procedure and the global solution. 

\myparagraph{Influence of input costs. } Our input data contain many soft constraints for obviously matching pairs of detections. Those are edges with negative costs significantly higher in absolute value than other edges costs. LifT finds an initial feasible solution using only base edges. This solution may be already very good due to the costs of obviously matching pairs. Moreover, Gurobi contains a lot of efficient precomputing steps, so it can recognize that the respective variables should be active in the optimum and reduce the search space.  

\myparagraph{Parameters. } We adjust parameters of our solver to work with comparable data as LifT. For instance, we do not set cost of any base edges to zero (as described in Section~\ref{sec:graph_construction}) because LifT does not enable this option. So, the costs of overlapping base and lifted edges are duplicated as opposed to the most of other experiments. Moreover, if there is no edge between two detections within the maximal time distance in the input data, we can add a~lifted edge with high positive cost for such pair in ApLift. This is useful for reducing the input size for MOT20 dataset. LifT does not have this option. Therefore, we disable this option for ApLift too.

\myparagraph{Subsequences of MOT20. } We present a comparison between our solver and LifT using two-step procedure on an example subsequence of MOT20-01 in Table~\ref{tabel:runtime} in the main text. On that subsequence, our solver is faster and has even slightly better IDF1 score than LifT. 
%In Table~\ref{tabel:runtime02}, we present a comparison on another example subsequence of MOT20 where LifT finds solutions faster than our solver using many iterations.
In Table~\ref{tabel:runtime02}, we present a comparison on first $n$ frames  of sequence MOT20-02 where LifT finds solutions faster than our solver using many iterations.
We assume that this is caused by the input costs that are convenient for Gurobi, see the discussion above. 
%The persons appearing in the first parts of sequence MOT20-02 are less dense, even though the average density of persons per frame for the whole sequence is higher than in the sequence evaluated in Table~\ref{tabel:runtime}. Also the most persons in the first 200 frames move into the same direction, which is benefitial for our heuristics used to find obviously true connections.

\myparagraph{Train set of MOT17.} We compare our solver with LifT on global training sequences of MOT17. That is, we do not use two-step procedure. Therefore, LifT finds the globally optimal solution if it finishes successfully. 
%We performe the experiments on a machine having 2000 GB RAM, 255 CPUs each having 64 cores.
The runtime of LifT is exponential in general and it can be often killed because of memory consumption if run on global sequences. Therefore, we perform these experiments on a machine having 2000 GB RAM and multiple CPUs each having 64 cores.
%So, the costs of overlapping base and lifted edges was duplicated as opposed to our other experiments.

The results are in Table~\ref{tabel:runtimeMOT17}. Asterisk in LifT time column indicate that the problem cannot be finished. Some of the processes are killed by the system because of too much memory consumption. Some processes do not finish within more than 27 hours. Moreover, LifT often occupied up to 30 cores for solving one sequence. Our solver uses only one core.  In the cases when LifT does not finish, we evaluate the best feasible solution found by LifT. Those were typically the initial feasible solutions. That is, the solutions that ignore the lifted edges. Obtaining the initial solutions for these difficult instances took between 1700 and 4600 seconds.
The numbers in brackets relate our results to LifT results. The time column provides the ratio between our time and LifT time. The IDF1 column presents the difference between ApLift and LifT.

\begin{table}[h]
\small{
%\caption{LDP solver runtime and IDF1 comparison between our solver with 6, 11, 31 and 51 iterations against two-step
%Lif\_T\cite{hornakova2020lifted} on first $n$ frames of sequence \textit{MOT20-02} from MOT20.}
\caption{Runtime and IDF1 comparison of LDP solvers: ApLift (ours)  with 6, 11, 31 and 51 iterations and LifT\cite{hornakova2020lifted} (two step procedure) on first $n$ frames of sequence \textit{MOT20-01} from MOT20.}
\label{tabel:runtime02}

\begin{tabular}{c l c c c c c}
%\cline{1-7}
\toprule

$n$ & Measure      & LifT & Our6 & Our11 & Our31 & Our51   \\ \toprule
%\cline{1-7}
%\multirow{2}{*}{30}     & IDF1         & 84.6 & 85.2   & 85.2    & 85.2    & 85.2   \\ % &  &  \\ %\cline{1-7}
%     & time {[}s{]} & 38   & 0      & 1       & 3       & 6    \\ \midrule%   &  &  \\ \cline{1-7}
\multirow{2}{*}{50}     & IDF1$\uparrow$         & $\mathbf{83.4}$ &  $\mathbf{83.4}$   &  $\mathbf{83.4}$    &  $\mathbf{83.4}$    &  $\mathbf{83.4}$   \\ % &  &  \\ %\cline{1-7}
     & time {[}s{]} & $62$  & $4$      &$ 7$       & $25$      & $46$     \\ \midrule% &  &  \\ \cline{1-7}
%\multirow{2}{*}{75}     & IDF1         & 79.4 & 82.6   & 82.6    & 81.5    & 81.5   \\ % &  &  \\ %\cline{1-7}
%     & time {[}s{]} & 396  & 8      & 14      & 53      & 121    \\ \midrule% &  &  \\ \cline{1-7}
\multirow{2}{*}{100}    & IDF1$\uparrow$         & $\mathbf{80.6}$ & $79.9$   & $79.9$    & $79.9$    & $79.9$   \\ % &  &  \\ %\cline{1-7}
    & time {[}s{]} & $124$  & $30$     &$ 54$      & $182$      & $360$    \\ \midrule% &  &  \\ \cline{1-7}
\multirow{2}{*}{150}    & IDF1$\uparrow$         & $\mathbf{78.7}$ & $76.8$   & $76.8$    & $76.8$    & $76.8$   \\ % &  &  \\ %\cline{1-7}
    & time {[}s{]} & $222$ & $61$     & $110$      & $378$     & $780$    \\ \midrule% &  &  \\ \cline{1-7}
\multirow{2}{*}{200}    & IDF1$\uparrow$         & $\mathbf{77.6}$& $75.8$   & $75.8$    & $75.8$    & $75.8$  \\ %  &  &  \\ %\cline{1-7}
    & time {[}s{]} & $354$ & $95$     & $177$      & $604$     & $1195$   \\ \bottomrule% &  &  \\ \cline{1-7}
\end{tabular}
}
\end{table}

\subsection{Qualitative Results}\label{sec:qualitative-results}
Figure~\ref{fig:mot2} and Figure~\ref{fig:mot4} show qualitative tracking results from the  MOT20~\cite{MOTChallenge20} and MOT17~\cite{MOTChallenge20} datasets. Comparing the samples, it becomes apparent that the density of objects in MOT20 is much higher than in MOT17. The sequence MOT20-04 (Figure~\ref{fig:mot2}) has an average density of $178.6$ objects per frame and sequence MOT17-12 (Figure~\ref{fig:mot4}) only $9.6$. The high density in MOT20 results in very crowded groups of persons which are occluding each other completely or partially. Accordingly, appearance information are ambiguous, leading to less discriminative edge costs.
%These occlusions are very challenging especially for visual features. 
%Missing body parts and body parts from other persons in the bounding boxes affect the meaningfulness of visual features negatively. 
An additional challenge arises due to the distance between the  camera and the persons, as well as global illumination changes in some sequences. The images shown in Figure~\ref{fig:mot2} are captured in a temporal distance of 40 frames (\ie $1.6$ seconds) and the illumination changes heavily. This leads to appearance changes within a  short time, which makes re-identification challenging. For instance, the  person with id $666$ (top right corner) in Figure~\ref{fig:mot2} is wearing a red scarf and a beige jacket. Only a few frames later, the person is barely visible and colors  have changed.

Despite these challenges, our system delivers accurate tracking results, as can be seen from the result images. Also the combinatorial and computational challenge in computing optimal trajectories for MOT20, considering for each detections all possible connections within a $50$ frame range becomes apparent. 

Result video for all test sequences can be obtain from the official evaluation server, for MOT15\footnote{\scriptsize{ \url{https://motchallenge.net/method/MOT=4031&chl=2}}}, MOT16\footnote{\scriptsize{\url{https://motchallenge.net/method/MOT=4031&chl=5}}}, MOT17\footnote{\scriptsize{ \url{https://motchallenge.net/method/MOT=4031&chl=10}}}, and MOT20\footnote{ \scriptsize{\url{https://motchallenge.net/method/MOT=4031&chl=13}}}.

%Another challenge arises in MOT17 due to moving cameras, making spatio-temporal measurements difficult. An displacement induced by camera motion can be seen by comparing Figure~\ref{fig:mot3} with Figure~\ref{fig:mot4}.

\subsection{Tracking Metrics}\label{sec:tracking_metrics}
A detailed evaluation of our proposed MOT system in terms of tracking metrics for all sequences of the datasets MOT20~\cite{MOTChallenge20} and MOT17~\cite{MOT16} are reported in Table \ref{tab:mot metrics}. %The results are divided into training and test sets as provided by the authors. 
Evaluations on the test set are performed by the official benchmark evaluation server at \href{https://motchallenge.net}{https://motchallenge.net} where our test results are reported as well. The tracking method for training sequences are trained with leave-one-out strategy to avoid overfitting on the corresponding training sequence.

\clearpage

\begin{table*}
\center
\small{
%\caption{LDP solver runtime and IDF1 comparison between our solver with 6, 11, 31, 51 and 101 iterations against globally optimal
\caption{Runtime and IDF1 comparison of LDP solvers: ApLift (ours)  with 6, 11, 31, 51 and 101 iterations and globally optimal (one step)
LifT\cite{hornakova2020lifted} on MOT17 train. Numbers in parenthesis in the time column show the difference between the solvers, in the IDF1 column the ratio between Lift and ApLift.}
\label{tabel:runtimeMOT17}
\begin{tabular}{l  c c  c c  c c  c c  c c  c c}
%\cline{1-7}
\toprule
 & \multicolumn{2}{c}{LifT} & \multicolumn{2}{c}{Ours-6} & \multicolumn{2}{c}{Ours-11} & \multicolumn{2}{c}{Ours-31} & \multicolumn{2}{c}{Ours-51} & \multicolumn{2}{c}{Ours-101}\\
Sequence Name & Time$\downarrow$ & IDF1$\uparrow$ & Time$\downarrow$ & IDF1$\uparrow$ & Time$\downarrow$ & IDF1$\uparrow$ & Time$\downarrow$ & IDF1$\uparrow$ & Time$\downarrow$ & IDF1$\uparrow$ & Time$\downarrow$ & IDF1$\uparrow$\\ \toprule
\multirow{2}{*}{02-DPM} &$ 7324 $&$ \textbf{49.4} $&$ 94 $&$ 47.4 $&$ 157 $&$ 47.4 $&$ 513 $&$ 49.1 $&$ 989 $&$ 49.1 $&$ 2415 $&$ 49.1 $\\
 &$ (1.0) $&$ (0.0) $&$ (0.01) $&$ (-2.00) $&$ (0.02) $&$ (-2.00) $&$ (0.07) $&$ (-0.30) $&$ (0.14) $&$ (-0.30) $&$ (0.33) $&$ (-0.30)$ \\  \midrule
\multirow{2}{*}{02-FRCNN} &$ 4073 $&$ 54.7 $&$ 97 $&$ \mathbf{54.9 }$&$ 161 $&$ \mathbf{54.9 } $&$ 526 $&$ \mathbf{54.9 } $&$ 1021 $&$ \mathbf{54.9 } $&$ 2503 $&$ \mathbf{54.9 } $\\
 &$ (1.0) $&$ (0.0) $&$ (0.02) $&$ (0.20) $&$ (0.04) $&$ (0.20) $&$ (0.13) $&$ (0.20) $&$ (0.25) $&$ (0.20) $&$ (0.61) $&$ (0.20)$ \\  \midrule
\multirow{2}{*}{02-SDP} &$ 7795 $&$ \mathbf{56.7} $&$ 131 $&$ 55.0 $&$ 219 $&$ 55.0 $&$ 717 $&$ 55.0 $&$ 1410 $&$ 55.0 $&$ 3685 $&$ 55.0 $\\
 &$ (1.0) $&$ (0.0) $&$ (0.02) $&$ (-1.70) $&$ (0.03) $&$ (-1.70) $&$ (0.09) $&$ (-1.70) $&$ (0.18) $&$ (-1.70) $&$ (0.47) $&$ (-1.70)$ \\  \midrule
\multirow{2}{*}{04-DPM} &$ * $&$ \mathbf{75.4} $&$ 449 $&$ 74.7 $&$ 756 $&$ 74.7 $&$ 2220 $&$ 74.7 $&$ 3929 $&$ 75.0 $&$ 8578 $&$ 75.0 $\\
 &$ (*) $&$ (0.0) $&$ (*) $&$ (-0.70) $&$ (*) $&$ (-0.70) $&$ (*) $&$ (-0.70) $&$ (*) $&$ (-0.40) $&$ (*) $&$ (-0.40)$ \\  \midrule
\multirow{2}{*}{04-FRCNN} &$ 4889 $&$ \mathbf{79.2} $&$ 383 $&$ 78.1 $&$ 644 $&$ 78.1 $&$ 1811 $&$ 76.3 $&$ 3111 $&$ 78.2 $&$ 6565 $&$ 78.2$ \\
 &$ (1.0) $&$ (0.0) $&$ (0.08) $&$ (-1.10) $&$ (0.13) $&$ (-1.10) $&$ (0.37) $&$ (-2.90) $&$ (0.64) $&$ (-1.00) $&$ (1.34) $&$ (-1.00)$ \\  \midrule
\multirow{2}{*}{04-SDP} &$ * $&$ \mathbf{82.3} $&$ 499 $&$ 78.0 $&$ 839 $&$ 78.0 $&$ 2441 $&$ 78.0 $&$ 4294 $&$ 77.7 $&$ 9269 $&$ 79.9 $\\
 &$ (*) $&$ (0.0) $&$ (*) $&$ (-4.30) $&$ (*) $&$ (-4.30) $&$ (*) $&$ (-4.30) $&$ (*) $&$ (-4.60) $&$ (*) $&$ (-2.40) $\\  \midrule
\multirow{2}{*}{05-DPM} &$ 535 $&$ \mathbf{65.0} $&$ 10 $&$ 62.6 $&$ 15 $&$ 62.6 $&$ 57 $&$ 63.5 $&$ 116 $&$ 63.5 $&$ 298 $&$ 63.5$ \\
 &$ (1.0) $&$ (0.0) $&$ (0.02) $&$ (-2.40) $&$ (0.03) $&$ (-2.40) $&$ (0.11) $&$ (-1.50) $&$ (0.22) $&$ (-1.50) $&$ (0.56) $&$ (-1.50)$ \\  \midrule
\multirow{2}{*}{05-FRCNN} &$ 514 $&$ \mathbf{66.6} $&$ 10 $&$ 63.8 $&$ 15 $&$ 63.8 $&$ 57 $&$ 64.0 $&$ 118 $&$ 63.9 $&$ 315 $&$ 65.6 $\\
 &$ (1.0) $&$ (0.0) $&$ (0.02) $&$ (-2.80) $&$ (0.03) $&$ (-2.80) $&$ (0.11) $&$ (-2.60) $&$ (0.23) $&$ (-2.70) $&$ (0.61) $&$ (-1.00) $\\  \midrule
\multirow{2}{*}{05-SDP} &$ 604 $&$ \mathbf{67.9} $&$ 11 $&$ \mathbf{67.9} $&$ 18 $&$ \mathbf{67.9} $&$ 67 $&$ 67.1 $&$ 137 $&$ 67.1 $&$ 364 $&$ 67.6 $\\
 &$ (1.0) $&$ (0.0) $&$ (0.02) $&$ (0.00) $&$ (0.03) $&$ (0.00) $&$ (0.11) $&$ (-0.80) $&$ (0.23) $&$ (-0.80) $&$ (0.60) $&$ (-0.30) $\\  \midrule
\multirow{2}{*}{09-DPM} &$ 6692 $&$ \mathbf{67.5} $&$ 42 $&$ 66.4 $&$ 70 $&$ 66.4 $&$ 232 $&$ \mathbf{67.5} $&$ 480 $&$ \mathbf{67.5} $&$ 1281 $&$ \mathbf{67.5} $\\
 &$ (1.0) $&$ (0.0) $&$ (0.01) $&$ (-1.10) $&$ (0.01) $&$ (-1.10) $&$ (0.03) $&$ (0.00) $&$ (0.07) $&$ (0.00) $&$ (0.19) $&$ (0.00) $\\  \midrule
\multirow{2}{*}{09-FRCNN} &$ 11888 $&$ \mathbf{68.2} $&$ 37 $&$ \mathbf{68.2} $&$ 61 $&$ \mathbf{68.2} $&$ 201 $&$ \mathbf{68.2} $&$ 407 $&$ \mathbf{68.2} $&$ 1095 $&$ \mathbf{68.2} $\\
 &$ (1.0) $&$ (0.0) $&$ (0.00) $&$ (0.00) $&$ (0.01) $&$ (0.00) $&$ (0.02) $&$ (0.00) $&$ (0.03) $&$ (0.00) $&$ (0.09) $&$ (0.00) $\\  \midrule
\multirow{2}{*}{09-SDP} &$ 1462 $&$ \mathbf{68.6} $&$ 44 $&$ 67.1 $&$ 74 $&$ 67.1 $&$ 247 $&$ 68.5 $&$ 512 $&$ 68.5 $&$ 1443 $&$ 68.5 $\\
 &$ (1.0) $&$ (0.0) $&$ (0.03) $&$ (-1.50) $&$ (0.05) $&$ (-1.50) $&$ (0.17) $&$ (-0.10) $&$ (0.35) $&$ (-0.10) $&$ (0.99) $&$ (-0.10) $\\  \midrule
\multirow{2}{*}{10-DPM} &$ * $&$ 66.0 $&$ 279 $&$ \mathbf{68.0} $&$ 466 $&$ \mathbf{68.0} $&$ 1524 $&$ 66.8 $&$ 3087 $&$ 67.0 $&$ 9478 $&$ 67.9 $\\
 &$ (*) $&$ (0.0) $&$ (*) $&$ (2.00) $&$ (*) $&$ (2.00) $&$ (*) $&$ (0.80) $&$ (*) $&$ (1.00) $&$ (*) $&$ (1.90) $\\  \midrule
\multirow{2}{*}{10-FRCNN} &$ * $&$ 65.2 $&$ 310 $&$ 68.8 $&$ 511 $&$ 68.5 $&$ 1689 $&$ \mathbf{69.4 }$&$ 3428 $&$ \mathbf{69.4 } $&$ 10743 $&$ \mathbf{69.4 } $\\
 &$ (*) $&$ (0.0) $&$ (*) $&$ (3.60) $&$ (*) $&$ (3.30) $&$ (*) $&$ (4.20) $&$ (*) $&$ (4.20) $&$ (*) $&$ (4.20) $\\  \midrule
\multirow{2}{*}{10-SDP} &$ * $&$ 65.4 $&$ 379 $&$ 67.0 $&$ 630 $&$ 67.0 $&$ 2090 $&$ 67.4 $&$ 4294 $&$ 67.1 $&$ 13379 $&$ \mathbf{69.8} $\\
 &$ (*) $&$ (0.0) $&$ (*) $&$ (1.60) $&$ (*) $&$ (1.60) $&$ (*) $&$ (2.00) $&$ (*) $&$ (1.70) $&$ (*) $&$ (4.40) $\\  \midrule
\multirow{2}{*}{11-DPM} &$ 1991 $&$ \mathbf{76.3} $&$ 60 $&$ \mathbf{76.3} $&$ 99 $&$ \mathbf{76.3} $&$ 335 $&$\mathbf{76.3} $&$ 672 $&$ \mathbf{76.3} $&$ 1672 $&$ \mathbf{76.3} $\\
 &$ (1.0) $&$ (0.0) $&$ (0.03) $&$ (0.00) $&$ (0.05) $&$ (0.00) $&$ (0.17) $&$ (0.00) $&$ (0.34) $&$ (0.00) $&$ (0.84) $&$ (0.00) $\\  \midrule
\multirow{2}{*}{11-FRCNN} &$ 2382 $&$ \mathbf{78.3} $&$ 68 $&$ \mathbf{78.3} $&$ 113 $&$ \mathbf{78.3} $&$ 366 $&$ \mathbf{78.3} $&$ 729 $&$ \mathbf{78.3} $&$ 1799 $&$ \mathbf{78.3} $\\
 &$ (1.0) $&$ (0.0) $&$ (0.03) $&$ (0.00) $&$ (0.05) $&$ (0.00) $&$ (0.15) $&$ (0.00) $&$ (0.31) $&$ (0.00) $&$ (0.76) $&$ (0.00) $\\  \midrule
\multirow{2}{*}{11-SDP} &$ 3195 $&$ 80.0 $&$ 68 $&$ 79.8 $&$ 113 $&$ 79.8 $&$ 370 $&$ \mathbf{80.1} $&$ 748 $&$ 80.0 $&$ 2057 $&$ 80.0 $\\
 &$ (1.0) $&$ (0.0) $&$ (0.02) $&$ (-0.20) $&$ (0.04) $&$ (-0.20) $&$ (0.12) $&$ (0.10) $&$ (0.23) $&$ (0.00) $&$ (0.64) $&$ (0.00) $\\  \midrule
\multirow{2}{*}{13-DPM} &$ * $&$ 62.8 $&$ 152 $&$ \mathbf{66.8} $&$ 252 $&$ \mathbf{66.8} $&$ 944 $&$ \mathbf{66.8} $&$ 2008 $&$ \mathbf{66.8} $&$ 6340 $&$ 65.7 $\\
 &$ (*) $&$ (0.0) $&$ (*) $&$ (4.00) $&$ (*) $&$ (4.00) $&$ (*) $&$ (4.00) $&$ (*) $&$ (4.00) $&$ (*) $&$ (2.90) $\\  \midrule
\multirow{2}{*}{13-FRCNN} &$ * $&$ 62.5 $&$ 217 $&$ \mathbf{69.8} $&$ 351 $&$ \mathbf{69.8} $&$ 1331 $&$\mathbf{69.8} $&$ 2942 $&$ 67.7 $&$ 9813 $&$ 66.2 $\\
 &$ (*) $&$ (0.0) $&$ (*) $&$ (7.30) $&$ (*) $&$ (7.30) $&$ (*) $&$ (7.30) $&$ (*) $&$ (5.20) $&$ (*) $&$ (3.70) $\\  \midrule
\multirow{2}{*}{13-SDP} &$ * $&$ 64.5 $&$ 196 $&$ \mathbf{66.8} $&$ 326 $&$ \mathbf{66.8} $&$ 1237 $&$ \mathbf{66.8} $&$ 2698 $&$ 66.2 $&$ 8954 $&$ 65.6 $\\
 &$ (*) $&$ (0.0) $&$ (*) $&$ (2.30) $&$ (*) $&$ (2.30) $&$ (*) $&$ (2.30) $&$ (*) $&$ (1.70) $&$ (*) $&$ (1.10) $\\  \midrule
\multirow{2}{*}{OVERALL} &$ * $&$ \mathbf{70.7} $&$ 168 $&$ 70.3 $&$ 280 $&$ 70.3 $&$ 904 $&$ 70.2 $&$ 1768 $&$ 70.3 $&$ 4859 $&$ \mathbf{70.7} $\\
 &$ (*) $&$ (0.0) $&$ (*) $&$ (-0.40) $&$ (*) $&$ (-0.40) $&$ (*) $&$ (-0.50) $&$ (*) $&$ (-0.40) $&$ (*) $&$ (0.00) $\\  \bottomrule
\end{tabular}
}
\end{table*}

\begin{table*}
\center
\small{
\caption{Evaluation results for training and test sequences for datasets MOT17~\cite{MOT16} and MOT20~\cite{MOTChallenge20}} 
\begin{tabular}{c l c c c c c c c c }
\toprule
& Sequence & MOTA$\uparrow$ & IDF1$\uparrow$ & MT$\uparrow$ & ML$\downarrow$ & FP$\downarrow$ & FN$\downarrow$ & IDS$\downarrow$ & Frag. $\downarrow$\\
\toprule
\parbox[t]{3mm}{\multirow{3}{*}{\rotatebox[origin=c]{90}{MOT20 Train\ }}}
& MOT20-01 & 65.8 & 62.0 & 31 & 10 & 180 & 6512 & 109 & 87\\
& MOT20-02 & 62.3 & 55.1 & 108 & 18 & 1393 & 56420 & 548 & 534\\
& MOT20-03 & 80.4 & 76.1 & 427 & 66 & 5427 & 55552 & 623 & 591\\
& MOT20-05 & 74.6 & 57.8 & 643 & 115 & 8778 & 152927 & 2231 & 2063\\
& OVERALL & 74.4 & 62.8 & 1209 & 209 & 15778 & 271411 & 3511 & 3275\\
\midrule
\parbox[t]{3mm}{\multirow{3}{*}{\rotatebox[origin=c]{90}{MOT20 Test\ \ }}}
& MOT20-04 & 79.3 & 68.8 & 412 & 40 & 8315 & 47364 & 968 & 840\\
& MOT20-06 & 36.1 & 36.8 & 41 & 111 & 4786 & 79313 & 740 & 744\\
& MOT20-07 & 56.9 & 54.7 & 40 & 15 & 936 & 13135 & 194 & 195\\
& MOT20-08 & 26.5 & 33.8 & 20 & 98 & 3702 & 52924 & 339 & 333\\
& OVERALL & 58.9 & 56.5 & 513 & 264 & 17739 & 192736 & 2241 & 2112\\
\midrule
\parbox[t]{3mm}{\multirow{21}{*}{\rotatebox[origin=c]{90}{MOT17 Train}}}
& MOT17-02-DPM & 42.2 & 52.5 & 14 & 29 & 125 & 10588 & 26 & 26\\
& MOT17-02-FRCNN & 47.3 & 58.4 & 15 & 21 & 227 & 9532 & 27 & 30\\
& MOT17-02-SDP & 55.1 & 60.5 & 17 & 16 & 289 & 7994 & 53 & 52\\
& MOT17-04-DPM & 70.9 & 78.9 & 40 & 21 & 340 & 13481 & 17 & 29\\
& MOT17-04-FRCNN & 68.0 & 78.4 & 39 & 21 & 179 & 15044 & 5 & 13\\
& MOT17-04-SDP & 77.9 & 80.8 & 47 & 13 & 439 & 10035 & 29 & 68\\
& MOT17-05-DPM & 60.0 & 64.5 & 48 & 34 & 475 & 2260 & 31 & 24\\
& MOT17-05-FRCNN & 57.8 & 64.0 & 55 & 32 & 650 & 2225 & 46 & 41\\
& MOT17-05-SDP & 62.6 & 67.8 & 59 & 19 & 693 & 1842 & 53 & 46\\
& MOT17-09-DPM & 73.0 & 72.8 & 14 & 1 & 46 & 1380 & 10 & 9\\
& MOT17-09-FRCNN & 71.5 & 68.4 & 14 & 1 & 105 & 1403 & 10 & 9\\
& MOT17-09-SDP & 74.1 & 72.9 & 14 & 1 & 66 & 1302 & 10 & 11\\
& MOT17-10-DPM & 65.3 & 67.4 & 32 & 6 & 847 & 3545 & 61 & 74\\
& MOT17-10-FRCNN & 62.8 & 65.8 & 40 & 2 & 2121 & 2513 & 139 & 114\\
& MOT17-10-SDP & 66.3 & 66.5 & 43 & 2 & 1967 & 2189 & 173 & 120\\
& MOT17-11-DPM & 69.2 & 75.7 & 34 & 21 & 248 & 2624 & 37 & 17\\
& MOT17-11-FRCNN & 71.5 & 76.8 & 38 & 18 & 412 & 2233 & 47 & 15\\
& MOT17-11-SDP & 72.6 & 78.5 & 42 & 13 & 547 & 1981 & 58 & 19\\
& MOT17-13-DPM & 64.4 & 64.8 & 55 & 33 & 627 & 3436 & 83 & 56\\
& MOT17-13-FRCNN & 67.8 & 63.4 & 77 & 8 & 1739 & 1892 & 120 & 76\\
& MOT17-13-SDP & 67.2 & 63.7 & 72 & 18 & 1388 & 2312 & 117 & 60\\
& OVERALL & 66.0 & 71.4 & 809 & 330 & 13530 & 99811 & 1152 & 909\\
\midrule
\parbox[t]{3mm}{\multirow{21}{*}{\rotatebox[origin=c]{90}{MOT17 Test}}}
& MOT17-01-DPM & 48.8 & 54.3 & 8 & 10 & 113 & 3181 & 8 & 21\\
& MOT17-01-FRCNN & 47.0 & 57.5 & 9 & 10 & 360 & 3050 & 11 & 21\\
& MOT17-01-SDP & 45.2 & 55.4 & 9 & 10 & 488 & 3033 & 13 & 29\\
& MOT17-03-DPM & 73.8 & 73.4 & 85 & 17 & 4360 & 22905 & 118 & 261\\
& MOT17-03-FRCNN & 72.8 & 74.7 & 74 & 17 & 3471 & 24883 & 109 & 234\\
& MOT17-03-SDP & 77.7 & 75.4 & 94 & 13 & 4676 & 18482 & 139 & 386\\
& MOT17-06-DPM & 57.7 & 61.2 & 94 & 76 & 1142 & 3765 & 77 & 91\\
& MOT17-06-FRCNN & 57.3 & 58.4 & 102 & 59 & 1652 & 3279 & 102 & 140\\
& MOT17-06-SDP & 57.2 & 59.5 & 107 & 58 & 1700 & 3251 & 87 & 125\\
& MOT17-07-DPM & 45.7 & 52.5 & 11 & 15 & 1062 & 8038 & 80 & 126\\
& MOT17-07-FRCNN & 45.0 & 53.1 & 11 & 15 & 1345 & 7862 & 75 & 135\\
& MOT17-07-SDP & 46.6 & 53.8 & 13 & 11 & 1622 & 7310 & 87 & 166\\
& MOT17-08-DPM & 33.7 & 44.3 & 17 & 37 & 421 & 13533 & 48 & 67\\
& MOT17-08-FRCNN & 31.5 & 42.1 & 17 & 37 & 462 & 13948 & 53 & 74\\
& MOT17-08-SDP & 34.5 & 45.2 & 18 & 34 & 445 & 13339 & 63 & 85\\
& MOT17-12-DPM & 47.6 & 61.9 & 23 & 36 & 563 & 3959 & 20 & 32\\
& MOT17-12-FRCNN & 47.8 & 62.3 & 18 & 40 & 296 & 4219 & 13 & 24\\
& MOT17-12-SDP & 50.0 & 66.1 & 19 & 42 & 488 & 3836 & 11 & 31\\
& MOT17-14-DPM & 37.8 & 51.0 & 19 & 71 & 1147 & 10191 & 151 & 150\\
& MOT17-14-FRCNN & 33.9 & 48.4 & 25 & 62 & 2369 & 9636 & 206 & 228\\
& MOT17-14-SDP & 37.0 & 49.9 & 25 & 58 & 2427 & 8970 & 238 & 246\\
& OVERALL & 60.5 & 65.6 & 798 & 728 & 30609 & 190670 & 1709 & 2672\\
\bottomrule
\end{tabular}
\label{tab:mot metrics}
}
\end{table*}

% \begin{figure*}[t]
%     \centering
%     \includegraphics[scale=0.5]{mot20-04-1.jpg}
%     \caption{Example image from sequence MOT20-04 at frame 1092. The image shows a crowded scene captured by a static camera directly before an illumination change happens. Source data can be found at \href{https://motchallenge.net/method/MOT=4031\&chl=13\&vidSeq=MOT20-04}{https://motchallenge.net/method/MOT=4031\&chl=13\&vidSeq=MOT20-04}.}
%     \label{fig:mot1}
% \end{figure*}
\begin{figure*}[t]
    \centering
    \includegraphics[scale=0.5]{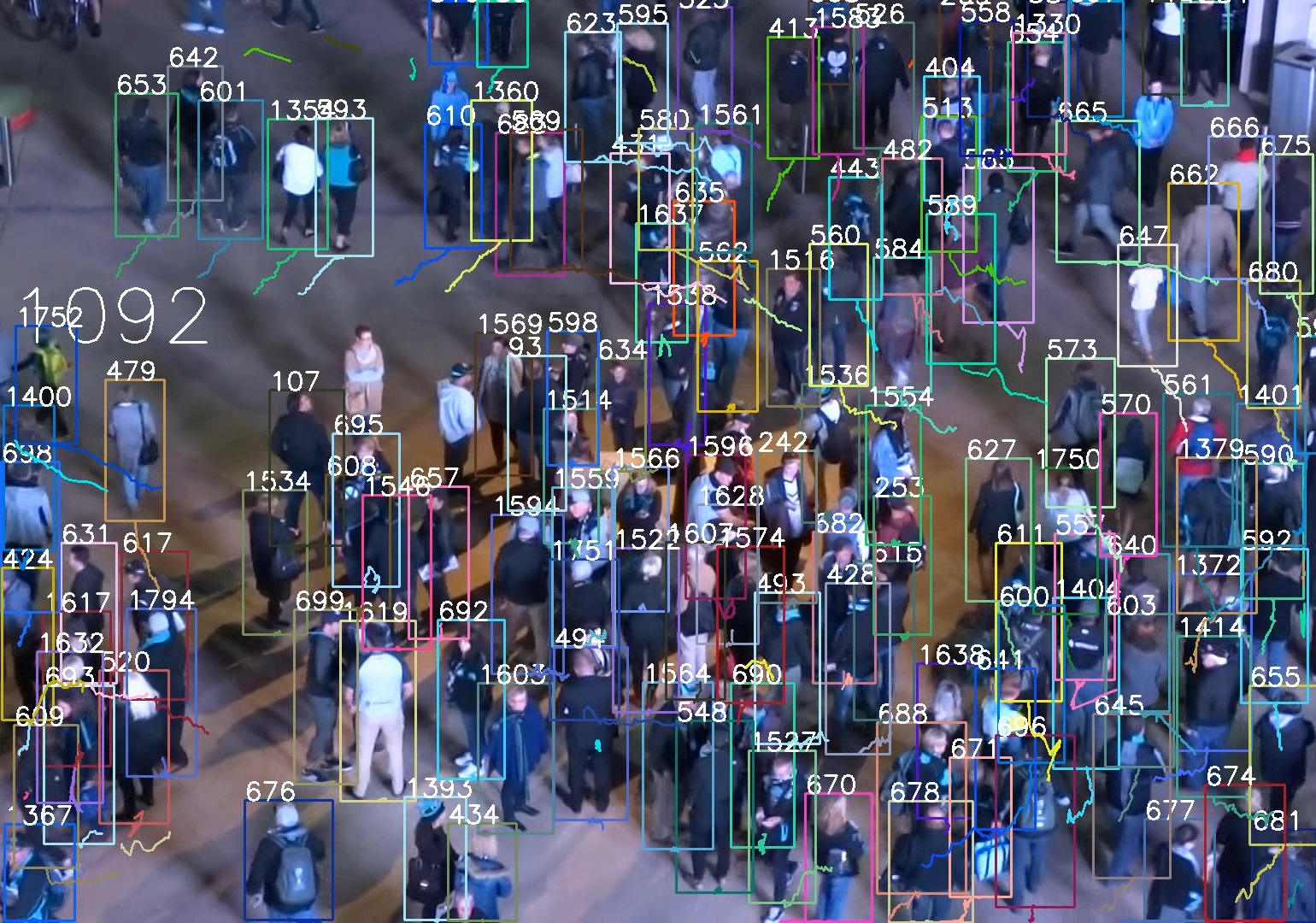}
    \includegraphics[scale=0.5]{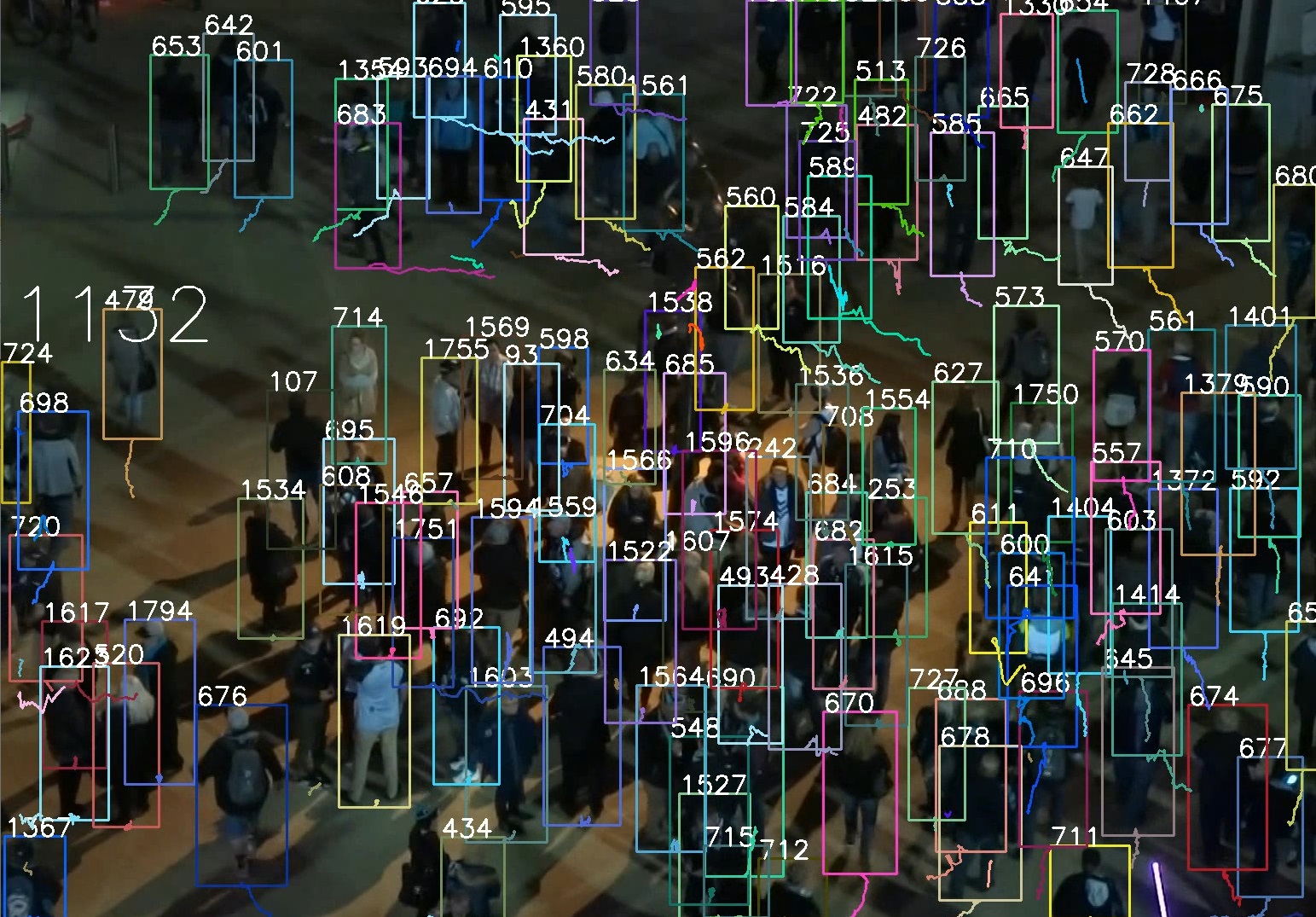}
    \caption{Example images from sequence MOT20-04 at frames 1092 and 1132. The images shows a crowded scene captured by a static camera. The above image is captured before an illumination change happens. The lower image is captured after an illumination change happens. The appearance of persons changes consequently to the illumination changes (\eg ID 666 in the top right corner). The result video can be found at \href{https://motchallenge.net/method/MOT=4031\&chl=13\&vidSeq=MOT20-04}{https://motchallenge.net/method/MOT=4031\&chl=13\&vidSeq=MOT20-04}.}    \label{fig:mot2}
\end{figure*}
% \begin{figure*}[t]
%     \centering
%     \includegraphics[scale=0.5]{mot17-12-1.jpg}
%     \caption{Example image from sequence MOT17-12 at frame 21. The image shows a scene captured by a moving camera. Source data can be found at \href{https://motchallenge.net/method/MOT=4031\&chl=10\&vidSeq=MOT17-12-FRCNN}{https://motchallenge.net/method/MOT=4031\&chl=10\&vidSeq=MOT17-12-FRCNN}.}
%     \label{fig:mot3}
% \end{figure*}
\begin{figure*}[t]
    \centering
    \includegraphics[scale=0.5]{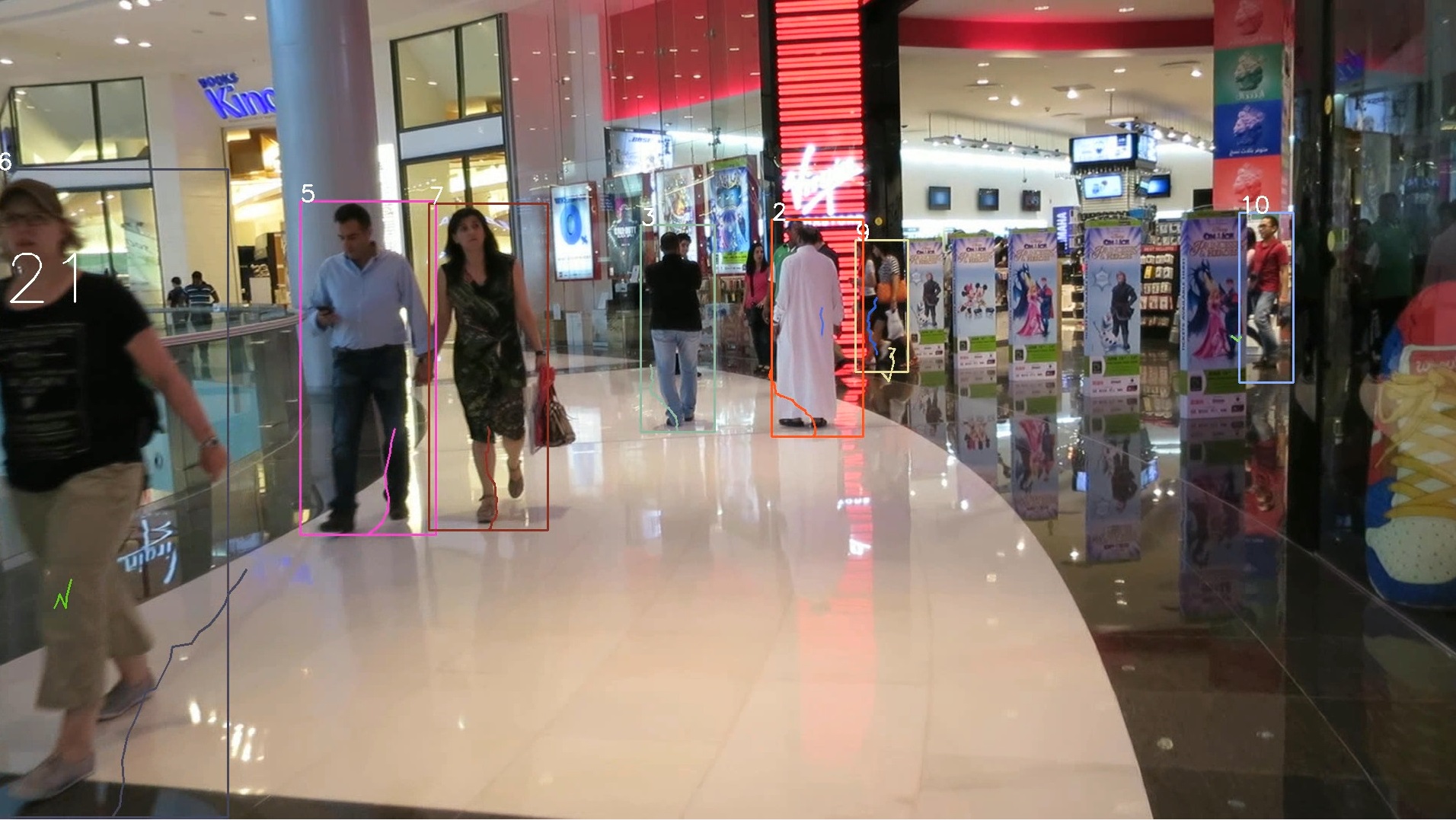}
    \includegraphics[scale=0.5]{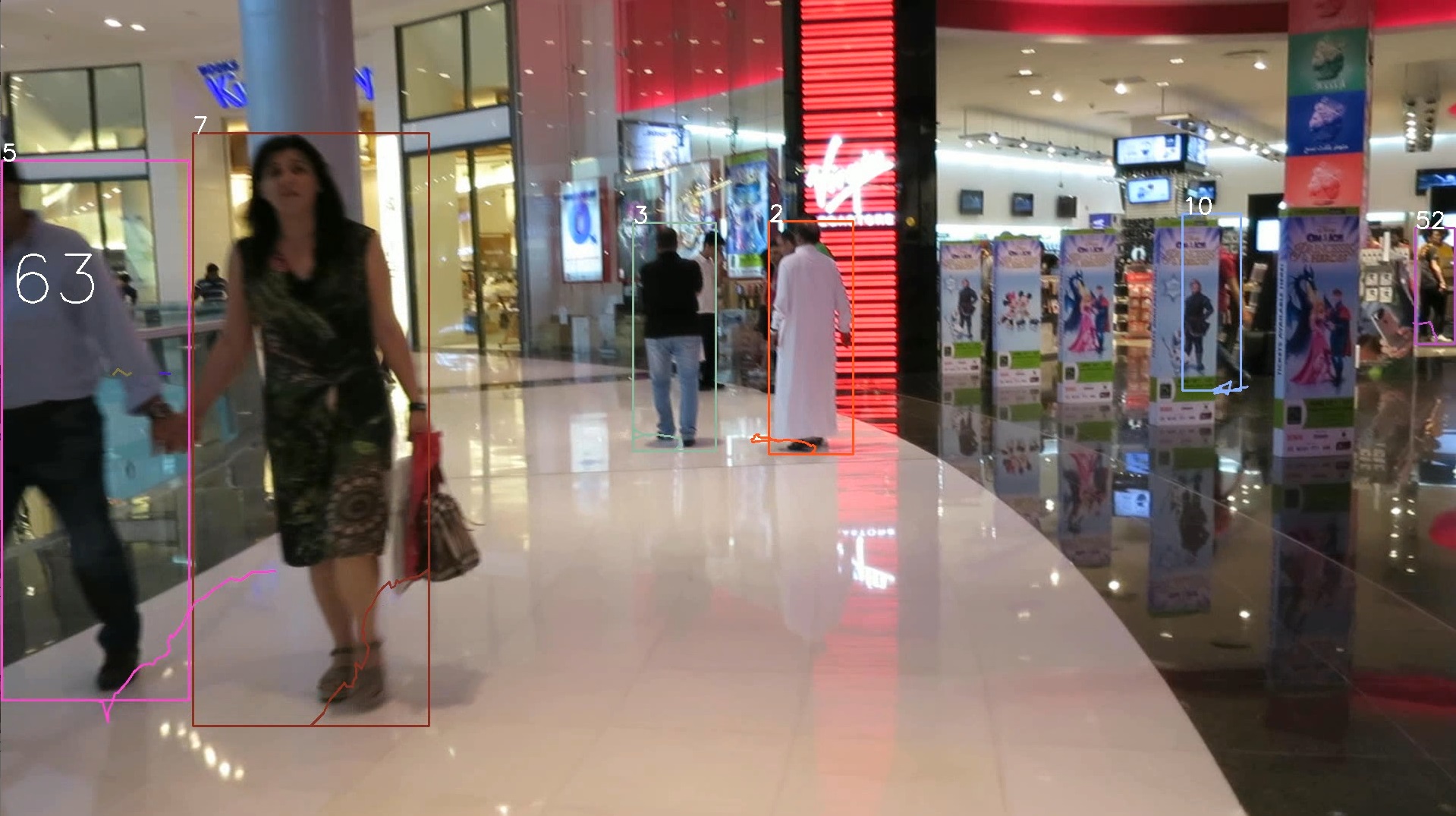}
    \caption{Example images from sequence MOT17-12 at frames 21 and 63. The image shows a scene captured by a moving camera. Compared to MOT20, the number of persons in the scene is lower and occlusions are seldom. The result video can be found at \href{https://motchallenge.net/method/MOT=4031\&chl=10\&vidSeq=MOT17-12-FRCNN}{https://motchallenge.net/method/MOT=4031\&chl=10\&vidSeq=MOT17-12-FRCNN}.}
    \label{fig:mot4}
\end{figure*}

\end{document}